\documentclass[11pt]{article}
\usepackage[a4paper,margin=1in]{geometry}
\usepackage{amsmath,amssymb,amsfonts,amsthm,mathtools,bm}
\usepackage{booktabs,tabularx,array,multirow}
\usepackage[table]{xcolor}
\usepackage{graphicx}
\usepackage{subcaption}
\usepackage{float}
\usepackage{pgfplots}
\pgfplotsset{compat=1.18}
\usepgfplotslibrary{groupplots}
\usepgfplotslibrary{statistics}
\usepackage{tikz}
\usetikzlibrary{
    arrows.meta,
    positioning,
    fit,
    backgrounds,
    calc,
    decorations.pathreplacing
}
\usepackage[normalem]{ulem}

\usepackage[section]{placeins}
\usepackage{tcolorbox}
\tcbuselibrary{skins,breakable}
\newcounter{algctr}
\newenvironment{myalgo}[1]{%
  \refstepcounter{algctr}%
  \begin{tcolorbox}[
    breakable, enhanced,
    colback=white, colframe=black!35, arc=3pt, boxrule=0.55pt,
    title={\textbf{Algorithm~\thealgctr}\;#1},
    fonttitle=\small\bfseries, coltitle=black!80,
    attach boxed title to top left={yshift=-2mm,xshift=5mm},
    boxed title style={colback=white,colframe=black!35,arc=2pt,boxrule=0.55pt},
    left=6pt, right=6pt, top=4pt, bottom=4pt
  ]%
  \small
}{%
  \end{tcolorbox}%
}
\newcommand{\kwb}[1]{\textbf{#1}}
\newcommand{\AlgIn}[1]{\noindent\textbf{Input:}\enspace #1\par\vspace{2pt}}
\newcommand{\AlgOut}[1]{\noindent\textbf{Output:}\enspace #1\par\vspace{4pt}\hrule height 0.3pt\vspace{4pt}}
\newcommand{\AlgSec}[1]{\par\vspace{3pt}\noindent\textit{\textbf{#1}}\par\vspace{2pt}}
\usepackage{enumitem}
\usepackage{microtype}
\usepackage[numbers,sort&compress]{natbib}
\usepackage{hyperref}
\hypersetup{
  pdftitle={Fixed and Adaptive Topological DeepONets: Functional Measurements on Hausdorff Locally Convex Spaces},
  pdfauthor={Khemraj Shukla and George Em Karniadakis},
  pdfsubject={Topological Operator learning},
  pdfkeywords={DeepONet, operator learning, locally convex spaces, SciDAC}
}
\usepackage[nameinlink,noabbrev]{cleveref}

\hypersetup{
  colorlinks=true,
  linkcolor=blue!55!black,
  citecolor=blue!55!black,
  urlcolor=blue!55!black
}

\newtheorem{theorem}{Theorem}[section]
\newtheorem{corollary}{Corollary}[section]
\newtheorem{remark}{Remark}[section]

\definecolor{fixedcolor}{RGB}{31,119,180}
\definecolor{adaptivecolor}{RGB}{255,127,14}
\definecolor{sensorcolor}{RGB}{44,160,44}
\definecolor{twostepcolor}{RGB}{214,39,40}
\definecolor{vanillacolor}{RGB}{148,103,189}
\definecolor{fnocolor}{RGB}{23,23,23}

\title{Fixed and Adaptive Topological DeepONets:
Functional Measurements on Hausdorff Locally Convex Spaces}
\author{
Khemraj Shukla\thanks{
Corresponding author:
\href{mailto:khemraj_shukla@brown.edu}
{khemraj\_shukla@brown.edu}
}
\\
\small Division of Applied Mathematics, Brown University\\
\small Providence, Rhode Island 02912, USA
\and
George Em Karniadakis
\\
\small Division of Applied Mathematics, Brown University\\
\small Providence, Rhode Island 02912, USA
}
\date{}
\begin{document}
\maketitle
\begin{abstract}
Deep Operator Networks (DeepONets)~\cite{lu2021deeponet} typically encode an
input function through point values on a fixed discretization. Building on
the Topological DeepONet framework of
Ismailov~\cite{ismailov2026topological}, we replace point samples by
continuous linear functionals drawn from the continuous dual
$(\mathcal{V},\{p_\alpha\}_{\alpha\in A})$\footnote{Here $\mathcal{V}$ is a
vector space and $\{p_\alpha\}_{\alpha\in A}$ is a point-separating family of
seminorms indexed by a (possibly infinite) index set $A$; the seminorms
generate the locally convex topology on $\mathcal{V}$, and $\mathcal{V}'$
denotes its continuous dual, the space of continuous linear functionals on
$\mathcal{V}$.}, whose topology is generated by a
point-separating family of seminorms rather than a single norm, and develop
fixed and adaptive functional measurement systems. These measurements are
combined with a coefficient-space Two-Step procedure~\cite{lee2024training},
while a training-only decoder and regularization stabilize the adaptive
coordinates. We also derive a discrete error decomposition separating
measurement, output-basis, and neural-approximation errors, together with a
Barron-rate refinement.
The framework is evaluated on the antiderivative operator, a genuinely
non-normable locally convex input space, heterogeneous Darcy flow, a
controlled operator, and fixed-time and time-evolving Navier--Stokes
vorticity operators. The non-normable benchmark directly demonstrates that
the method can learn operators when the input topology is generated by a
family of seminorms and cannot be represented by a single norm. In the
heterogeneous Darcy problem, the functional models retain nearly
resolution-independent errors of approximately \(5.5\%\)--\(5.6\%\) on
unseen grids, while in the controlled problem adaptive measurements reduce
the mean error to below \(1.2\%\).
For the fixed-time Navier--Stokes problem, the Adaptive Topological DeepONet
is the most accurate DeepONet-based model, attaining a three-seed mean
relative \(L^2\) error of \(1.685\%\pm0.017\%\) using only \(128\)
functional coordinates. A comparably sized Fourier neural operator
(FNO)~\cite{li2021fourier} achieves the lower error
\(0.832\%\pm0.172\%\) on the uniform periodic grid, but requires the full
\(64\times64\) input field, approximately twice the training time, and about
\(10.7\times\) greater peak GPU memory. The proposed formulation therefore
provides compact, interpretable, and discretization-portable coordinates in
the continuous dual \(\mathcal V'\), including for input spaces that are not
normable.
\end{abstract}
\section{Introduction}
\label{sec:introduction}
Many problems in computational science and engineering require the repeated
evaluation of a possibly nonlinear operator
\begin{equation}
    \mathcal{G}:\mathcal{V}\longrightarrow\mathcal{U},
    \qquad v\longmapsto\mathcal{G}(v),
    \label{eq:operator-map}
\end{equation}
where $\mathcal{V}$ and $\mathcal{U}$ are spaces of input and output
functions, respectively
\cite{shukla2024deep,oommen2022learning,oommen2024rethinking,de2025deep,
laudato2025neural,shukla2024deepI, raissi2019physics,karniadakis2021physics,rackauckas2020universal,
cuomo2022scientific}.
Examples include parameter-to-solution maps for partial differential
equations, integral operators, constitutive relations, and evolution operators
that advance an initial state to its future configuration. Within this broader area, neural operators learn maps
between function spaces and have been developed using branch--trunk,
spectral, graph-based, wavelet, transformer, and neural-field
representations
\cite{lu2021deeponet,kovachki2023neuraloperator,li2021fourier,
gupta2021multiwavelet,hao2023gnot,pfaff2021meshgraphnets,
serrano2023coral}. Learning such maps
directly from data yields efficient surrogates for simulation, optimization,
inverse problems, design, and uncertainty quantification
\cite{kovachki2023neuraloperator, shukla2026uncertainty}.
Neural operators are designed to learn mappings between infinite-dimensional
function spaces rather than between finite-dimensional vectors tied to a single
discretization \cite{kovachki2023neuraloperator,li2021fourier}.
Prominent architectures include graph-based neural operators, Fourier Neural
Operators (FNOs), and Deep Operator Networks (DeepONets), which have been
applied to parametric PDEs, fluid mechanics, porous-media flow, solid
mechanics, and multiscale systems
\cite{li2021fourier,seidman2022nomad,wang2021physicsinformed,
shukla2024deep,oommen2022learning}.
Deep Operator Networks approximate nonlinear operators through a
branch--trunk architecture motivated by universal approximation results for
operators \cite{chen1995operator,lu2021deeponet}.
Given an input function $v$ sampled at $m$ prescribed sensor locations, the
branch network maps the resulting vector of pointwise values to a set of $p$
expansion coefficients, while the trunk network evaluates a learned output
basis at the query coordinate $y$:
\begin{equation}
    \widehat{\mathcal{G}}(v)(y)
    =
    \sum_{k=1}^{p}
    b_k\!\bigl(v(x_1),\ldots,v(x_m)\bigr)\,t_k(y).
    \label{eq:standard-deeponet}
\end{equation}
Several extensions have improved DeepONets' flexibility and physical
consistency. Physics-informed DeepONets incorporate governing equations into
the training objective \cite{wang2021physicsinformed}.
MIONet handles operators with multiple input functions \cite{jin2022mionet}.
SVD-DeepONet links the trunk representation to singular-value decomposition
and reduced-order modeling \cite{venturi2023svd}, while Two-Step DeepONet
separates output-basis learning from the branch coefficient map to improve
stability and generalization \cite{lee2024training}.
Latent-space and nonlinear-decoder formulations address solution sets
that resist low-dimensional linear description \cite{seidman2022nomad},
and multiscale architectures target oscillatory and high-frequency phenomena
\cite{liu2022multiscale}.
Despite these advances, a fundamental limitation of the original DeepONet
branch network is its reliance on a fixed-length vector of pointwise function
values. As a consequence, inputs observed on different meshes, resolutions, or
experimental sampling patterns must be interpolated, resampled, or padded to
a common representation before training. Furthermore, pointwise values can be
an inefficient coordinate system when the target operator depends primarily on
global structures, moments, averages, or spatial correlations of the input
function.
Variable-Input Deep Operator Networks (VIDONs) \cite{prasthofer2022vidon},
BelNet \cite{zhang2022belnet}, and discretization-invariant extensions
\cite{zhang2023discretization} have relaxed the fixed-grid constraint by
allowing sensor locations to vary across samples. These developments show that
the fixed-grid restriction is not intrinsic to operator learning and motivate
a more fundamental question: can the branch representation be formulated
directly in terms of the topology and continuous dual of the underlying
input space?
Topological DeepONets address this question by extending the classical
framework from Banach spaces of continuous functions to general Hausdorff
locally convex spaces (HLCSs), whose topology is generated by a family of
seminorms $\{p_\alpha\}_{\alpha\in A}$ \cite{ismailov2026topological}.
Rather than sampling $v$ at prescribed points, Topological DeepONets encode
the input through a finite collection of continuous linear functionals
$\ell_1,\ldots,\ell_q$ drawn from the continuous dual $\mathcal{V}'$.
The resulting coordinate vector $(\ell_1(v),\ldots,\ell_q(v)) \in \mathbb{R}^q$
is then processed by the branch network in exactly the same way as before.
Classical point-sensor DeepONets are recovered as a special case when each
functional is a continuous point-evaluation map.
This functional perspective separates the infinite-dimensional mathematical
input from the finite-dimensional branch encoding, and it provides a natural
mechanism for handling variable-resolution observations: the same $q$
functionals, approximated by quadrature rules adapted to whatever grid is
available, convert heterogeneous input representations into a common
coordinate space.
In this work we develop and evaluate two concrete instantiations of the
topological framework. The fixed Topological DeepONet employs
prescribed global functionals (e.g.\ inner products against a Lagrange or
spectral basis), while the adaptive Topological DeepONet learns the
measurement functionals directly from operator data, with a training-only
reconstruction decoder and soft regularization to prevent coordinate collapse.
We integrate both variants with an enhanced Two-Step DeepONet strategy
\cite{lee2024training}, which first constructs a stable low-dimensional
output basis via weighted singular-value decomposition and then trains the
branch network to predict the corresponding coefficients. This separation
allows point-sensor, fixed-functional, and adaptive-functional branch inputs
to be compared under a uniform output representation and on equal
parameter budgets. Full mathematical details of all three input encodings,
the Two-Step construction, and the training objectives are developed in
Section~\ref{sec:methods}.
We evaluate the proposed framework on three benchmark operator-learning
problems: the antiderivative operator, the heterogeneous Darcy-flow
equation, and the two-dimensional Navier--Stokes vorticity
problem, the latter in both fixed-time and time-evolving settings.
These benchmarks span integral, elliptic, and nonlinear time-dependent
operators, and the Darcy and Navier--Stokes datasets follow standard
neural-operator evaluation protocols from the FNO literature
\cite{li2021fourier,kovachki2023neuraloperator}.
\paragraph{Contributions.}
The principal contributions of this work are as follows.
\begin{enumerate}
    \item We provide a computational realization of Topological DeepONets in
    which finite-dimensional branch coordinates are constructed from continuous
    linear functionals in the dual space $\mathcal{V}'$ of a Hausdorff locally
    convex input space
    $(\mathcal{V},\{p_\alpha\}_{\alpha\in A})$.
    Unlike canonical point-sensor DeepONets, the resulting coordinates may
    represent weak, local, spectral, finite-element, or distributional
    measurements compatible with the topology of $\mathcal{V}$.
    \item We introduce fixed and supervised-adaptive functional measurement
    systems. The adaptive coordinates remain continuous linear functionals
    because they are learned within the span of an admissible dual dictionary.
    A training-only decoder and soft regularization are used to control
    information loss, feature collapse, and excessive drift from the structured
    initialization.
    \item We integrate these measurements with a coefficient-space Two-Step
    DeepONet construction. A weighted singular value decomposition produces a
    rank-stable output basis, and the branch model is trained directly to
    predict the corresponding reduced coefficients. This separates input
    representation learning from output-subspace approximation.
    \item We derive a discrete error decomposition that separates measurement
    reconstruction, output-basis truncation, and neural approximation errors,
    together with a Barron-rate refinement of the neural term. For the
    antiderivative operator, we evaluate the decomposition over
    $q\in\{8,16,32,64,128\}$ and verify that the theorem-consistent bound holds
    for every test sample while the measurement and reconstruction errors
    decrease as the number of functionals increases.
    \item We demonstrate discretization transfer in a heterogeneous-resolution
    Darcy benchmark. The models are trained using coefficient fields observed
    on $33\times33$, $49\times49$, $65\times65$, and $85\times85$ grids and are
    evaluated on previously unseen $57\times57$, $73\times73$, and
    $97\times97$ grids. The same functionals are evaluated directly on each
    native grid using grid-dependent quadrature, without first interpolating
    every input to a common mesh.
    \item We isolate the role of functional representation through a controlled
    operator with a known task-relevant input subspace. Under a common latent
    dimension, output basis, branch architecture, and test set, we compare the
    proposed measurements with PCA, random projections, fixed and optimized
    sensors, a learned dense bottleneck, decoder variants, and a full-field
    multilayer perceptron.
    \item We provide matched computational comparisons on the antiderivative,
    Darcy, and two-dimensional Navier--Stokes benchmarks. For the fixed-time
    Navier--Stokes problem, the Adaptive Topological DeepONet is the most
    accurate DeepONet-based model, while a parameter-matched Fourier neural
    operator achieves lower error on the uniform periodic grid. This comparison
    clarifies that the principal contribution of the topological formulation is
    compact, interpretable, operator-adapted, and discretization-portable
    functional coordinates rather than universal superiority over
    grid-specific spectral architectures.
\end{enumerate}
\paragraph{Paper organization.}
Section~\ref{sec:def} reviews the required function-space background and
fixes notation.
Section~\ref{sec:methods} presents the fixed and adaptive functional
measurements, the enhanced Two-Step construction, and the training
objectives, together with the complete algorithm and architecture
(Section~\ref{sec:algorithm}).
Section~\ref{sec:theory} establishes a discrete approximation error bound for
Topological DeepONets.
Section~\ref{sec:experiments} describes the benchmark problems, numerical
settings, and results for the antiderivative, Darcy-flow, and
fixed-time and time-evolving Navier--Stokes
problems, and Section~\ref{sec:summary} summarizes the conclusions and
directions for future work.
\begin{figure}[h!]
\centering
\includegraphics[width=\linewidth]{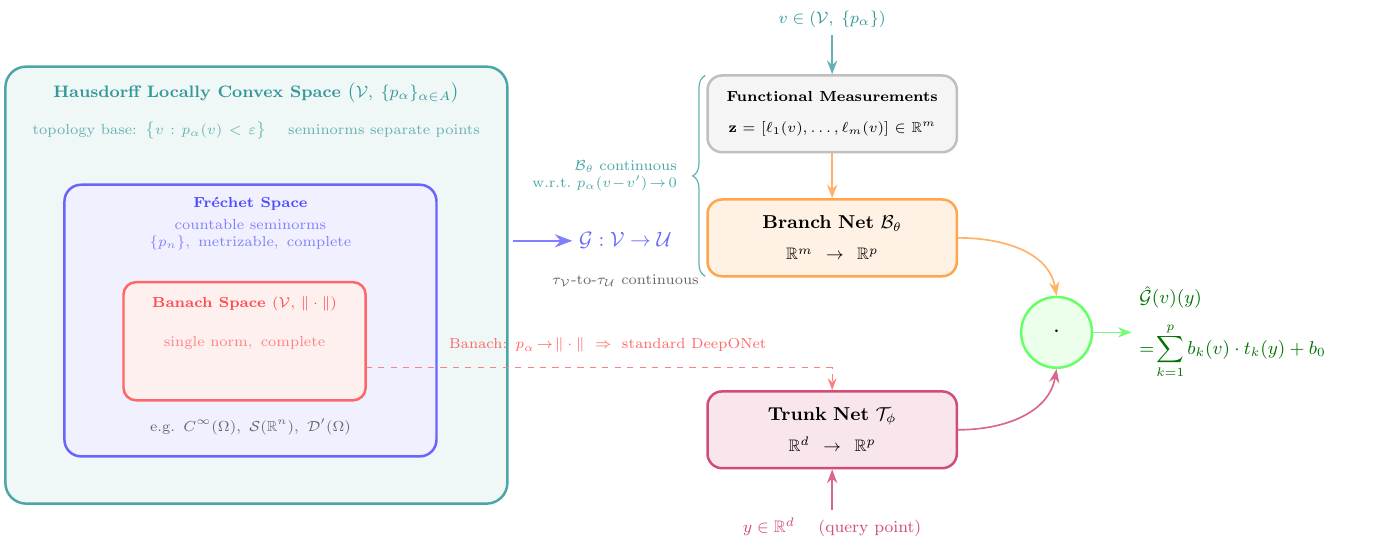}
\caption{Topological DeepONet: operator learning framework
$\mathcal{G}:\mathcal{V}\to\mathcal{U}$ where $\mathcal{V}$ is a Hausdorff
locally convex input space and $\mathcal{U}$ is the Hilbert output space.
\emph{Left}: space hierarchy
$\mathrm{Banach}(\mathcal{V},\|\cdot\|)
 \subset\mathrm{Fr\acute{e}chet}
 \subset\mathrm{HLCS}(\mathcal{V},\{p_\alpha\})$.
\emph{Right}: branch--trunk architecture; the input is encoded by
functional measurements $\ell_j\in\mathcal{V}'$ (point evaluations
being the canonical special case), and continuity of $\mathcal{B}_\theta$
is enforced with respect to the seminorm topology $\{p_\alpha\}$ on
$\mathcal{V}$ rather than a single norm.}
\label{fig:topodeeponet}
\end{figure}
\section{Mathematical Preliminaries}\label{sec:def}
This section collects the functional-analytic background required for the
topological formulation: seminorms and locally convex topologies
(\autoref{subsec:seminorm}), the
Hausdorff separation property (\autoref{subsec:HLCS}), and continuous linear
functionals on such
spaces (\autoref{subsec:cdcf}). Throughout the paper, generic operator inputs are denoted by
$v\in\mathcal{V}$ and outputs by $u\in\mathcal{U}$; in the PDE benchmarks of
Section~\ref{sec:experiments}, where the input is a coefficient or initial
field, we follow the neural-operator literature and write $a$ for the input.
\subsection{Seminorms and Locally Convex Spaces} \label{subsec:seminorm}
A seminorm on a vector space $\mathcal{V}$ is a map
$p:\mathcal{V}\to[0,\infty)$ satisfying
\begin{enumerate}
    \item $p(\lambda v) = |\lambda|\,p(v)$ for all $v\in\mathcal{V}$,
      $\lambda\in\mathbb{R}$,
    \item $p(u+v) \leq p(u) + p(v)$ for all $u,v\in\mathcal{V}$.
\end{enumerate}
Unlike a norm, a seminorm permits $p(v)=0$ for $v\neq 0$. A locally convex space is a vector space $\mathcal{V}$ whose
topology is generated by a family of seminorms
$\{p_\alpha\}_{\alpha\in A}$; a net $v_\lambda\to v$ in this topology
if and only if $p_\alpha(v_\lambda - v)\to 0$ for every $\alpha\in A$.
\subsection{Hausdorff Locally Convex Spaces (HLCS)} \label{subsec:HLCS}
A locally convex space $(\mathcal{V},\{p_\alpha\}_{\alpha\in A})$ is
called Hausdorff (or separated) if the seminorm family
separates points: for every $u\neq v$ in $\mathcal{V}$ there exists
$\alpha\in A$ such that $p_\alpha(u-v)>0$.
This condition ensures that limits of convergent nets are unique.
The class of Hausdorff locally convex spaces (HLCSs) encompasses
Banach spaces (a single norm, complete),
Fr\'{e}chet spaces (countably many seminorms, metrizable, complete),
and spaces of distributions or smooth functions such as
$C^\infty(\Omega)$, $\mathcal{S}(\mathbb{R}^n)$, and
$\mathcal{D}'(\Omega)$.
\subsection{Continuous Dual and Continuous Functionals}\label{subsec:cdcf}
The continuous dual $\mathcal{V}'$ of a locally convex space
$\mathcal{V}$ consists of all continuous linear functionals
$\ell:\mathcal{V}\to\mathbb{R}$.
A linear functional $\ell$ is continuous in the seminorm topology if
and only if there exist finitely many seminorms
$p_{\alpha_1},\ldots,p_{\alpha_k}$ and a constant $C>0$ such that
\[
    |\ell(v)| \leq C\max_{i}\,p_{\alpha_i}(v)
    \quad\text{for all }v\in\mathcal{V}.
\]
Point evaluations $\ell(v)=v(x_0)$ belong to $\mathcal{V}'$ whenever
pointwise evaluation is continuous in the topology of $\mathcal{V}$,
as in $C(\Omega)$ with the uniform norm. More general examples include
inner products $\ell_j(v)=\int_\Omega v(x)\varphi_j(x)\,\mathrm{d}x$
against fixed test functions $\varphi_j\in L^2(\Omega)$.
The Banach-space setting does not, by itself, imply that point
sensors are admissible. The allowable measurements are arbitrary elements
of the continuous dual $\mathcal{V}'$. Point evaluation is continuous on
$C(\Omega)$ equipped with the supremum norm, but it is neither well defined
on equivalence classes in $L^2(\Omega)$ nor continuous in the $L^2$ topology.
Thus, canonical point-sensor DeepONet is a narrower special case obtained
only when point evaluations belong to the chosen continuous dual.
\section{Methodology}\label{sec:methods}
In this section, we first present a detailed comparison between the
conventional DeepONet and the Topological DeepONet
(\autoref{subsub:dnetvtopodnet}). We then introduce the proposed Two-Step
Topological DeepONet (\autoref{subsec:formulation}) and summarize the
complete algorithm and architecture (\autoref{sec:algorithm}).
\subsection{Canonical versus Topological DeepONet}\label{subsub:dnetvtopodnet}
Figure~\ref{fig:topodeeponet} presents the overall architecture of the
proposed Topological DeepONet framework.
The left panel situates the framework within a strict hierarchy of
topological vector spaces.
At the outermost level sits the Hausdorff locally convex space
$(\mathcal{V}, \{p_\alpha\}_{\alpha \in A})$, whose topology is generated
by an arbitrary family of seminorms $\{p_\alpha\}$ that jointly separate
points --- the defining property that replaces the single norm of classical
functional analysis.
Nested within it is the class of Fr\'{e}chet spaces, which impose the
additional regularity that the seminorm family be countable and the space be
complete with respect to the induced metric.
The innermost class, the Banach space $(\mathcal{V}, \|\cdot\|)$,
corresponds to the degenerate case $|A|=1$ in which a single norm suffices;
this is precisely the setting assumed by the original DeepONet of
Lu~et~al.~\cite{lu2021deeponet}.
The right panel depicts the operator learning pipeline: an input function
$v \in (\mathcal{V}, \{p_\alpha\})$ is first encoded by $m$ continuous
functional measurements into the vector
$\mathbf{z} = [\ell_1(v), \ldots, \ell_m(v)] \in \mathbb{R}^m$, which is
processed by the Branch Net $\mathcal{B}_\theta : \mathbb{R}^m \to
\mathbb{R}^p$, while a query point $y \in \mathbb{R}^d$ is processed
independently by the Trunk Net $\mathcal{T}_\phi : \mathbb{R}^d \to
\mathbb{R}^p$.
The operator output is reconstructed as
\begin{equation}
    \hat{\mathcal{G}}(v)(y)
    \;=\;
    \sum_{k=1}^{p} b_k(v)\, t_k(y) \;+\; b_0,
    \label{eq:deeponet-output}
\end{equation}
an inner product between the branch and trunk embeddings.
The dashed arrow connecting the Banach sub-box to the right panel makes
explicit that all existing DeepONet theory is recovered as a strict special
case of this more general construction.
The critical topological subtlety introduced by this framework is indicated
by the brace annotation alongside the Branch Net in
Figure~\ref{fig:topodeeponet}: the composite branch map
$v \mapsto \mathcal{B}_\theta(\mathbf{z}(v))$ is
required to be continuous with respect to the seminorm topology on
$\mathcal{V}$, meaning that $p_\alpha(v - v') \to 0$ for all $\alpha \in A$
must imply
$\|\mathcal{B}_\theta(\mathbf{z}(v)) -
  \mathcal{B}_\theta(\mathbf{z}(v'))\|_{\mathbb{R}^p} \to 0$.
In a Banach space this reduces to ordinary norm continuity, a condition
already implicit in classical DeepONet.
The class of admissible continuous linear functionals depends on the
chosen topology. If $\tau_1\subset\tau_2$ are two locally convex topologies
on the same vector space $\mathcal{V}$, then $\tau_1$ is coarser than
$\tau_2$ and
\[
(\mathcal{V},\tau_1)' \subseteq (\mathcal{V},\tau_2)'.
\]
Thus, coarsening the domain topology generally imposes a stronger continuity
requirement and may reduce the continuous dual, whereas refining the topology
may enlarge the class of continuous linear functionals. Moreover,
nonnormability does not mean that the topology is simply coarser than every
norm topology; rather, it means that the topology cannot be generated by a
single norm. This distinction is relevant when $\mathcal{V}$ contains spaces
such as distributions $\mathcal{D}'(\Omega)$, Schwartz functions
$\mathcal{S}(\mathbb{R}^n)$, or smooth functions $C^\infty(\Omega)$, whose
natural locally convex topologies are generated by families of seminorms.
By grounding the operator $\mathcal{G} : \mathcal{V} \to \mathcal{U}$ in
the $\tau_{\mathcal{V}}$-to-$\tau_{\mathcal{U}}$ continuity of the
\textsc{hlcs} topology rather than in any norm, the Topological DeepONet
extends universal approximation guarantees to these infinite-dimensional,
non-normable spaces.
Figure~\ref{fig:branch-comparison} isolates the central
architectural departure from canonical DeepONet by placing the two branch
input paradigms side by side.
In the left panel, the canonical approach is depicted: the underlying
function $v\in C(\Omega)\subset\mathcal{V}$, shown as a smooth curve, is
sampled at $m$ pre-fixed sensor locations $\{x_i\}_{i=1}^{m}$ (marked as
discrete dots with dashed projection lines to the domain axis), reducing it
to a finite-dimensional vector
$\mathbf{v}=[v(x_1),\ldots,v(x_m)]\in\mathbb{R}^m$ before any learning
takes place.
This discretisation step is inherent to the standard formulation --- the
Branch Net $\mathcal{B}_\theta$ sees only $\mathbb{R}^m$, and its
continuity is measured in the Euclidean norm
$\|\mathbf{v}-\mathbf{v}'\|_{\mathbb{R}^m}$.
The right panel shows the topological alternative: the same function $v$ is
treated as an element of the Hausdorff locally convex space
$(\mathcal{V}, \{p_\alpha\}_{\alpha \in A})$, visualised as a shaded
continuous curve to emphasise that the branch coordinates draw on the
entire functional structure of $v$ rather than on $m$ prescribed pointwise
values.
The branch input is the coordinate vector
$\mathbf{z}_m(v)=[\ell_1(v),\ldots,\ell_m(v)]$, where each
$\ell_j\in\mathcal{V}'$ is continuous with respect to the seminorm family:
$p_\alpha(v - v') \to 0$ for all $\alpha \in A$ forces
$\mathbf{z}_m(v)\to\mathbf{z}_m(v')$ and hence drives the branch outputs
together.
Consequently, the Topological DeepONet requires no prescribed sensor grid:
the same functionals can be approximated by quadrature on whatever
discretization of $v$ is available, and the finite-dimensional encoding is
adapted to the topology of $\mathcal{V}$ rather than tied to a particular
mesh.
Taken together, the two figures articulate both the theoretical scope and
the practical motivation of the proposed framework.
Figure~\ref{fig:topodeeponet} establishes that Banach-space DeepONet
occupies the innermost level of a well-ordered hierarchy of function spaces,
and that the branch--trunk architecture carries over to every level of that
hierarchy upon replacing norm continuity with seminorm continuity.
Figure~\ref{fig:branch-comparison} then makes the departure operationally
concrete: the key change is not in the network architecture itself --- the
branch--trunk decomposition and the inner-product reconstruction
formula~\eqref{eq:deeponet-output} remain intact --- but in
how the branch coordinates are generated: point evaluations tied to
a fixed grid are replaced by continuous functionals compatible with the
intrinsic topology of $\mathcal{V}$.
This change is what enables the Topological DeepONet to handle function
classes inaccessible to norm-based formulations, including spaces of smooth
functions, distributions, and other non-normable locally convex spaces that
arise naturally in the analysis of partial differential equations and
quantum field theory.
\begin{figure}[t]
\centering
\includegraphics[width=\linewidth]{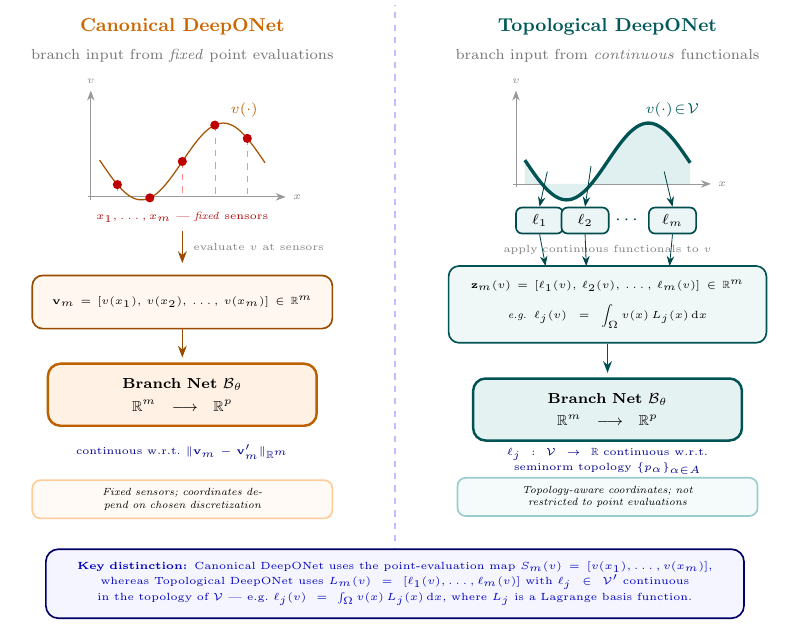}
\caption{
Construction of the branch-network input.
\emph{Left}: canonical DeepONet~\cite{lu2021deeponet} evaluates \(v\) at
\(m\) fixed sensors and forms
\(\mathbf v_m=[v(x_1),\ldots,v(x_m)]\in\mathbb R^m\).
\emph{Right}: Topological DeepONet applies continuous functionals
\(\ell_j:\mathcal V\to\mathbb R\) and forms the topology-aware coordinate
vector
\(\mathbf z_m(v)=[\ell_1(v),\ldots,\ell_m(v)]\in\mathbb R^m\),
for example with \(\ell_j(v)=\int_{\Omega}v(x)L_j(x)\,\mathrm{d}x\) and
\(L_j\) a Lagrange basis function.
Both branch networks receive a finite-dimensional vector, but the
topological construction is not restricted to point evaluations and can
use functionals compatible with the seminorm topology
\(\{p_\alpha\}_{\alpha\in A}\) of \(\mathcal V\).
Here the number of functionals is denoted \(m\) for direct
comparison with the \(m\) point sensors; in
Section~\ref{subsec:formulation} the (possibly smaller) number of learned
functional coordinates is denoted \(q\le m\).
}
\label{fig:branch-comparison}
\end{figure}
\begin{remark}
\cite{ismailov2026topological} also mentions the difference between his proposed approach and the approximation study of \cite{lanthaler2021error}. Here we reiterate these comparison in a simplified way.
Lanthaler et al.~\cite{lanthaler2021error} establish DeepONet universality
for measurable operators in a probabilistic setting, where the approximation
error is measured in a distribution-dependent $L^2$ norm. This permits
noncompact input domains and weaker regularity assumptions, but provides an
average rather than uniform guarantee. In contrast, Ismailov~
\cite{ismailov2026topological} considers continuous mappings on compact subsets
of locally convex topological vector spaces and constructs neural networks
using continuous linear functionals. His result therefore extends
approximation theory beyond Banach input spaces while retaining uniform
approximation. Thus, the two results are complementary: Lanthaler et al.\
generalize the admissible operators and domains, whereas Ismailov generalizes
the topology of the input space. A succinct comparison is shown in \autoref{tab:lanthaler-ismailov}.
\begin{table}[t]
\centering
\caption{Comparison of  approximation frameworks: Lanthaler et al. \cite{lanthaler2021error} vs Ismailov  \cite{ismailov2026topological}.}
\label{tab:lanthaler-ismailov}
\renewcommand{\arraystretch}{1.1}
\begin{tabular}{lll}
\toprule
\textbf{Aspect} & \textbf{Lanthaler et al. \cite{lanthaler2021error}} & \textbf{Ismailov \cite{ismailov2026topological}} \\
\midrule
Operator & Measurable & Continuous \\
Input domain & Possibly noncompact & Compact subset \\
Input space & Probabilistic Banach/Hilbert space & Locally convex space \\
Error metric & $L^2$ average error & Uniform $L^\infty$ error \\
\bottomrule
\end{tabular}
\end{table}
\end{remark}
\subsection{Two-Step Adaptive Topological DeepONet}
\label{subsec:formulation}
\subsubsection{Operator and function spaces}
Let $\mathcal{V}$ be a Hausdorff locally convex topological vector space
and let $\mathcal{U}$ be a separable Hilbert space with inner product
$\langle\cdot,\cdot\rangle_{\mathcal{U}}$ and induced norm
$\|\cdot\|_{\mathcal{U}}$.
We seek to approximate a continuous nonlinear operator
\begin{equation}
  \mathcal{G} : \mathcal{V} \rightarrow \mathcal{U},
  \label{eq:operator}
\end{equation}
from a finite training set
$\mathcal{D}_N = \{(v_i, u_i)\}_{i=1}^{N}$, where $u_i = \mathcal{G}(v_i)$.
No norm is assumed on $\mathcal{V}$; only a finite family of continuous
observations is required for the computational realisation.
\subsubsection{Base observation map}
Let $\lambda_1, \ldots, \lambda_m \in \mathcal{V}'$ be continuous linear
functionals on $\mathcal{V}$.
The \emph{base observation map} is
\begin{equation}
  \Lambda_m : \mathcal{V} \rightarrow \mathbb{R}^m, \qquad
  \Lambda_m(v) = \bigl[\lambda_1(v),\, \ldots,\, \lambda_m(v)\bigr]^\top.
  \label{eq:base-obs}
\end{equation}
For function spaces on which point evaluation is continuous, such as
$C(\overline{\Omega})$ with the supremum norm, one may take
$\lambda_j(v)=v(x_j)$.
When point evaluation is not well-defined or not continuous, each
$\lambda_j$ may instead be a weak measurement
$\lambda_j(v)=\langle v,\varphi_j\rangle$ for a test function
$\varphi_j$, a local average, or a finite-element degree of freedom.
In the Darcy and Navier--Stokes experiments, point-sensor models are
used only as discrete numerical baselines after spatial discretization. They
are not interpreted as continuum elements of $L^2(\Omega)'$.
The training observation matrix is
\begin{equation}
  S = \bigl[\Lambda_m(v_1),\, \ldots,\, \Lambda_m(v_N)\bigr]^\top
    \in \mathbb{R}^{N \times m}.
\end{equation}
\begin{myalgo}{Fixed and Adaptive Topological DeepONet}
\label{alg:adaptive-topo-deeponet}
\AlgIn{
Training data $\mathcal{D}_N=\{(v_i,u_i)\}_{i=1}^N$;
base observations $S\in\mathbb{R}^{N\times m}$;
outputs $U\in\mathbb{R}^{n_y\times N}$;
initial measurement matrix $M_0\in\mathbb{R}^{m\times q}$;
mode $\in\{\textsc{Fixed},\textsc{Adaptive}\}$
}
\AlgOut{Operator surrogate $\widehat{\mathcal G}$}
\AlgSec{Stage I: output basis}
\begin{enumerate}[label=\arabic*.,leftmargin=2em,itemsep=1pt,parsep=0pt]
\item Initialize trunk $\bm{\phi}_\mu$ and free coefficients
$A\in\mathbb{R}^{p_0\times N}$.
\item Minimize
\[
\mathcal L_{\mathrm I}
=
\frac{
\|W_y^{1/2}(\Phi_\mu A-U)\|_F^2
}{
N\,\operatorname{tr}(W_y)
}
\]
with respect to $(\mu,A)$.
\item Compute
\[
W_y^{1/2}\Phi_{\mu^\star}=P\Sigma V^\top,
\qquad
r=\#\{\sigma_k>\tau_{\mathrm{rank}}\sigma_1\}.
\]
\item Set
\[
Q=W_y^{-1/2}P_r,
\qquad
C^\star=\Sigma_rV_r^\top A^\star,
\qquad
\boldsymbol{\psi}(y)=\Sigma_r^{-1}V_r^\top
\bm{\phi}_{\mu^\star}(y),
\]
and freeze $\mu^\star$.
\end{enumerate}
\AlgSec{Stage II: branch and measurement training}
\begin{enumerate}[label=\arabic*.,leftmargin=2em,itemsep=1pt,parsep=0pt,start=5]
\item Set
\[
M=
\begin{cases}
M_0, & \textsc{Fixed},\\
\text{trainable }M_0, & \textsc{Adaptive}.
\end{cases}
\]
Initialize
$\mathcal B_\theta:\mathbb{R}^q\to\mathbb{R}^r$.
For the adaptive model, also initialize the training-only decoder
$D_\omega:\mathbb{R}^q\to\mathbb{R}^m$.
\item Compute and store normalization statistics from $SM$.
\item \kwb{for} each epoch \kwb{do}
\begin{enumerate}[label*=\arabic*.,leftmargin=2em,itemsep=1pt,parsep=0pt]
\item Draw a mini-batch, optionally using hard-example probabilities.
\item Form
\[
\widetilde Z_{\mathcal B}
=
\operatorname{standardize}(S_{\mathcal B}M),
\qquad
\widehat C_{\mathcal B}
=
\mathcal B_\theta(\widetilde Z_{\mathcal B}).
\]
\item Minimize
\[
\mathcal L=
\begin{cases}
\mathcal L_c,
& \textsc{Fixed},\\[1mm]
\mathcal L_c
+\lambda_{\mathrm{rec}}\mathcal L_{\mathrm{rec}}
+\lambda_{\mathrm{orth}}\mathcal L_{\mathrm{orth}}
+\lambda_{\mathrm{drift}}\mathcal L_{\mathrm{drift}},
& \textsc{Adaptive},
\end{cases}
\]
with losses defined in
\cref{eq:lc,eq:lrec,eq:lorth-ldrift}.
\item Update $\theta$; for the adaptive model also update $M$ and $\omega$.
Retain the best validation checkpoint.
\end{enumerate}
\item \kwb{end for}
\item Discard $D_\omega$ after adaptive training.
\end{enumerate}
\AlgSec{Inference}
\begin{enumerate}[label=\arabic*.,leftmargin=2em,itemsep=1pt,parsep=0pt,start=9]
\item For a new input $v$, compute
\[
\mathbf s=\Lambda_m(v),
\qquad
\widetilde{\mathbf z}
=
\operatorname{standardize}(M^{\star\top}\mathbf s),
\qquad
\widehat{\mathbf c}
=
\mathcal B_{\theta^\star}(\widetilde{\mathbf z}).
\]
\item \textbf{return}
\[
\widehat{\mathcal G}(v)(y)
=
\boldsymbol{\psi}(y)^\top\widehat{\mathbf c}.
\]
\end{enumerate}
\end{myalgo}
\subsubsection{Learned measurement map}
\label{sec:learned-meas}
A conventional DeepONet passes $\Lambda_m(v)$ directly to the branch
network.
Instead, we learn a compact linear map $M \in \mathbb{R}^{m \times q}$
($q \leq m$) that produces \emph{operator-adapted} features:
\begin{equation}
  M^\top \Lambda_m(v) \in \mathbb{R}^q, \qquad
  \ell_k^M(v) = \sum_{j=1}^{m} M_{jk}\,\lambda_j(v), \quad k = 1,\ldots,q.
  \label{eq:learned-meas}
\end{equation}
Each $\ell_k^M$ is a finite linear combination of elements of
$\mathcal{V}'$, hence continuous on $\mathcal{V}$.
$M$ is initialised from a structured weak-measurement basis (e.g.\ cosine
or polynomial test functions) and adapted during training.
\subsubsection{Output basis: Stage~I}
\label{sec:stage1}
Let $y_1, \ldots, y_{n_y}$ be aligned output locations in the output
domain $\Omega_y\subset\mathbb{R}^d$,
$W_y = \mathrm{diag}(w_1^y,\ldots,w_{n_y}^y) \succ 0$ a quadrature weight
matrix, and $U \in \mathbb{R}^{n_y \times N}$ the snapshot matrix whose
$i$-th column stores $u_i$ evaluated on the grid.
We learn a trunk network $\bm{\phi}_\mu : \Omega_y \rightarrow
\mathbb{R}^{p_0}$, with output width $p_0\in\mathbb{N}$, by solving
\begin{equation}
  (\mu^\star, A^\star)
  \in \operatorname*{arg\,min}_{\mu,\, A \in \mathbb{R}^{p_0 \times N}}
  \frac{1}{N\,\mathrm{tr}(W_y)}
  \left\|W_y^{1/2}\!\left(\Phi_\mu A - U\right)\right\|_F^2,
  \label{eq:stage1}
\end{equation}
where $\Phi_\mu \in \mathbb{R}^{n_y \times p_0}$ stacks the trunk
evaluations on the grid.
The free coefficient matrix $A$ decouples output-subspace learning from
input regression.
After training, a rank-revealing weighted SVD
\begin{equation}
  W_y^{1/2}\Phi_{\mu^\star} = P\Sigma V^\top
  \label{eq:svd}
\end{equation}
yields the \emph{weighted-orthonormal basis}
\begin{equation}
  Q = W_y^{-1/2} P_r \in \mathbb{R}^{n_y \times r}, \qquad
  Q^\top W_y Q = I_r,
  \label{eq:Q}
\end{equation}
where $r = \#\{k : \sigma_k > \tau_{\mathrm{rank}}\,\sigma_1\}$ is the
numerical rank determined by a prescribed tolerance
$\tau_{\mathrm{rank}}\in(0,1)$.
The Stage~II coefficient targets are
$C^\star = \Sigma_r V_r^\top A^\star \in \mathbb{R}^{r \times N}$.
Because $Q$ is weighted-orthonormal, output error and coefficient error are
isometric: for any $a, b \in \mathbb{R}^r$,
\begin{equation}
  \|Q(a - b)\|_{W_y} = \|a - b\|_2.
  \label{eq:isometry}
\end{equation}
\subsubsection{Coefficient prediction: Stage~II}
\label{sec:stage2}
With $Q$ and $C^\star$ fixed, the Branch Net
$\mathcal{B}_\theta : \mathbb{R}^q \rightarrow \mathbb{R}^r$ and the
measurement matrix $M$ are trained jointly.
Standardised measurements
$\tilde{\mathbf{z}}(v)
 = \mathrm{diag}(\bm{\sigma}_z)^{-1}(M^\top\Lambda_m(v) - \bm{\mu}_z)$
serve as input, where $\bm{\mu}_z\in\mathbb{R}^q$ and
$\bm{\sigma}_z\in\mathbb{R}^q$ denote the componentwise mean and standard
deviation of the training features, and the prediction is
\begin{equation}
  \hat{\mathcal{G}}(v)(y)
  = \boldsymbol{\psi}(y)^\top
    \mathcal{B}_\theta\!\left(\tilde{\mathbf{z}}(v)\right),
  \label{eq:prediction}
\end{equation}
where $\boldsymbol{\psi}(y)^\top
= \bm{\phi}_{\mu^\star}(y)^\top V_r \Sigma_r^{-1}$ evaluates the basis at
any query point $y$.
Note that \autoref{eq:prediction} is the adaptive realisation of
\autoref{eq:deeponet-output}, with $b_k(v) = \mathcal{B}_\theta(
\tilde{\mathbf{z}}(v))_k$, $t_k(y) = \psi_k(y)$, and $b_0 = 0$ (absorbed
into the basis).
\paragraph{Training objective.}
The Stage~II loss combines coefficient regression with three regularisers:
\begin{equation}
  \mathcal{L} =
  \underbrace{\mathcal{L}_c}_{\text{coeff.\ loss}}
  + \lambda_{\mathrm{rec}}\,\underbrace{\mathcal{L}_{\mathrm{rec}}}_{\text{reconstruction}}
  + \lambda_{\mathrm{orth}}\,\underbrace{\mathcal{L}_{\mathrm{orth}}}_{\text{soft ortho.}}
  + \lambda_{\mathrm{drift}}\,\underbrace{\mathcal{L}_{\mathrm{drift}}}_{\text{drift}}.
  \label{eq:loss}
\end{equation}
The coefficient loss is
\begin{equation}
  \mathcal{L}_c = \frac{1}{Br}\sum_{i \in \mathcal{B}}
  \|\hat{\mathbf{c}}_i - \mathbf{c}_i^\star\|_2^2
  + \lambda_{\mathrm{rel}}\,\frac{1}{B}\sum_{i \in \mathcal{B}} e_i^c
  + \lambda_{\mathrm{worst}}\,\tau_c \log\!\left(
    \frac{1}{B}\sum_{i \in \mathcal{B}}
    \exp\!\left(e_i^c / \tau_c\right)
  \right),
  \label{eq:lc}
\end{equation}
where
$e_i^c = \|\hat{\mathbf{c}}_i - \mathbf{c}_i^\star\|_2 /
(\|\mathbf{c}_i^\star\|_2 + \varepsilon_c)$
is the per-sample relative error, $\mathcal{B}$ is a mini-batch of
size $B$, $\lambda_{\mathrm{rel}},\lambda_{\mathrm{worst}}\ge 0$
are fixed weights, $\tau_c>0$ is the temperature of the log-sum-exp
worst-case term, and $\varepsilon_c>0$ is a small constant preventing
division by zero.
A training-only decoder $D_\omega : \mathbb{R}^q \rightarrow \mathbb{R}^m$
reconstructs the base observations to prevent information collapse:
\begin{equation}
  \mathcal{L}_{\mathrm{rec}} =
  \frac{\dfrac{1}{Bm}\displaystyle\sum_{i \in \mathcal{B}}
    \|D_\omega(\tilde{\mathbf{z}}_i) - \mathbf{s}_i\|_2^2}
  {\dfrac{1}{Bm}\displaystyle\sum_{i \in \mathcal{B}}
    \|\mathbf{s}_i\|_2^2 + \varepsilon_s},
  \label{eq:lrec}
\end{equation}
where $\mathbf{s}_i=\Lambda_m(v_i)$ denotes the base observation
vector of the $i$-th sample and $\varepsilon_s>0$ is a small constant.
The soft weighted-orthogonality and drift penalties are
\begin{equation}
  \mathcal{L}_{\mathrm{orth}}
  = \frac{1}{q^2}\|M^\top W_x^{-1} M - I_q\|_F^2,
  \qquad
  \mathcal{L}_{\mathrm{drift}}
  = \frac{\|M - M_0\|_F^2}{\|M_0\|_F^2 + \varepsilon_M},
  \label{eq:lorth-ldrift}
\end{equation}
where $M_0$ is the structured initialisation, $W_x$ is the sensor
quadrature matrix, and $\varepsilon_M>0$ is a small constant.
The decoder is discarded at inference; the deployed model consists solely
of $M$, $\mathcal{B}_\theta$, and the frozen basis $Q$.
\subsection{Algorithm and Architecture} \label{sec:algorithm}
The complete training and inference procedures for the fixed and adaptive
topological DeepONets are summarized in
\autoref{alg:adaptive-topo-deeponet}. Detailed architectural block diagrams
of both variants are presented in
\autoref{fig:fixed-adaptive-topological-architecture}, which summarizes the
full data flow of the Adaptive
Topological DeepONet, from input observations $\Lambda_m(v)$ through the
learned measurement map and Branch Net $\mathcal{B}_\theta$ to the final
coefficient-weighted output $\hat{\mathcal{G}}(v)(y)$.
\begin{figure}[t]
\centering
\includegraphics[width=0.98\textwidth]{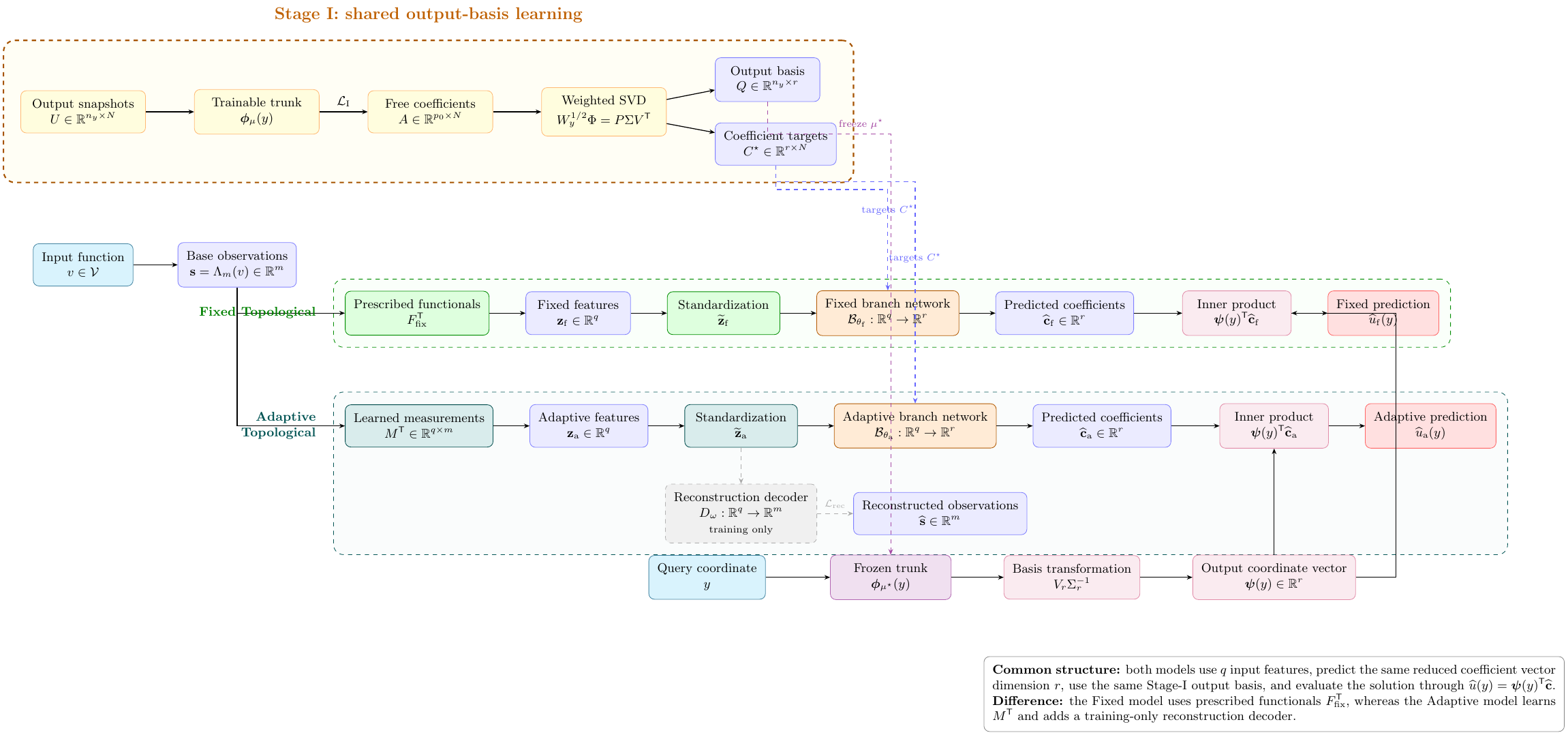}
\caption{%
Unified architecture of the Fixed and Adaptive Topological DeepONets.
Stage~I learns a common reduced output basis from the output snapshots.
In Stage~II, the Fixed model constructs \(q\) features using prescribed
functionals \(F_{\mathrm{fix}}^{\mathsf T}\), whereas the Adaptive model
learns the measurement matrix \(M^{\mathsf T}\). The adaptive pathway also
contains a training-only decoder \(D_{\omega}\), regularised through
\(\mathcal L_{\mathrm{rec}}\). Both models predict reduced coefficients
and evaluate the output continuously as
\(\widehat u(y)=\boldsymbol{\psi}(y)^{\mathsf T}\widehat{\mathbf c}\).%
}
\label{fig:fixed-adaptive-topological-architecture}
\end{figure}
\section{Discrete Error Bounds for Topological DeepONets}
\label{sec:theory}
The universal approximation result of
Ismailov~\cite{ismailov2026topological} establishes that continuous
operators on compact subsets of Hausdorff locally convex spaces can be
approximated uniformly by branch--trunk expansions whose branch networks
depend on finitely many continuous linear functionals from the continuous
dual. The result is qualitative and does not distinguish the different
sources of approximation error. We therefore derive a discrete error bound
that separates the effects of input measurement compression, output-basis
truncation, and neural-network approximation.
Let
\[
K_h\subset V_h\subset\mathbb{R}^{m}
\]
be a compact set of discrete inputs, and let
\[
\mathcal{G}_h:V_h\longrightarrow\mathbb{R}^{n}
\]
denote the corresponding discrete solution operator. Let
\[
M_q:\mathbb{R}^{m}\longrightarrow\mathbb{R}^{q},
\qquad
M_qv_h
=
\begin{bmatrix}
\ell_{1,h}(v_h)&\cdots&\ell_{q,h}(v_h)
\end{bmatrix}^{\top},
\]
where
\(\ell_{j,h}\in(\mathbb{R}^{m})^\ast\) are discrete continuous linear
functionals. Let
\[
R_q:\mathbb{R}^{q}\longrightarrow\mathbb{R}^{m}
\]
be a continuous reconstruction map.
The map $R_q$ is an auxiliary reconstruction used only in the error
analysis to quantify information lost by the measurement map. It is distinct
from the training-only decoder $D_\omega$ in Section~\ref{sec:stage2}, which
reconstructs the finite base-observation vector
$\bm s=\Lambda_m(v)$. The theorem therefore does not assume that the trained
decoder reconstructs the full continuum input. When the base observations
form a stable discrete representation of $v_h$, a reconstruction of $v_h$
may be composed with the decoded observation vector; otherwise these two
reconstruction notions should be kept separate.
To match the weighted output basis used in the algorithm, let
$W_y\succ0$ denote the output quadrature matrix and introduce preweighted
coordinates
$\widetilde{\bm g}_h=W_y^{1/2}\bm g_h$ and
$\widetilde Q_r=W_y^{1/2}Q_r$. Since the algorithm satisfies
$Q_r^{\top}W_yQ_r=I_r$, the transformed basis obeys
$\widetilde Q_r^{\top}\widetilde Q_r=I_r$. The theorem below is written in
these preweighted Euclidean coordinates; equivalently, every output norm may
be read as the weighted norm
$\|v\|_{W_y}=\|W_y^{1/2}v\|_2$.
For the output representation, let
\[
\overline{\bm g}_h\in\mathbb{R}^{n},
\qquad
Q_r\in\mathbb{R}^{n\times r},
\qquad
Q_r^{\top}Q_r=I_r,
\]
where $Q_r$ and $\overline{\bm g}_h$ now denote the corresponding
preweighted quantities.
The corresponding Topological DeepONet approximation is
\begin{equation}
\widehat{\mathcal{G}}_{h,\theta}(v_h)
=
\overline{\bm g}_h
+
Q_rb_\theta(M_qv_h),
\label{eq:discrete-topological-approximation}
\end{equation}
where
\[
b_\theta:\mathbb{R}^{q}\longrightarrow\mathbb{R}^{r}
\]
is the branch network.
\begin{figure}[t]
    \centering
    \begin{subfigure}[t]{0.48\textwidth}
        \centering
        \includegraphics[width=\linewidth]
        {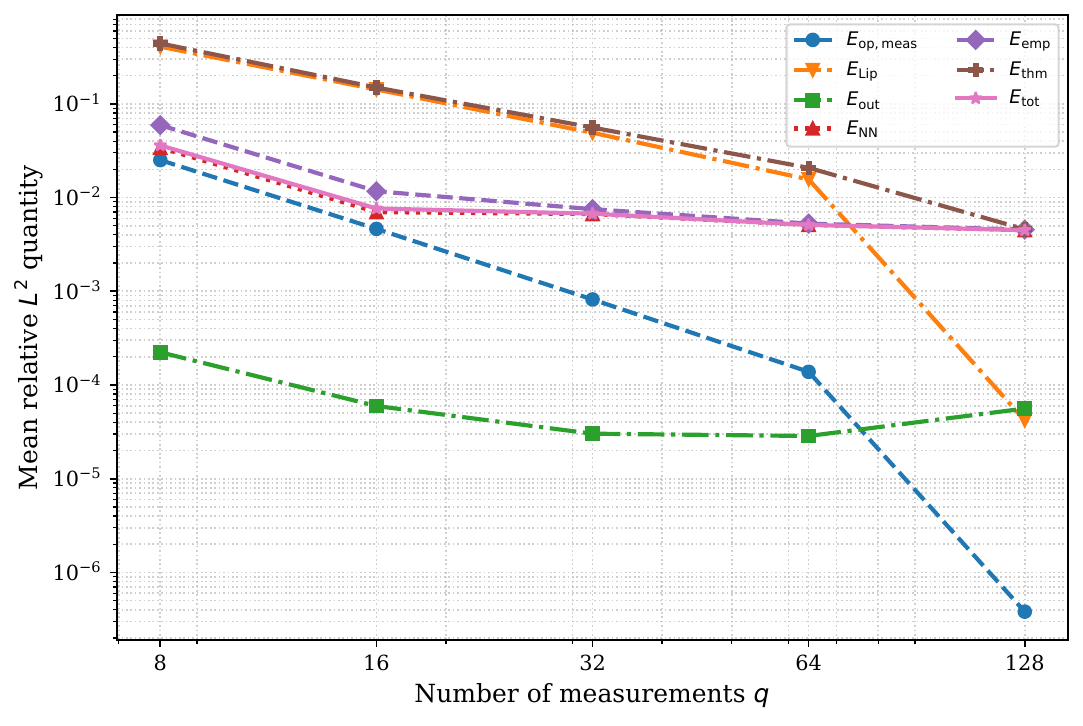}
        \caption{Fixed-model error budget.}
    \end{subfigure}
    \hfill
    \begin{subfigure}[t]{0.48\textwidth}
        \centering
        \includegraphics[width=\linewidth]
        {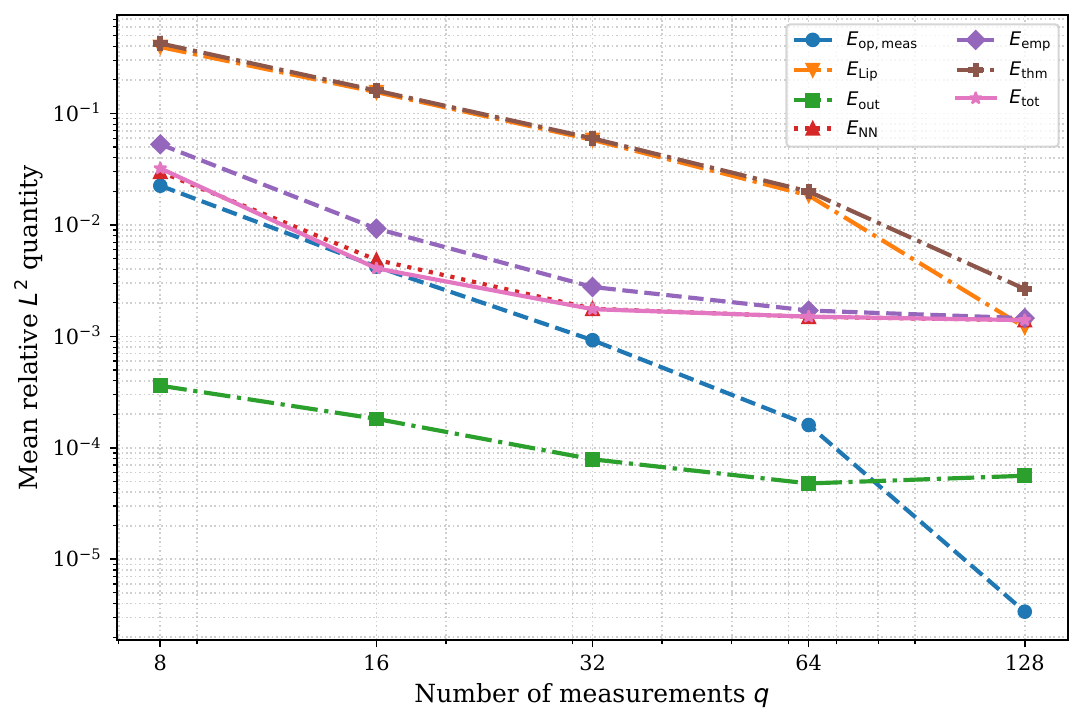}
        \caption{Adaptive-model error budget.}
    \end{subfigure}
    \vspace{0.5em}
    \begin{subfigure}[t]{0.48\textwidth}
        \centering
        \includegraphics[width=\linewidth]
        {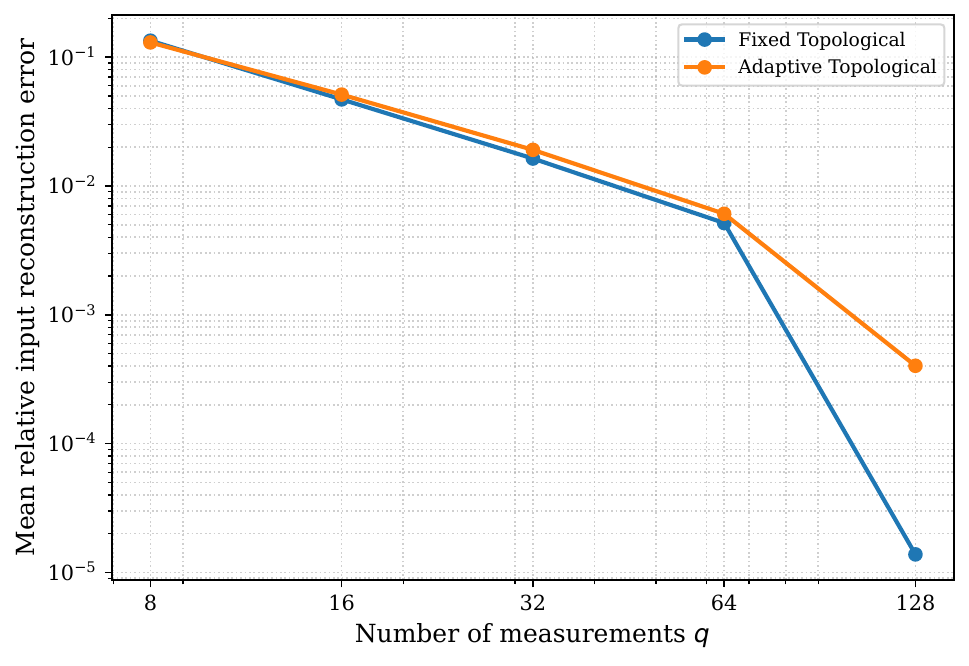}
        \caption{Input reconstruction convergence.}
    \end{subfigure}
    \hfill
    \begin{subfigure}[t]{0.48\textwidth}
        \centering
        \includegraphics[width=\linewidth]      {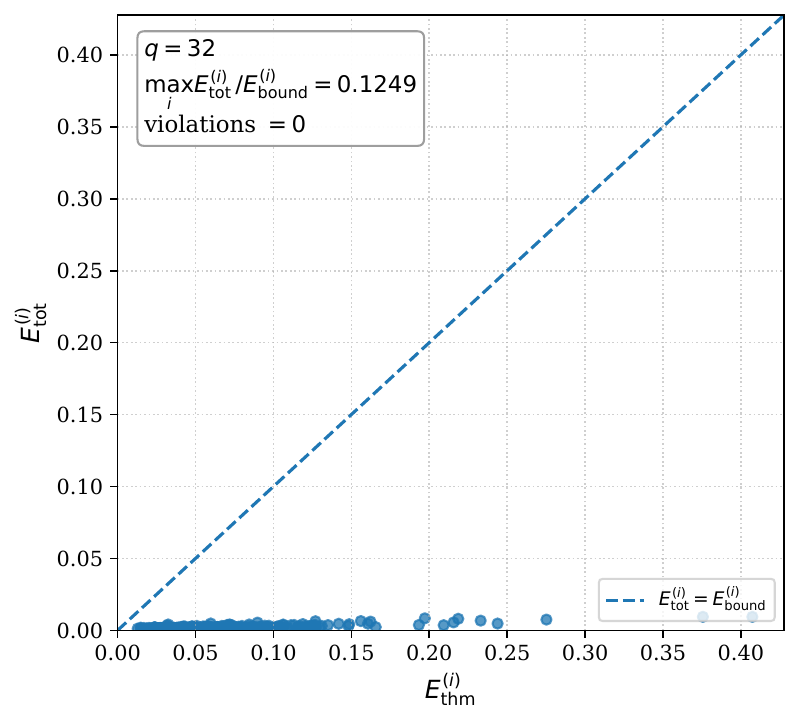}
        \caption{Samplewise bound verification at \(q=32\).}
    \end{subfigure}
    \caption{
    Numerical verification of the discrete Topological DeepONet error
    decomposition for the antiderivative operator. Panels (a) and (b) show
    the operator-level measurement error \(E_{\mathrm{op,meas}}\), its
    Lipschitz upper bound \(E_{\mathrm{Lip}}\), the output truncation error
    \(E_{\mathrm{out}}\), neural approximation error \(E_{\mathrm{NN}}\),
    empirical bound \(E_{\mathrm{emp}}\), theorem-consistent bound
    \(E_{\mathrm{thm}}\), and total error \(E_{\mathrm{tot}}\). Panel (c)
    shows that the input reconstruction error decreases as the measurement
    dimension \(q\) increases. Panel (d) verifies the samplewise inequality
    \(E_{\mathrm{tot}}^{(i)}\leq E_{\mathrm{thm}}^{(i)}\); at \(q=32\), the
    maximum ratio is \(0.1249\), with zero violations.
    }
    \label{fig:complete-theorem-verification}
\end{figure}
\begin{theorem}[Discrete Topological DeepONet error decomposition]
\label{thm:discrete-topological-error}
Let \(K_h\subset V_h\subset\mathbb{R}^{m}\) be compact and assume that
\[
R_q(M_qK_h)\subset V_h.
\]
Suppose that \(\mathcal{G}_h\) is Lipschitz continuous on
\[
K_h\cup R_q(M_qK_h)
\]
with constant \(L_h>0\), namely,
\begin{equation}
\|\mathcal{G}_h(v_h)-\mathcal{G}_h(w_h)\|_2
\leq
L_h\|v_h-w_h\|_2
\label{eq:discrete-operator-lipschitz}
\end{equation}
for all
\(v_h,w_h\in K_h\cup R_q(M_qK_h)\).
Define the measurement--reconstruction error by
\begin{equation}
\varepsilon_{\mathrm{rec}}(q)
:=
\sup_{v_h\in K_h}
\|v_h-R_q(M_qv_h)\|_2,
\label{eq:measurement-reconstruction-error}
\end{equation}
and define the output-basis truncation error by
\begin{equation}
\varepsilon_{\mathrm{out}}(r,q)
:=
\sup_{v_h\in K_h}
\left\|
(I_n-Q_rQ_r^{\top})
\left[
\mathcal{G}_h(R_q(M_qv_h))
-
\overline{\bm g}_h
\right]
\right\|_2.
\label{eq:output-truncation-error}
\end{equation}
Assume that the activation function in \(b_\theta\) is continuous and
nonpolynomial. Then, for every
\(\varepsilon_{\mathrm{NN}}>0\), there exists a feed-forward neural network
\(b_\theta:\mathbb{R}^{q}\to\mathbb{R}^{r}\) such that
\begin{equation}
\boxed{
\sup_{v_h\in K_h}
\left\|
\mathcal{G}_h(v_h)
-
\widehat{\mathcal{G}}_{h,\theta}(v_h)
\right\|_2
\leq
L_h\varepsilon_{\mathrm{rec}}(q)
+
\varepsilon_{\mathrm{out}}(r,q)
+
\varepsilon_{\mathrm{NN}}.
}
\label{eq:discrete-topological-error-bound}
\end{equation}
\end{theorem}
\begin{proof}
Because \(K_h\) is compact and \(M_q\) is continuous, the measurement set
\[
Z_q:=M_qK_h\subset\mathbb{R}^{q}
\]
is compact. Define the reduced coefficient map
\[
c_{r,q}:Z_q\longrightarrow\mathbb{R}^{r}
\]
by
\begin{equation}
c_{r,q}(z)
:=
Q_r^{\top}
\left[
\mathcal{G}_h(R_q(z))
-
\overline{\bm g}_h
\right].
\label{eq:reduced-coefficient-map}
\end{equation}
The continuity of \(R_q\), \(\mathcal{G}_h\), and \(Q_r^{\top}\) implies
that \(c_{r,q}\) is continuous on \(Z_q\). Therefore, the finite-dimensional
universal approximation theorem implies that, for every
\(\varepsilon_{\mathrm{NN}}>0\), there exists a feed-forward neural network
\(b_\theta\) satisfying
\begin{equation}
\sup_{z\in Z_q}
\|c_{r,q}(z)-b_\theta(z)\|_2
\leq
\varepsilon_{\mathrm{NN}}.
\label{eq:branch-network-approximation}
\end{equation}
Fix \(v_h\in K_h\) and set
\[
\widetilde v_h:=R_q(M_qv_h).
\]
Adding and subtracting
\(\mathcal{G}_h(\widetilde v_h)\) and the orthogonal projection of the
centered reconstructed output gives
\begin{align}
&
\mathcal{G}_h(v_h)
-
\overline{\bm g}_h
-
Q_rb_\theta(M_qv_h)
\nonumber\\
&=
\underbrace{
\mathcal{G}_h(v_h)-\mathcal{G}_h(\widetilde v_h)
}_{\mathrm{(I)}}
\nonumber\\
&\quad+
\underbrace{
(I_n-Q_rQ_r^{\top})
\left[
\mathcal{G}_h(\widetilde v_h)-\overline{\bm g}_h
\right]
}_{\mathrm{(II)}}
\nonumber\\
&\quad+
\underbrace{
Q_r
\left[
c_{r,q}(M_qv_h)-b_\theta(M_qv_h)
\right]
}_{\mathrm{(III)}}.
\label{eq:three-term-error-decomposition}
\end{align}
By the Lipschitz continuity of \(\mathcal{G}_h\),
\begin{align}
\|\mathrm{(I)}\|_2
&\leq
L_h\|v_h-\widetilde v_h\|_2
\nonumber\\
&\leq
L_h\varepsilon_{\mathrm{rec}}(q).
\end{align}
By definition,
\[
\|\mathrm{(II)}\|_2
\leq
\varepsilon_{\mathrm{out}}(r,q).
\]
Since \(Q_r^{\top}Q_r=I_r\), the matrix \(Q_r\) is an isometry on
\(\mathbb{R}^{r}\), and hence
\begin{align}
\|\mathrm{(III)}\|_2
&=
\left\|
c_{r,q}(M_qv_h)-b_\theta(M_qv_h)
\right\|_2
\nonumber\\
&\leq
\varepsilon_{\mathrm{NN}}.
\end{align}
Applying the triangle inequality and taking the supremum over
\(v_h\in K_h\) proves
\eqref{eq:discrete-topological-error-bound}.
\end{proof}
\begin{remark}[Interpretation of the error bound]
\label{rem:error-interpretation}
Theorem~\ref{thm:discrete-topological-error} separates the total
approximation error into
\begin{equation}
\underbrace{
L_h\varepsilon_{\mathrm{rec}}(q)
}_{\text{measurement information loss}}
+
\underbrace{
\varepsilon_{\mathrm{out}}(r,q)
}_{\text{output-basis truncation}}
+
\underbrace{
\varepsilon_{\mathrm{NN}}
}_{\text{branch-network approximation}}.
\label{eq:error-bound-interpretation}
\end{equation}
Increasing the number of measurements \(q\) can reduce the first term,
increasing the output rank \(r\) can reduce the second term, and increasing
the approximation capacity of the branch network can reduce the third term.
The theorem does not imply that these terms decrease monotonically for an
arbitrary choice of measurements, basis, or training procedure.
\end{remark}
\begin{corollary}[Barron-rate refinement]
\label{cor:barron-rate}
Let \(\nu\) be a probability measure on \(K_h\), and let
\[
\mu_q:=(M_q)_{\#}\nu
\]
be its pushforward measure on
\[
Z_q=M_qK_h\subset B_{r_0}(0)\subset\mathbb{R}^{q}.
\]
Assume the hypotheses of
Theorem~\ref{thm:discrete-topological-error}. Suppose that every component
\(c_{r,q,j}\) of the reduced map \(c_{r,q}\) admits an extension
\[
g_j:\mathbb{R}^{q}\longrightarrow\mathbb{R}
\]
with finite Barron norm
\[
C_j:=\|g_j\|_{\mathcal B}.
\]
Define
\[
C_{r,q}
:=
\left(
\sum_{j=1}^{r}C_j^2
\right)^{1/2}.
\]
Then, for every \(N\in\mathbb{N}\), there exists a two-layer sigmoidal
network
\[
b_\theta:\mathbb{R}^{q}\longrightarrow\mathbb{R}^{r},
\]
with at most \(N\) hidden units for each output component, such that
\begin{equation}
\boxed{
\begin{aligned}
&
\left(
\mathbb{E}_{v_h\sim\nu}
\left\|
\mathcal{G}_h(v_h)
-
\overline{\bm g}_h
-
Q_rb_\theta(M_qv_h)
\right\|_2^2
\right)^{1/2}
\\
&\qquad\leq
L_h\varepsilon_{\mathrm{rec}}(q)
+
\varepsilon_{\mathrm{out}}(r,q)
+
\frac{2r_0C_{r,q}}{\sqrt{N}}.
\end{aligned}
}
\label{eq:barron-refined-error-bound}
\end{equation}
\end{corollary}
\begin{proof}
For \(v_h\in K_h\), use the decomposition
\eqref{eq:three-term-error-decomposition}. The first two terms satisfy the
pointwise estimates
\[
\|\mathrm{(I)}\|_2
\leq
L_h\varepsilon_{\mathrm{rec}}(q),
\qquad
\|\mathrm{(II)}\|_2
\leq
\varepsilon_{\mathrm{out}}(r,q),
\]
and therefore obey the same bounds in
\(L^2(\nu;\mathbb{R}^{n})\).
Applying Barron's approximation theorem componentwise gives functions
\(g_{N,j}\), each represented by a sum of \(N\) sigmoidal ridge functions,
such that
\begin{equation}
\mathbb{E}_{z\sim\mu_q}
\left[
\left|
g_j(z)-g_{N,j}(z)
\right|^2
\right]
\leq
\frac{4r_0^2C_j^2}{N}.
\end{equation}
Define
\[
b_\theta
:=
(g_{N,1},\ldots,g_{N,r}).
\]
Because \(Q_r^{\top}Q_r=I_r\),
\begin{align}
\mathbb{E}_{v_h\sim\nu}
\|\mathrm{(III)}\|_2^2
&=
\sum_{j=1}^{r}
\mathbb{E}_{z\sim\mu_q}
\left[
\left|
c_{r,q,j}(z)-g_{N,j}(z)
\right|^2
\right]
\nonumber\\
&\leq
\frac{4r_0^2}{N}
\sum_{j=1}^{r}C_j^2
=
\frac{4r_0^2C_{r,q}^2}{N}.
\end{align}
Hence,
\[
\|\mathrm{(III)}\|_{L^2(\nu)}
\leq
\frac{2r_0C_{r,q}}{\sqrt{N}}.
\]
The result follows from Minkowski's inequality in
\(L^2(\nu;\mathbb{R}^{n})\).
\end{proof}
\begin{remark}
The Barron refinement provides the network-width rate
\(N^{-1/2}\). The exponent does not explicitly deteriorate with the
measurement dimension \(q\); however, the constant \(C_{r,q}\) may depend
on \(q\). Therefore, the result should not be described as completely
dimension independent without an additional uniform bound on the relevant
Barron norms.
\end{remark}
\section{Computational experiments}
\label{sec:experiments}
In this section, \autoref{sec:antiderivative-verification} presents the
convergence analysis of the topological DeepONet and
\autoref{sec:antiderivative-benchmark} benchmarks the proposed architectures
on the antiderivative operator considered in \cite{lu2021deeponet}. The
performance of the proposed methods on the Darcy flow problem is examined in
\autoref{sec:darcy-main}, \autoref{sec:fixed-time-navier-stokes} presents the
benchmark results for the fixed-time Navier--Stokes equations in the
vorticity formulation, and \autoref{sec:time-evolving-navier-stokes} extends
the comparison to the time-evolving trajectory-prediction operator.
\subsection{Convergence study for topological DeepONets}
\label{sec:antiderivative-verification}
\paragraph{Numerical verification of the discrete error decomposition.}
Figure~\ref{fig:complete-theorem-verification} evaluates the error terms in
Theorem~\ref{thm:discrete-topological-error} for the antiderivative operator
using
\[
q\in\{8,16,32,64,128\}.
\]
For each test sample, we compute the operator-level measurement error
\(E_{\mathrm{op,meas}}\), its Lipschitz upper bound \(E_{\mathrm{Lip}}\),
the output-basis truncation error \(E_{\mathrm{out}}\), the neural
approximation error \(E_{\mathrm{NN}}\), the total error \(E_{\mathrm{tot}}\),
and the theorem-consistent bound
\[
E_{\mathrm{thm}}
=
E_{\mathrm{Lip}}
+
E_{\mathrm{out}}
+
E_{\mathrm{NN}}.
\]
Figures~\ref{fig:complete-theorem-verification}(a) and
\ref{fig:complete-theorem-verification}(b) show that the
measurement-induced error decreases rapidly as the number of functionals
increases. For both the fixed and adaptive models,
\(E_{\mathrm{op,meas}}\) decreases by several orders of magnitude between
\(q=8\) and \(q=128\). The input reconstruction error in
Fig.~\ref{fig:complete-theorem-verification}(c) exhibits the same trend,
confirming that increasing \(q\) reduces the information lost by the
measurement map. At large \(q\), the total error is no longer dominated by
measurement compression; instead, the neural approximation error becomes
the principal contribution, while the output-basis truncation error remains
comparatively small.
The adaptive model achieves a lower total error at the largest measurement
dimensions, reaching approximately \(10^{-3}\) at \(q=128\), compared with
approximately \(4\times10^{-3}\) for the fixed model. The
theorem-consistent bound remains above the observed total error for every
test sample. In particular, at \(q=32\), the maximum samplewise ratio is
\[
\max_i
\frac{E_{\mathrm{tot}}^{(i)}}
     {E_{\mathrm{thm}}^{(i)}}
=
0.1249,
\]
with no bound violations, as shown in
Fig.~\ref{fig:complete-theorem-verification}(d). The gap between the total
error and the theorem bound reflects the conservativeness of replacing the
actual operator perturbation by the global Lipschitz estimate.
\subsection{Antiderivative operator benchmark}
\label{sec:antiderivative-benchmark}
We next compare the fixed- and learned-measurement Topological
DeepONets against Vanilla DeepONet and the enhanced Two-Step DeepONet on
the antiderivative operator for the aligned dataset of
\cite{lu2021deeponet} under matched parameter budgets.
Table~\ref{tab:deeponet_comparison} reports the error statistics over the
\(400\) test functions for the cosine-basis variants, and
Figure~\ref{fig:combined_error_comparison} compares the mean and global
relative \(L^2\) errors for the cosine- and Legendre-basis variants.
The learned-measurement (adaptive) model attains the lowest error in
nearly every statistic and outperforms the Two-Step baseline on \(72\%\)
of the test functions.
\begin{table*}[t]
    \centering
    \caption{
    Comparison of Vanilla DeepONet, Two-Step DeepONet, Fixed-Measurement
    DeepONet, and the proposed Learned-Measurement Topological DeepONet.
    Lower values are better for all error metrics.
    The win rates for the fixed- and learned-measurement models are computed
    relative to the Two-Step DeepONet on a per-function basis.
    }
    \label{tab:deeponet_comparison}
    \resizebox{\textwidth}{!}{
    \begin{tabular}{lcccc}
        \toprule
        \textbf{Metric}
        & \textbf{Vanilla}
        & \textbf{Two-Step}
        & \textbf{Fixed Measurement}
        & \textbf{Learned Measurement} \\
        \midrule
        Mean relative $L^2$
        & $3.6995\times10^{-2}$
        & $3.1643\times10^{-2}$
        & $3.7294\times10^{-2}$
        & $\mathbf{2.1166\times10^{-2}}$ \\
        Median relative $L^2$
        & $2.7768\times10^{-2}$
        & $2.3653\times10^{-2}$
        & $2.6736\times10^{-2}$
        & $\mathbf{1.4408\times10^{-2}}$ \\
        90th-percentile relative $L^2$
        & $7.1132\times10^{-2}$
        & $6.2938\times10^{-2}$
        & $7.0133\times10^{-2}$
        & $\mathbf{3.9863\times10^{-2}}$ \\
        95th-percentile relative $L^2$
        & $9.3179\times10^{-2}$
        & $8.0246\times10^{-2}$
        & $9.5440\times10^{-2}$
        & $\mathbf{5.6343\times10^{-2}}$ \\
        99th-percentile relative $L^2$
        & $1.5537\times10^{-1}$
        & $1.2647\times10^{-1}$
        & $2.0463\times10^{-1}$
        & $\mathbf{1.0659\times10^{-1}}$ \\
        Maximum relative $L^2$
        & $4.1882\times10^{-1}$
        & $\mathbf{3.7490\times10^{-1}}$
        & $6.3234\times10^{-1}$
        & $4.0287\times10^{-1}$ \\
        Global relative $L^2$
        & $2.8898\times10^{-2}$
        & $4.8093\times10^{-2}$
        & $3.7742\times10^{-2}$
        & $\mathbf{2.5516\times10^{-2}}$ \\
        RMSE
        & $1.5351\times10^{-2}$
        & $2.5548\times10^{-2}$
        & $2.0049\times10^{-2}$
        & $\mathbf{1.3555\times10^{-2}}$ \\
        Maximum absolute error
        & $2.8847\times10^{-1}$
        & $3.3463\times10^{-1}$
        & $3.3942\times10^{-1}$
        & $\mathbf{2.6523\times10^{-1}}$ \\
        \midrule
        Stage-II / training time (s)
        & $40.23$
        & $12.93$
        & $12.75$
        & $21.70$ \\
        Training parameters
        & $6{,}220$
        & $6{,}254$
        & $6{,}606$
        & $8{,}624$ \\
        Inference parameters
        & $6{,}220$
        & $6{,}254$
        & $6{,}606^{\dagger}$
        & $6{,}224$ \\
        Per-function win rate vs.\ Two-Step
        & --
        & N/A
        & $38.0\%$
        & $\mathbf{72.0\%}$ \\
        \bottomrule
    \end{tabular}
    }
    \vspace{1mm}
    \begin{minipage}{0.98\textwidth}
        \footnotesize
        \textbf{Notes:}
        Bold values indicate the best result in each error row.
        The proposed model uses additional trainable parameters during training
        because of its learned measurement layer and auxiliary reconstruction
        decoder, but the decoder is removed at inference.
        Consequently, its inference parameter count is essentially matched to
        that of the Vanilla and Two-Step baselines.
        $^{\dagger}$ The reported value includes the fixed
        (non-trainable) measurement matrix; excluding it, the inference
        network size matches the Two-Step baseline.
    \end{minipage}
\end{table*}
\begin{figure}[t]
\centering
\includegraphics[width=\textwidth]{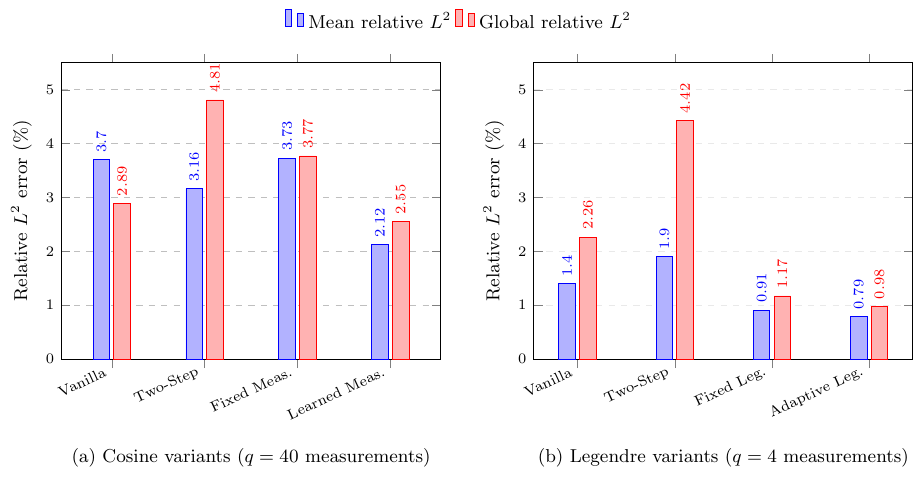}
\caption{Bar-chart comparison of DeepONet variants on the antiderivative
operator problem.
(a) Cosine-basis variants with $q=40$ functional measurements.
(b) Legendre-basis variants with $q=4$ functional measurements.
Lower values are better.}
\label{fig:combined_error_comparison}
\end{figure}
\subsection{Controlled operator with a known functional representation}
\label{sec:controlled-functional-operator}
To isolate the role of the input representation from the complexity of a
PDE solution operator, we consider a controlled operator whose dependence on
the input is known exactly. Let \(a\in L^2(0,1)\) and define
\begin{equation}
\label{eq:controlled-operator}
\mathcal{G}(a)(y)
=
\sum_{r=1}^{3}
\ell_r(a)\,\psi_r(y),
\qquad
\ell_r(a)
=
\int_0^1 a(x)\phi_r(x)\,dx,
\end{equation}
where
\begin{equation}
\phi_1(x)=1,
\qquad
\phi_2(x)=\sqrt{2}\cos(2\pi x),
\qquad
\phi_3(x)=\sqrt{2}\sin(4\pi x),
\end{equation}
and \(\{\psi_r\}_{r=1}^{3}\) is a fixed orthonormal output basis. The input
functions are sampled from a higher-dimensional random Fourier expansion,
\begin{equation}
a(x)
=
\sum_{j=1}^{J}\xi_j\varphi_j(x),
\qquad
\xi_j\sim\mathcal{N}(0,j^{-2}),
\end{equation}
so that only a small subset of the input directions influences the output.
This example is chosen because the exact task-relevant input subspace,
\[
\mathcal{S}_{\mathrm{true}}
=
\operatorname{span}\{\phi_1,\phi_2,\phi_3\},
\]
is known. It therefore permits a direct comparison between
operator-relevant functional learning and generic dimensionality-reduction
strategies. In particular, PCA preserves high-variance input directions,
random projection preserves no task-specific structure, and point sensors
provide only local information. By contrast, the adaptive topological models
learn three global continuous functionals from a prescribed dictionary. The
controlled setting also allows the learned measurement subspace,
conditioning, and drift from the structured initialization to be evaluated
directly.
We compare the fixed functional representation with adaptive variants that
introduce measurement learning, input reconstruction, task-relevant
decoding, orthogonality regularization, drift regularization, and either
random or warm initialization. Additional baselines include random and PCA
projections, fixed and optimized sensors, an unconstrained learned
bottleneck, and a full-field MLP. All compressed models use the same
three-dimensional latent representation, output basis, branch architecture,
training data, validation criterion, and test set. Results are averaged over
multiple random seeds.
\begin{table*}[t]
\centering
\caption{
Controlled operator benchmark averaged over multiple random seeds.
The first column reports the mean samplewise relative \(L^2\) error and its
standard deviation across seeds. The global relative \(L^2\) error and the
\(95\)th percentile samplewise error are also reported. The best value in
each column is shown in bold.
}
\label{tab:controlled-operator}
\begin{tabular}{lccc}
\toprule
Model
& Mean rel. \(L^2\)
& Global rel. \(L^2\)
& \(P_{95}\) \\
\midrule
Adaptive Task Decoder
& \(\mathbf{8.8264\times10^{-3}\pm6.28\times10^{-4}}\)
& \(8.5210\times10^{-3}\)
& \(\mathbf{1.6269\times10^{-2}}\) \\
Adaptive \(+\) Orthogonality
& \(9.0549\times10^{-3}\pm6.53\times10^{-4}\)
& \(8.2720\times10^{-3}\)
& \(1.8581\times10^{-2}\) \\
Adaptive Learned Only
& \(9.1422\times10^{-3}\pm8.11\times10^{-4}\)
& \(\mathbf{7.6412\times10^{-3}}\)
& \(2.0230\times10^{-2}\) \\
Adaptive Full Scratch
& \(9.6359\times10^{-3}\pm3.22\times10^{-4}\)
& \(1.0145\times10^{-2}\)
& \(1.9448\times10^{-2}\) \\
Adaptive Full Warm
& \(1.0449\times10^{-2}\pm5.09\times10^{-4}\)
& \(9.9917\times10^{-3}\)
& \(1.9340\times10^{-2}\) \\
Adaptive \(+\) Input Decoder
& \(1.1006\times10^{-2}\pm1.96\times10^{-4}\)
& \(1.0896\times10^{-2}\)
& \(2.1711\times10^{-2}\) \\
Adaptive \(+\) Drift
& \(1.1380\times10^{-2}\pm2.65\times10^{-3}\)
& \(1.0924\times10^{-2}\)
& \(2.0177\times10^{-2}\) \\
Learned Bottleneck
& \(1.1557\times10^{-2}\pm1.01\times10^{-3}\)
& \(1.0581\times10^{-2}\)
& \(2.3018\times10^{-2}\) \\
Full-field MLP
& \(1.5294\times10^{-2}\pm1.35\times10^{-3}\)
& \(1.3940\times10^{-2}\)
& \(3.2170\times10^{-2}\) \\
PCA Projection
& \(2.1695\times10^{-1}\pm1.01\times10^{-2}\)
& \(1.7974\times10^{-1}\)
& \(6.3162\times10^{-1}\) \\
Optimized Sensors
& \(4.9937\times10^{-1}\pm3.43\times10^{-2}\)
& \(3.5456\times10^{-1}\)
& \(1.3392\) \\
Fixed Topological
& \(5.6281\times10^{-1}\pm2.72\times10^{-2}\)
& \(5.0820\times10^{-1}\)
& \(1.0309\) \\
Fixed Sensors
& \(5.7582\times10^{-1}\pm2.95\times10^{-2}\)
& \(4.3769\times10^{-1}\)
& \(1.3365\) \\
Random Projection
& \(5.7822\times10^{-1}\pm1.46\times10^{-1}\)
& \(4.4839\times10^{-1}\)
& \(1.3155\) \\
\bottomrule
\end{tabular}
\end{table*}
As shown in \autoref{tab:controlled-operator}, all adaptive functional
models substantially outperform the fixed functional representation,
sensor-based representations, PCA, and random projection. The Adaptive Task
Decoder attains the lowest mean relative error,
\(8.83\times10^{-3}\), and the lowest \(P_{95}\) error,
\(1.63\times10^{-2}\). The unregularized Adaptive Learned Only model gives
the lowest global relative error, \(7.64\times10^{-3}\), while the
orthogonally regularized model provides a comparably low mean error of
\(9.05\times10^{-3}\). These small differences indicate that all three
variants reliably identify the task-relevant functional subspace, although
the auxiliary task decoder improves the distribution of samplewise errors.
The task-relevant decoder is more effective than reconstruction of the
complete input. The latter attains a mean error of
\(1.10\times10^{-2}\), because full-input reconstruction encourages the
three-dimensional latent representation to preserve variability that does
not influence \(\mathcal{G}\). The learned bottleneck is also competitive,
with a mean error of \(1.16\times10^{-2}\), confirming that unconstrained
task-dependent compression can discover the relevant low-dimensional
structure. Nevertheless, the Adaptive Task Decoder reduces the mean error by
approximately \(23.6\%\) relative to the learned bottleneck while retaining
an explicit representation in terms of continuous input functionals.
The poor performance of PCA, random projections, and point sensors further
demonstrates that low dimensionality alone is insufficient. PCA preserves
input variance rather than operator relevance, while three point values or
three random projections generally cannot recover the global integral
quantities defining the operator. In contrast, adaptive functional learning
reduces the mean error from \(56.3\%\) for the Fixed Topological model to
below \(1\%\), showing that the principal gain arises from identifying the
operator-relevant measurement subspace rather than merely increasing network
capacity.
\autoref{fig:controlled-functional-diagnostics} provides a complementary
mechanistic interpretation. Because an individual functional basis is
identifiable only up to an invertible rotation within its span, the exact
and learned bases are aligned before visualization. The associated
projector error and Gram-matrix condition number quantify recovery of the
true functional subspace and the linear independence of the learned
measurements.
\begin{figure}[t]
\centering
\includegraphics[width=\textwidth]
{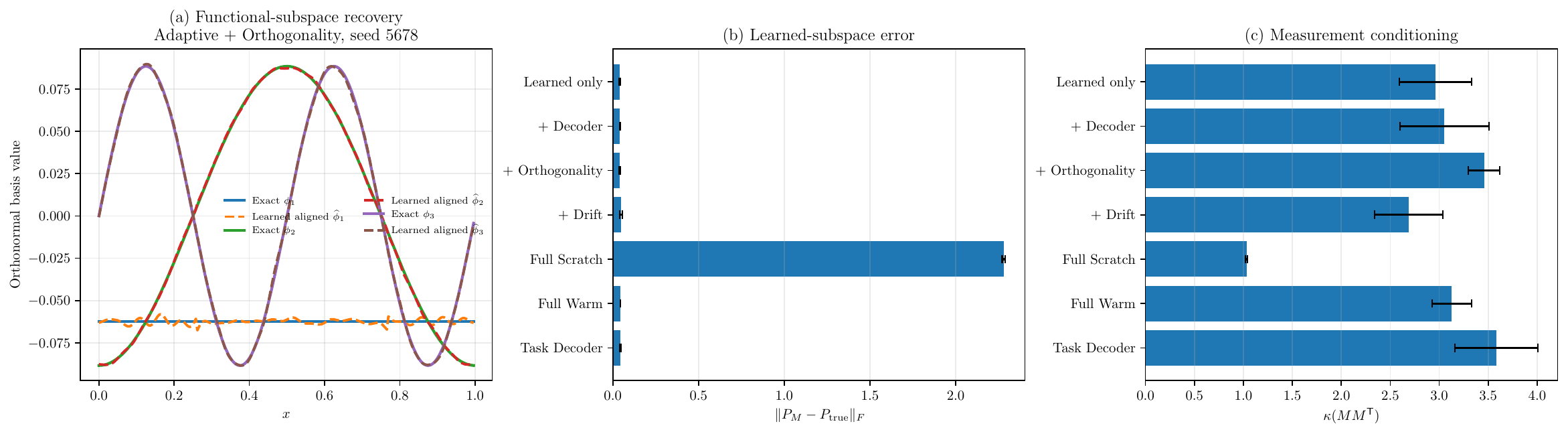}
\caption{
Recovery of the task-relevant functional subspace for the controlled
operator. (a) Exact functional basis and the aligned basis learned by the
adaptive functional model for a representative seed. Since the individual
measurement functions are identifiable only up to a rotation within their
span, the bases are aligned before visualization. (b) Frobenius error
between the learned and exact subspace projectors. (c) Condition number of
the learned measurement Gram matrix. Error bars represent variability across
random seeds.
}
\label{fig:controlled-functional-diagnostics}
\end{figure}
\subsection{Heterogeneous Darcy-flow operator}
\label{sec:darcy-main}
We consider the Darcy-flow benchmark commonly used in the neural-operator
literature and in the Fourier Neural Operator study
\cite{li2021fourier}. The steady pressure field satisfies
\begin{equation}
    -\nabla\cdot
    \left(
        a(\bm{x})\nabla u(\bm{x})
    \right)
    =
    f(\bm{x}),
    \qquad
    \bm{x}\in\Omega,
    \label{eq:darcy-main}
\end{equation}
where \(a(\bm{x})\) is the heterogeneous permeability field and
\(u(\bm{x})\) is the corresponding pressure solution. The learning objective
is to approximate the nonlinear parameter-to-solution operator
\begin{equation}
    \mathcal{G}:
    \mathcal{A}\subset \mathcal{V}=L^\infty(\Omega)
    \longrightarrow
    \mathcal{U}=H_0^1(\Omega),
    \qquad
    a\longmapsto u,
    \label{eq:darcy-operator-map}
\end{equation}
where
\begin{equation}
    \mathcal{A}
    =
    \left\{
    a\in L^\infty(\Omega):
    0<a_{\min}\le a(\bm{x})\le a_{\max}<\infty
    \text{ a.e. in }\Omega
    \right\}.
    \label{eq:darcy-spaces}
\end{equation}
The admissible coefficient set is equipped with the topology inherited from
$L^\infty(\Omega)$. Uniform ellipticity gives well-posedness and the natural
stability of the coefficient-to-solution map in this topology. Although the
numerical coefficient fields also belong to $L^2(\Omega)$ on the bounded
domain, we do not identify $L^2(\Omega)$ as the ambient operator topology and
do not claim continuity of the unrestricted coefficient-to-solution map from
plain $L^2(\Omega)$ into $H_0^1(\Omega)$.
The original fields are generated at a resolution of \(421\times421\).
For the present study, the data are represented on an \(85\times85\) grid,
and we use \(800\), \(100\), and \(100\) realizations for training,
validation, and testing, respectively.
For the fixed Topological DeepONet, the permeability field is represented
through \(q\) continuous linear functionals
\begin{equation}
    \ell_j(a)
    =
    \int_\Omega
        a(\bm{x})\phi_j(\bm{x})
    \,\mathrm{d}\bm{x},
    \qquad
    j=1,\ldots,q,
    \label{eq:darcy-functional-measurement}
\end{equation}
where \(\phi_j\in L^1(\Omega)\). For the polynomial and trigonometric
dictionaries used here, the atoms are bounded on the bounded domain and
therefore belong to $L^1(\Omega)$. Since the Darcy input space is
$\mathcal V=L^\infty(\Omega)$,
\begin{equation}
    |\ell_j(a)|
    \leq
    \|a\|_{L^\infty(\Omega)}\,\|\phi_j\|_{L^1(\Omega)},
    \label{eq:darcy-functional-continuity}
\end{equation}
so each \(\ell_j\in\mathcal V'=(L^\infty(\Omega))'\). The corresponding
functional-coordinate map is
\begin{equation}
    \mathcal{M}_q(a)
    =
    \left[
        \ell_1(a),\ldots,\ell_q(a)
    \right]^{\mathsf T}
    \in\mathbb{R}^q.
    \label{eq:darcy-functional-map}
\end{equation}
The implementation supports Legendre, cosine, Chebyshev, and Lagrange
functional dictionaries. In the default setting, the measurement functions
are constructed from a total-degree tensor-product Legendre dictionary,
\begin{equation}
    \phi_{ij}(x,y)
    =
    \widehat P_i(x)\widehat P_j(y),
    \qquad
    i+j\leq p,
    \label{eq:darcy-legendre-dictionary}
\end{equation}
where \(\widehat P_i\) denotes a normalized Legendre polynomial. The
measurements are evaluated numerically using quadrature on the discrete grid.
For the adaptive Topological DeepONet, the learned coordinates are linear
combinations of an active set of dictionary measurements:
\begin{equation}
    \bm{z}_{\mathrm{ad}}(a)
    =
    A_\theta\bm{\alpha}(a),
    \label{eq:darcy-adaptive-coordinates}
\end{equation}
where \(\bm{\alpha}(a)\) contains the active functional coordinates and
\(A_\theta\) is a trainable linear map. Consequently, each learned coordinate
remains a continuous linear functional belonging to the span of the
prescribed dictionary.
For all Two-Step models, the pressure snapshots are projected onto a common
rank-\(r\) basis obtained from the singular value decomposition of the
training outputs. The branch network predicts the corresponding reduced
coefficients, which are mapped back to the full pressure field using the
fixed output basis. The fixed model can therefore be summarized as
\begin{equation}
    a
    \xrightarrow{\;\mathcal{M}_q\;}
    \mathbb{R}^q
    \xrightarrow{\;\mathcal{B}_\theta\;}
    \mathbb{R}^r
    \xrightarrow{\;\text{fixed output basis}\;}
    \widehat u.
    \label{eq:darcy-fixed-pipeline}
\end{equation}
Further details on the data preprocessing, dictionary construction,
weighted orthonormalization, measurement selection, adaptive-measurement
regularization, reduced output representation, and
evaluation metrics are provided in
Appendix~\ref{app:darcy-sensor-interpretation}.
The compared architectures and their input
dimensions are illustrated in
Fig.~\ref{fig:dimension_comparison_twostep_sensor_topo}.
We compare the following models:
\begin{enumerate}
    \item full-field Two-Step, using all \(7225\) permeability values;
    \item Sensor Two-Step, using \(q=32\) point evaluations;
    \item Fixed Topological DeepONet, using \(q=32\) prescribed
    functional coordinates;
    \item Adaptive Topological DeepONet, using \(q=32\) learned
    functional coordinates;
    \item Vanilla DeepONet, with jointly trained branch and trunk
    networks.
\end{enumerate}
The most controlled comparison is between Sensor Two-Step and Fixed
Topological DeepONet. Both use the same input dimension \(q=32\), branch
architecture, reduced output basis, and \(144{,}915\) inference
parameters. They differ only in their input representation:
\begin{equation}
\underbrace{
\left\{
a(\bm{x}_j)
\right\}_{j=1}^{32}
}_{\text{local point evaluations}}
\qquad\text{versus}\qquad
\underbrace{
\left\{
\ell_j(a)
\right\}_{j=1}^{32}
}_{\text{global functional coordinates}},
\label{eq:darcy-input-comparison}
\end{equation}
with the functional coordinates \(\ell_j(a)\) defined
in~\eqref{eq:darcy-functional-measurement}.
The complete quantitative results are reported in
Table~\ref{tab:darcy-results}. The corresponding comparisons of global
relative \(L^2\) error, training time, and inference parameter count are
shown in Fig.~\ref{fig:darcy-metrics}.
\begin{figure}[t]
\centering
\includegraphics[width=\textwidth]{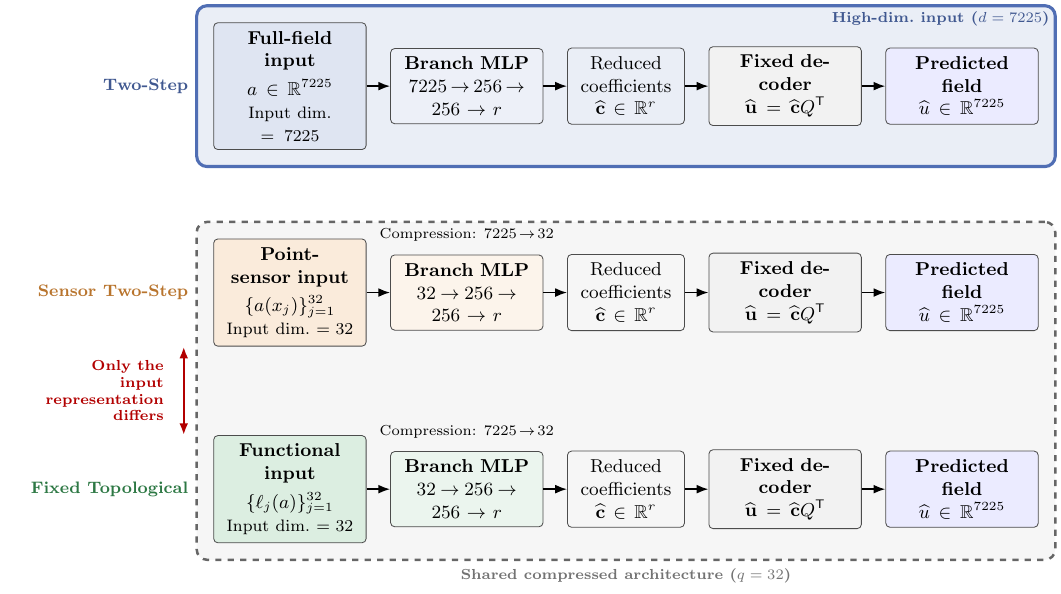}
\caption{Three operator-learning architectures compared. Two-Step takes the
full input field ($7225$ values) as branch input (solid box). Sensor
Two-Step and Fixed Topological DeepONet both compress to $32$ dimensions
and share an identical downstream architecture (dashed box) --- differing
only in how the $32$ inputs are formed: point values vs.\ global functional
measurements.}
\label{fig:dimension_comparison_twostep_sensor_topo}
\end{figure}
Fixed Topological DeepONet achieves a global relative \(L^2\) error of
\(5.88\%\), compared with \(10.39\%\) for Sensor Two-Step. This
corresponds to a relative error reduction of
\begin{equation}
\frac{0.1038841-0.0587551}{0.1038841}\times100
\approx 43.4\%.
\end{equation}
The full-field Two-Step, Adaptive Topological, and Vanilla DeepONet
models achieve global relative errors of \(8.31\%\), \(7.27\%\), and
\(9.21\%\), respectively.
The reduced models compress the branch input from \(7225\) values to
\(32\) features, corresponding to
$\approx225.8$
times input compression. The results therefore indicate that, for this
Darcy dataset, global functional coordinates provide a substantially
more informative reduced representation than sparse point evaluations
at the same input dimension and parameter count. The interpretation and
limitations of this comparison are discussed further in
Appendix~\ref{app:darcy-sensor-interpretation}.
\begin{table*}[t]
\centering
\caption{
Darcy operator-learning test performance and computational cost.
The highlighted columns form the controlled comparison:
Sensor Two-Step and Fixed Topological DeepONet use the same
input dimension, branch architecture, output decoder, and number
of inference parameters.
}
\label{tab:darcy-results}
\setlength{\tabcolsep}{4.2pt}
\renewcommand{\arraystretch}{1.16}
\resizebox{\textwidth}{!}{%
\begin{tabular}{
    l
    r
    |>{\columncolor{blue!6}}r
    >{\columncolor{blue!6}}r|
    r
    r
}
\toprule
\textbf{Metric}
&
\textbf{Two-Step}
&
\cellcolor{blue!15}\textbf{Sensor Two-Step}
&
\cellcolor{blue!15}\textbf{Fixed Topological}
&
\textbf{Adaptive Topological}
&
\textbf{Vanilla DeepONet}
\\
\midrule
Mean relative $L^2$ error
&
$8.197754\times10^{-2}$
&
$1.017463\times10^{-1}$
&
$\mathbf{5.739814\times10^{-2}}$
&
$7.064593\times10^{-2}$
&
$9.002882\times10^{-2}$
\\
Global relative $L^2$ error
&
$8.312111\times10^{-2}$
&
$1.038841\times10^{-1}$
&
$\mathbf{5.875507\times10^{-2}}$
&
$7.266852\times10^{-2}$
&
$9.208734\times10^{-2}$
\\
RMSE
&
$5.670150\times10^{-4}$
&
$7.086511\times10^{-4}$
&
$\mathbf{4.008007\times10^{-4}}$
&
$4.957121\times10^{-4}$
&
$6.281782\times10^{-4}$
\\
MAE
&
$4.094164\times10^{-4}$
&
$4.926076\times10^{-4}$
&
$\mathbf{2.849369\times10^{-4}}$
&
$3.518841\times10^{-4}$
&
$4.798732\times10^{-4}$
\\
Best validation global relative $L^2$
&
$8.872261\times10^{-2}$
&
$1.106748\times10^{-1}$
&
$\mathbf{6.104581\times10^{-2}}$
&
$7.354654\times10^{-2}$
&
$9.573192\times10^{-2}$
\\
Best epoch
&
$350$
&
$50$
&
$150$
&
$100$
&
$1250$
\\
Training time (s)
&
$41.348$
&
$\mathbf{23.777}$
&
$28.095$
&
$41.206$
&
$92.113$
\\
Peak GPU memory (MB)
&
$75.248$
&
$41.759$
&
$\mathbf{35.288}$
&
$36.658$
&
$90.384$
\\
Training parameters
&
$1{,}986{,}323$
&
$144{,}915$
&
$144{,}915$
&
$202{,}842$
&
$1{,}956{,}391$
\\
Inference parameters
&
$1{,}986{,}323$
&
$144{,}915$
&
$144{,}915$
&
$152{,}307$
&
$1{,}956{,}391$
\\
Measurement parameters
&
$0$
&
$0$
&
$0$
&
$7{,}392$
&
$0$
\\
Decoder parameters
&
$0$
&
$0$
&
$0$
&
$50{,}535$
&
$0$
\\
Input dimension
&
$7{,}225$
&
$32$
&
$32$
&
$32$
&
$7{,}225$
\\
Compression ratio
&
$1.000$
&
$225.781$
&
$225.781$
&
$225.781$
&
$1.000$
\\
\bottomrule
\end{tabular}%
}
\end{table*}
\begin{figure*}[t]
\centering
\includegraphics[width=\textwidth]{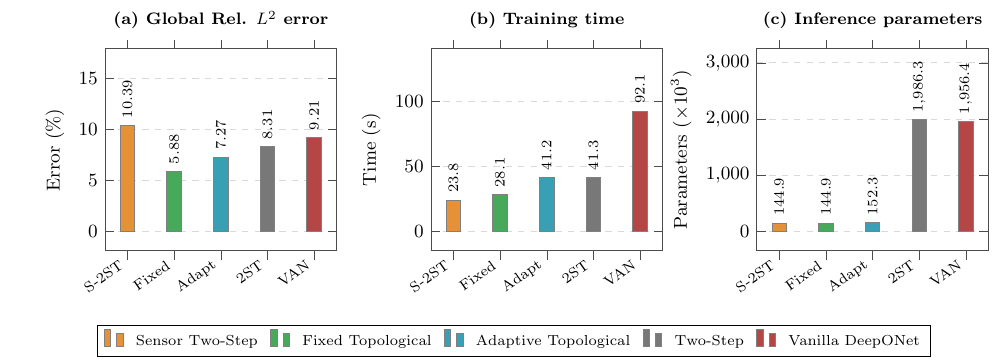}
\caption{
Comparison of DeepONet variants for the Darcy problem.
Abbreviations: S-2ST = Sensor Two-Step, Fixed = Fixed Topological DeepONet,
Adapt = Adaptive Topological DeepONet, 2ST = full Two-Step DeepONet,
and VAN = Vanilla DeepONet.
Panel (a) reports the global relative $L^2$ error,
panel (b) the measured training time, and panel (c) the inference parameter count.
}
\label{fig:darcy-metrics}
\end{figure*}
\subsection{Heterogeneous-discretization Darcy benchmark}
\label{sec:darcy-heterogeneous}
We next evaluate whether the learned input representation transfers across
spatial discretizations. The Darcy coefficient fields are observed on a
mixture of \(33\times33\), \(49\times49\), \(65\times65\), and
\(85\times85\) grids during training. At test time, the models are evaluated
on the reference \(85\times85\) grid and on previously unseen
\(57\times57\), \(73\times73\), and \(97\times97\) grids. Because the
available dataset is stored on a common fine grid, the heterogeneous
observations are generated by sampling each realization at the prescribed
native resolution.
The Fixed and Adaptive Topological DeepONets evaluate the same continuous
linear functionals on every native grid using grid-dependent quadrature.
Hence, the branch input dimension remains unchanged across discretizations.
We compare these models with an Interpolated Two-Step baseline, which first
maps every native observation to the common \(85\times85\) grid, and a
Sensor Two-Step baseline using a fixed set of physical sensor locations.
\begin{table*}[t]
\centering
\caption{
Mean relative \(L^2\) errors for the heterogeneous-discretization Darcy
benchmark. The models are tested on the reference grid, unseen
discretizations, noisy observations, and randomly missing observations.
The best result in each row is shown in bold.
}
\label{tab:darcy-heterogeneous-results}
\begin{tabular}{lcccc}
\toprule
Test condition
& Interpolated Two-Step
& Sensor Two-Step
& Fixed
& Adaptive \\
\midrule
\(85\times85\)
& \(6.630\times10^{-2}\)
& \(1.123\times10^{-1}\)
& \(\mathbf{5.565\times10^{-2}}\)
& \(5.629\times10^{-2}\) \\
Unseen \(57\times57\)
& \(6.643\times10^{-2}\)
& \(1.119\times10^{-1}\)
& \(\mathbf{5.536\times10^{-2}}\)
& \(5.546\times10^{-2}\) \\
Unseen \(73\times73\)
& \(6.631\times10^{-2}\)
& \(1.123\times10^{-1}\)
& \(\mathbf{5.537\times10^{-2}}\)
& \(5.585\times10^{-2}\) \\
Unseen \(97\times97\)
& \(6.631\times10^{-2}\)
& \(1.123\times10^{-1}\)
& \(\mathbf{5.583\times10^{-2}}\)
& \(5.644\times10^{-2}\) \\
\(57\times57\), \(1\%\) noise
& \(6.642\times10^{-2}\)
& \(1.119\times10^{-1}\)
& \(\mathbf{5.536\times10^{-2}}\)
& \(5.548\times10^{-2}\) \\
\(57\times57\), \(5\%\) noise
& \(6.646\times10^{-2}\)
& \(1.118\times10^{-1}\)
& \(\mathbf{5.539\times10^{-2}}\)
& \(5.553\times10^{-2}\) \\
\(57\times57\), \(10\%\) missing
& \(6.664\times10^{-2}\)
& \(1.136\times10^{-1}\)
& \(\mathbf{5.753\times10^{-2}}\)
& \(5.866\times10^{-2}\) \\
\(57\times57\), \(30\%\) missing
& \(1.163\times10^{-1}\)
& \(1.727\times10^{-1}\)
& \(\mathbf{6.101\times10^{-2}}\)
& \(6.498\times10^{-2}\) \\
\bottomrule
\end{tabular}
\end{table*}
As shown in Table~\ref{tab:darcy-heterogeneous-results}, both functional
models exhibit nearly discretization-independent accuracy. The Fixed model
achieves mean relative errors of \(5.56\%\), \(5.54\%\), \(5.54\%\), and
\(5.58\%\) on the \(85\times85\), unseen \(57\times57\), unseen
\(73\times73\), and unseen \(97\times97\) grids, respectively. The
Adaptive model shows a similarly small variation, with errors between
\(5.55\%\) and \(5.64\%\). In contrast, the Interpolated Two-Step and Sensor
Two-Step baselines remain less accurate across all test resolutions.
The functional representations are also insensitive to moderate
observation noise. On the unseen \(57\times57\) grid, increasing the noise
level from \(0\%\) to \(5\%\) changes the Fixed-model error only from
\(5.536\%\) to \(5.539\%\). The distinction becomes more pronounced under
missing observations. With \(30\%\) of the native-grid values removed, the
Interpolated Two-Step error increases from \(6.64\%\) to \(11.63\%\),
whereas the Fixed-model error increases only from \(5.54\%\) to \(6.10\%\).
This corresponds to approximately \(75\%\) degradation for the interpolated
baseline but only \(10\%\) for the Fixed Topological DeepONet.
\begin{figure}
    \centering
    \includegraphics[width=\textwidth]
    {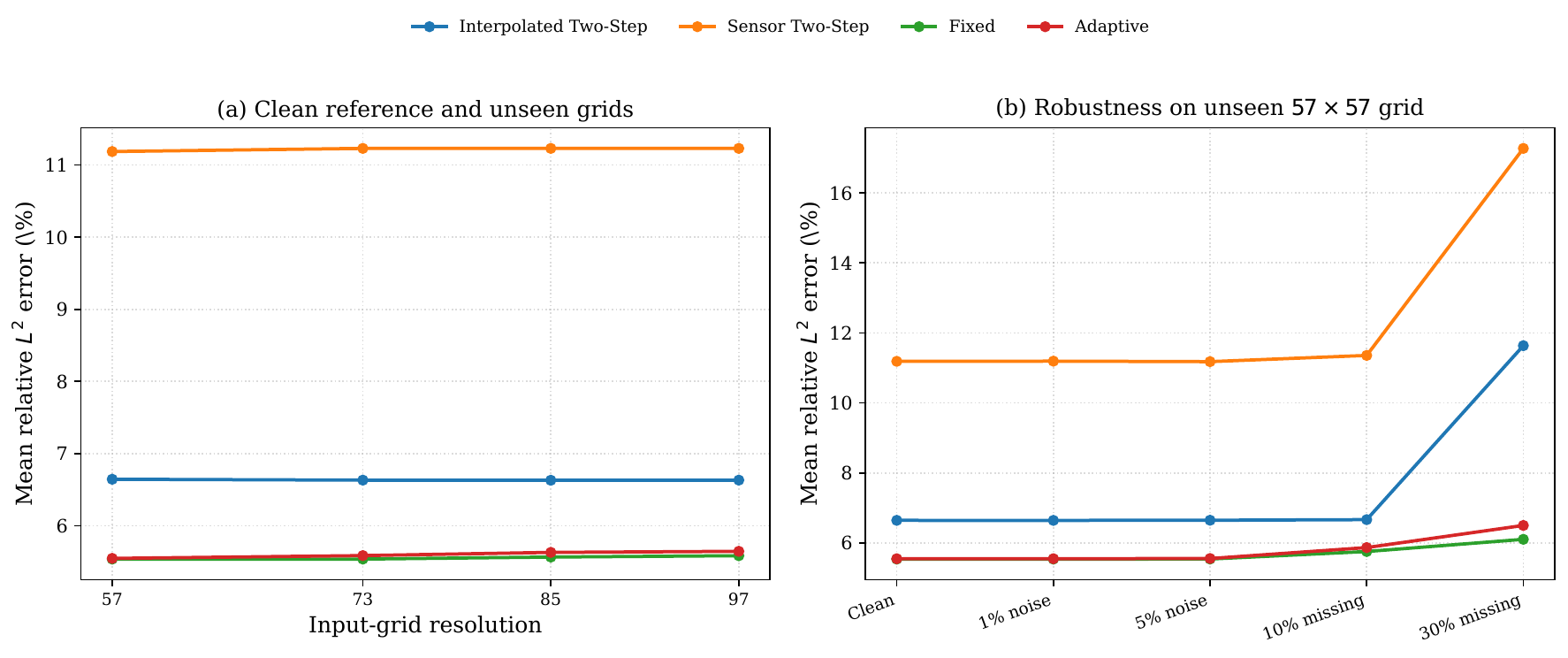}
    \caption{
    Heterogeneous-discretization Darcy results. Panel (a) compares the mean
    relative \(L^2\) error on the reference \(85\times85\) grid and on
    unseen \(57\times57\), \(73\times73\), and \(97\times97\) grids.
    Panel (b) reports robustness on the unseen \(57\times57\) grid under
    additive noise and randomly missing observations. The functional models
    retain nearly constant accuracy across discretizations and degrade much
    less than the interpolation and sensor baselines under incomplete
    observations.
    }
    \label{fig:darcy-heterogeneous}
\end{figure}
\autoref{fig:darcy-heterogeneous} confirms that the functional
coordinates provide a portable representation of the input field. Since the
same continuous functionals are evaluated directly on each native grid, the
learned branch network does not depend on a particular discretization.
These results therefore demonstrate both discretization transfer and
robustness to imperfect observations, while avoiding the need to reconstruct
every input on a common reference mesh.
\subsection{Fixed-time Navier--Stokes benchmark}
\label{sec:fixed-time-navier-stokes}
We consider the two-dimensional incompressible Navier-Stokes equations
in vorticity form on the periodic unit square \(D=(0,1)^2\):
\begin{equation}
    \frac{\partial \omega}{\partial t}
    +
    \bm{u}\cdot\nabla\omega
    =
    \nu\Delta\omega + f,
    \qquad
    \nabla\cdot\bm{u}=0,
    \label{eq:ns-vorticity-main}
\end{equation}
where \(\omega\) is the scalar vorticity,
\(\bm{u}=\nabla^\perp(-\Delta)^{-1}\omega\),
\(\nu=10^{-3}\), \(f\) is a prescribed forcing term, and periodic
boundary conditions are imposed in both
spatial directions. Following the standard neural-operator benchmark
of~\cite{li2021fourier}, the data consist of 5000 realizations sampled on a
\(64\times64\) grid at 50 stored time levels; a representative vorticity
evolution is shown in Fig.~\ref{fig:ns-vorticity}.
For the fixed-time Navier--Stokes benchmark, we learn
\begin{equation}
\mathcal{G}_{0\rightarrow10}:
L_{\mathrm{per}}^2(D)
\longrightarrow
L_{\mathrm{per}}^2(D),
\qquad
\omega(\cdot,t_0)
\mapsto
\omega(\cdot,t_{10}),
\end{equation}
where \(D=(0,1)^2\). After discretization on a \(64\times64\) grid,
\[
\mathcal{G}_{0\rightarrow10}^{h}:
\mathbb{R}^{4096}
\longrightarrow
\mathbb{R}^{4096}.
\]
\begin{figure}[t]
    \includegraphics[width=\linewidth]{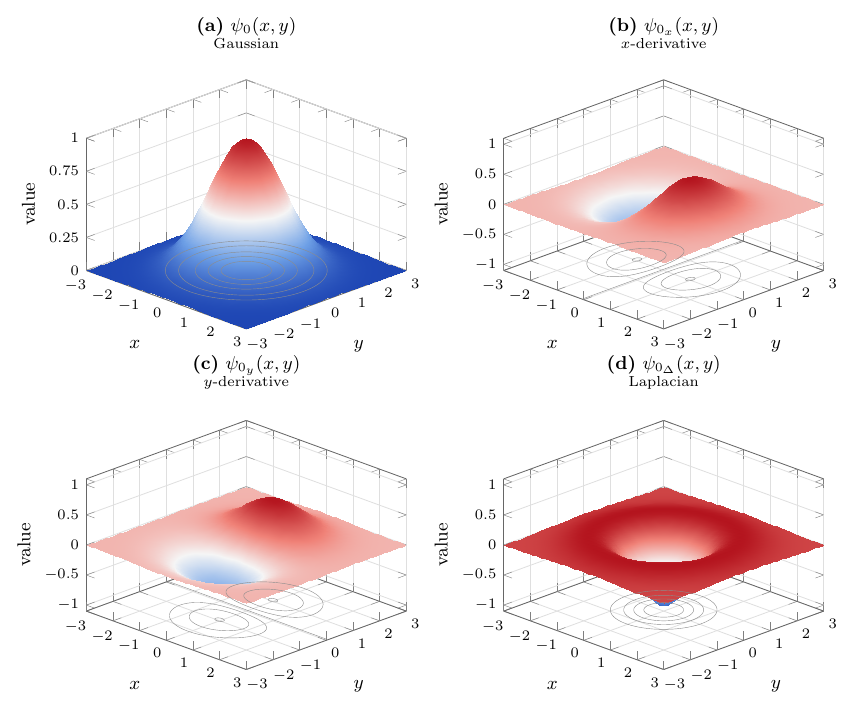}
    \caption{ Representative multiscale measurement atoms used in the fixed
    functional dictionary. Gaussian measurements capture localized averages,
    directionally weighted Gaussian atoms emphasize directional spatial variation, and the
    Laplacian-of-Gaussian atom emphasizes localized curvature and
    vortical structure.}
    \label{fig:ms_wavelet}
\end{figure}
The realizations are divided into 4000 training,
500 validation, and 500 test samples. After vectorization, both the
input and output fields belong to \(\mathbb{R}^{4096}\).
We compare five operator-learning formulations:
Full-field Two-Step, Sensor Two-Step, Fixed Topological Two-Step,
Adaptive Topological Two-Step, and Vanilla DeepONet. The full-field
models receive the complete \(64\times64\) input. The compressed models
use \(q=128\) measurements, corresponding to a \(32\times\) reduction
in input dimension. The Sensor model uses pointwise evaluations
\[
    z_j^{\mathrm{sens}}
    =
    \omega(\bm{x}_{s_j},t_0),
    \qquad
    j=1,\ldots,q,
\]
at prescribed sensor locations \(\bm{x}_{s_1},\ldots,\bm{x}_{s_q}\),
whereas the Fixed Topological model uses distributed linear functionals
\begin{equation}
    z_j^{\mathrm{fix}}
    =
    \left\langle
        \omega(\cdot,t_0),\psi_{i_j}
    \right\rangle,
    \qquad
    j=1,\ldots,q,
    \label{eq:fixed-topological-main}
\end{equation}
selected from a localized multiscale dictionary
\(\{\psi_k\}_{k=1}^{K}\). Representative multiscale atoms are shown
in \autoref{fig:ms_wavelet}.
The Adaptive Topological model replaces the
fixed selection by
\begin{equation}
    \mathbf{A}_{\theta}
    =
    \mathbf{A}_0
    +
    \Delta\mathbf{A}_{\theta},
    \qquad
    \mathbf{z}^{\mathrm{ad}}
    =
    \mathbf{h}\,
    \mathbf{A}_{\theta},
    \label{eq:adaptive-topological-main}
\end{equation}
where
\(\mathbf{h}
=
[\langle\omega,\psi_1\rangle,\ldots,
\langle\omega,\psi_K\rangle]\)
collects all \(K\) dictionary measurements,
\(\mathbf{A}_0\in\mathbb{R}^{K\times q}\) is the fixed selection matrix,
and \(\Delta\mathbf{A}_{\theta}\) is a trainable correction.
The initialization
\(\Delta\mathbf{A}_{\theta}=0\)
ensures that the adaptive and fixed models coincide at epoch zero.
For the Two-Step methods, the target fields are represented in a
truncated SVD basis. If
\[
    \mathbf{Y}_c
    =
    \mathbf{U}\mathbf{\Sigma}\mathbf{V}^{\top}
\]
is the SVD of the centered training-output matrix
\(\mathbf{Y}_c\) (Appendix~\ref{app:fixed-time-details}),
then the first \(r\) right singular vectors form
\(\mathbf{Q}\in\mathbb{R}^{4096\times r}\), and
\begin{equation}
    \mathbf{y}
    \approx
    \overline{\mathbf{y}}
    +
    \mathbf{Q}\mathbf{c},
    \label{eq:svd-reconstruction-main}
\end{equation}
where \(\overline{\mathbf{y}}\) is the mean training target field
and \(\mathbf{c}\) the reduced coefficient vector.
The branch network predicts the coefficient vector
\(\widehat{\mathbf{c}}\), after which the target field is reconstructed
using~\eqref{eq:svd-reconstruction-main}. For the common-backbone
comparison, all five models use an 2D Fourier spectral-layer branch
\cite{li2021fourier}; Vanilla DeepONet additionally uses a coordinate trunk
network as in~\cite{lu2021deeponet}.
The Two-Step models are trained using a coefficient-space mean-squared
error augmented by a relative coefficient loss,
\begin{equation}
    \mathcal{L}_{\mathrm{TwoStep}}
    =
    \frac{1}{Br}
    \sum_{i=1}^{B}
    \left\|
        \widehat{\mathbf{c}}^{(i)}
        -
        \mathbf{c}^{(i)}
    \right\|_2^2
    +
    \lambda_{\mathrm{rel}}
    \frac{1}{B}
    \sum_{i=1}^{B}
    \frac{
        \left\|
            \widehat{\mathbf{c}}^{(i)}
            -
            \mathbf{c}^{(i)}
        \right\|_2
    }{
        \left\|
            \mathbf{c}^{(i)}
        \right\|_2+\varepsilon
    },
    \label{eq:twostep-loss-main}
\end{equation}
where \(B\) is the mini-batch size, \(\lambda_{\mathrm{rel}}\ge0\)
a fixed weight, and \(\varepsilon>0\) a small constant.
The adaptive model includes the additional regularization
\[
    \lambda_{\mathrm{drift}}
    \left\|
        \Delta\mathbf{A}_{\theta}
    \right\|_F^2.
\]
Performance is reported using the global relative \(L^2\) error
\begin{equation}
    \varepsilon_{\mathrm{global}}
    =
    \frac{
        \left\|
            \widehat{\mathbf{Y}}-\mathbf{Y}
        \right\|_F
    }{
        \left\|
            \mathbf{Y}
        \right\|_F
    },
    \label{eq:global-error-main}
\end{equation}
where \(\mathbf{Y}\) and \(\widehat{\mathbf{Y}}\) stack the true
and predicted test targets,
together with the mean sample-wise relative \(L^2\) error. Further
details on the multiscale dictionary, functional selection, adaptive
parameterization, SVD construction, network architectures, and
optimization settings are provided in Appendix~\ref{app:fixed-time-details}.
\begin{figure}
    \centering
    \includegraphics[width=\linewidth]{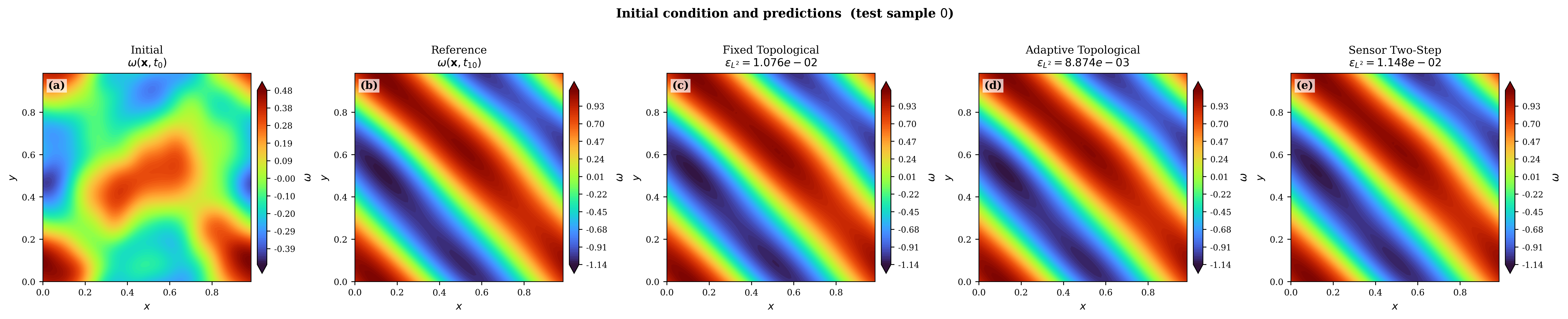}
    \caption{Vorticity fields for a representative realization of
the two-dimensional incompressible Navier--Stokes benchmark. The panels show
the initial condition, the reference target field, and the predictions of the
Fixed Topological, Adaptive Topological, and Sensor Two-Step models,
respectively.}
    \label{fig:ns-vorticity}
\end{figure}
The empirical cumulative distribution of the sample-wise relative
\(L^2\) error (Fig.~\ref{fig:ecdf_comparison}) reveals a clear advantage
of the topological models in the low-error regime, which is not fully
captured by aggregate mean-error metrics alone.
At the practically relevant \(2\%\) threshold, the Adaptive Topological
model attains \(76.0\%\) coverage (\(380\) of the \(500\) test samples),
compared with \(73.0\%\) for Fixed Topological, \(65.0\%\) for Sensor
Two-Step, \(60.2\%\) for Vanilla DeepONet, and \(44.8\%\) for the
full-field Two-Step model; coverages at the \(1\%\) and \(3\%\)
thresholds are reported in Table~\ref{tab:threshold-coverage}.
Thus, relative to Vanilla, the adaptive topological
representation improves the fraction of highly accurate predictions by
\(15.8\) percentage points at the \(2\%\) tolerance. These results
indicate that the learned topological measurements improve not only the
average predictive accuracy but also the consistency of the model
across the test ensemble, while retaining a compressed input
representation of only \(128\) measurements instead of the full
\(4096\)-dimensional field.
\begin{figure}[t]
    \centering
    \includegraphics[width=\linewidth]{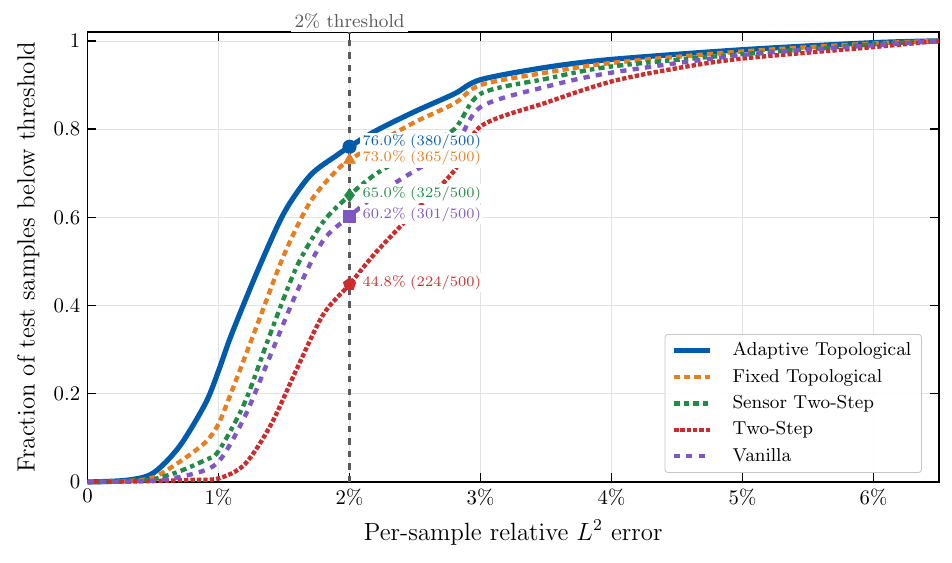}
    \caption{Empirical cumulative distribution function (ECDF) of per-sample
    relative $L^2$ errors on the Navier--Stokes test set ($N=500$) for five
    operator-learning models. The vertical dashed line marks the $2\%$ error
    threshold; filled markers indicate each model's cumulative fraction at
    that threshold. Cumulative fractions at the $1\%$, $2\%$, and $3\%$
    thresholds are summarised in Table~\ref{tab:threshold-coverage}.}
    \label{fig:ecdf_comparison}
\end{figure}
\begin{table}
\centering
\caption{
Percentage of test samples below selected relative \(L^2\)-error thresholds
for the fixed-time Navier--Stokes operator.
}
\label{tab:threshold-coverage}
\begin{tabular}{lccc}
\toprule
Model & \(1\%\) & \(2\%\) & \(3\%\) \\
\midrule
Adaptive Topological & \textbf{25.2} & \textbf{76.0} & \textbf{91.2} \\
Fixed Topological    & 13.2 & 73.0 & 90.0 \\
Sensor Two-Step      & 7.0  & 65.0 & 88.0 \\
Two-Step             & 0.8  & 44.8 & 80.6 \\
Vanilla              & 4.6  & 60.2 & 85.0 \\
\bottomrule
\end{tabular}
\end{table}
\subsection{Comparison with a parameter-matched Fourier neural operator.}
We additionally compare the proposed models with a direct Fourier neural
operator (FNO)~\cite{li2021fourier}. The FNO is a particularly strong
baseline for the present fixed-time Navier--Stokes benchmark because the
vorticity equation is solved on a uniform periodic domain. Fourier
convolution layers naturally respect the periodic geometry, provide direct
access to global spatial frequencies, and impose a translation-equivariant
spectral inductive bias that is well aligned with periodic vorticity
evolution. This benchmark therefore constitutes a favorable setting for the
FNO and provides a stringent comparison for the proposed
functional-coordinate models.
\begin{figure}[t]
\includegraphics[width=\textwidth]{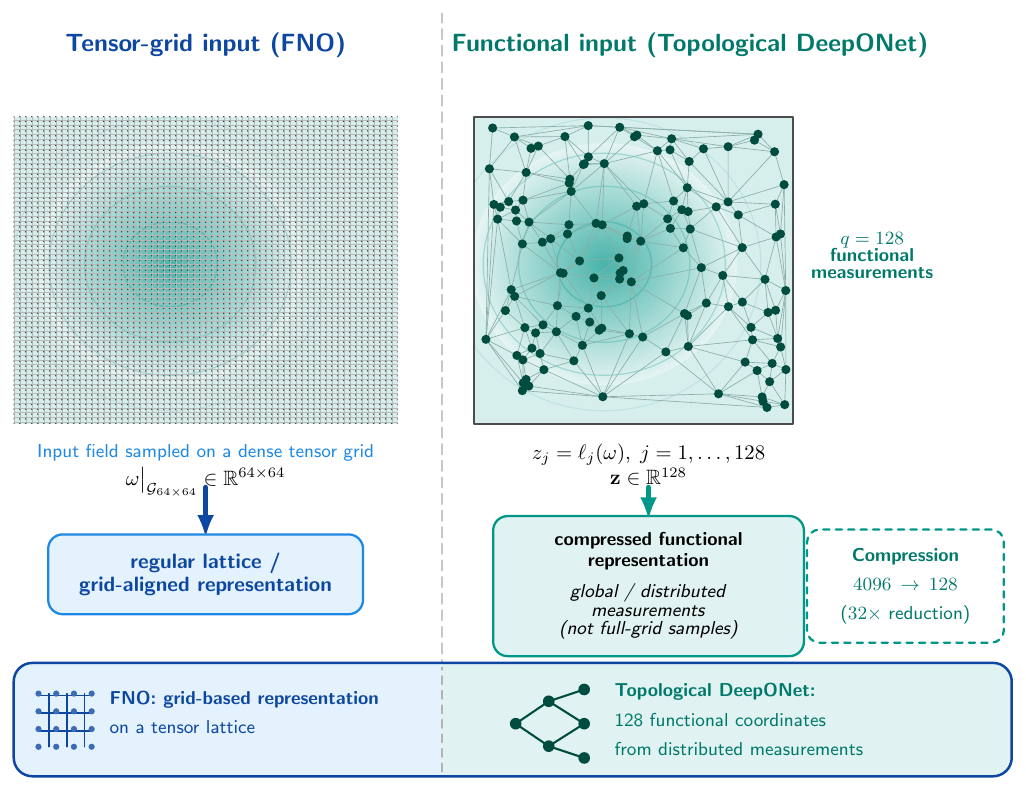}
\caption{
Schematic comparison between the tensor-grid representation used by the
Fourier neural operator and the continuous-functional representation used
by the Topological DeepONet for fixed time Navier-Stokes operator. The FNO receives the complete
\(64\times64\) field and applies a discrete Fourier transform and spectral
convolution on the uniform periodic grid. The Topological DeepONet instead
maps the input \(v\in(\mathcal V,\{p_\alpha\}_{\alpha\in A})\) to
\(q=128\) coordinates generated by continuous linear functionals
\(\ell_j\in\mathcal V'\). Numerically, each functional can be evaluated on
the available discretization using grid-dependent quadrature, without
requiring the branch input to contain the complete tensor-grid field.
}
\label{fig:fno-versus-functional-representation}
\end{figure}
Table~\ref{tab:ns-fno-comparison} reports the mean and sample standard
deviation over three independent runs with random seeds \(1234\), \(2345\),
and \(3456\). The FNO contains \(113{,}781\) trainable parameters, which is
within approximately \(5.4\%\) of the \(120{,}220\)-parameter Fixed
Topological and Sensor Two-Step models. The FNO receives the complete
\(64\times64\) vorticity field, corresponding to \(4096\) input values,
whereas the Fixed and Adaptive Topological DeepONets use only \(q=128\)
functional coordinates. The functional models therefore provide a
\(32\times\) reduction in branch-input dimension.
\begin{table*}[t]
\centering
\caption{
Multi-seed comparison for the fixed-time Navier--Stokes benchmark.
Results are reported as mean \(\pm\) sample standard deviation over three
independent runs. For the Adaptive Topological DeepONet, the reported
training time is end to end and includes both the initial Fixed Topological
training and the subsequent adaptive optimization. Lower values are better
for prediction error, training time, and peak GPU memory.
}
\label{tab:ns-fno-comparison}
\resizebox{\textwidth}{!}{
\begin{tabular}{lccccccc}
\toprule
\textbf{Model}
&
\textbf{Mean rel. \(L^2\) (\%)}
&
\textbf{Global rel. \(L^2\) (\%)}
&
\textbf{Training time (s)}
&
\textbf{Inference parameters}
&
\textbf{Peak GPU memory (MB)}
&
\textbf{Input dimension}
&
\textbf{Compression}
\\
\midrule
Fixed Topological
&
\(1.824 \pm 0.018\)
&
\(2.110 \pm 0.030\)
&
\(311.9 \pm 17.7\)
&
\(120{,}220\)
&
\(27.7\)
&
\(128\)
&
\(32\times\)
\\
Adaptive Topological
&
\(1.685 \pm 0.017\)
&
\(1.992 \pm 0.026\)
&
\(1012.7 \pm 38.4\)
&
\(193{,}948\)
&
\(30.1\)
&
\(128\)
&
\(32\times\)
\\
Sensor Two-Step
&
\(1.968 \pm 0.005\)
&
\(2.228 \pm 0.011\)
&
\(338.6 \pm 15.4\)
&
\(120{,}220\)
&
\(30.3\)
&
\(128\)
&
\(32\times\)
\\
Two-Step
&
\(2.424 \pm 0.028\)
&
\(2.706 \pm 0.033\)
&
\(371.3 \pm 40.3\)
&
\(628{,}124\)
&
\(41.6\)
&
\(4096\)
&
\(1\times\)
\\
Vanilla DeepONet
&
\(2.113 \pm 0.020\)
&
\(2.405 \pm 0.033\)
&
\(987.8 \pm 25.0\)
&
\(665{,}145\)
&
\(54.8\)
&
\(4096\)
&
\(1\times\)
\\
Direct FNO
&
\(\mathbf{0.832 \pm 0.172}\)
&
\(\mathbf{0.991 \pm 0.201}\)
&
\(2021.8 \pm 98.9\)
&
\(113{,}781\)
&
\(323.2\)
&
\(4096\)
&
\(1\times\)
\\
\bottomrule
\end{tabular}
}
\end{table*}
As shown in Table~\ref{tab:ns-fno-comparison}, the FNO achieves the lowest
average predictive error, with mean and global relative \(L^2\) errors of
\(0.832\%\pm0.172\%\) and \(0.991\%\pm0.201\%\), respectively. This result
is consistent with the strong alignment between the Fourier architecture
and the uniform periodic geometry of the benchmark. Among the DeepONet-based
models, the Adaptive Topological DeepONet performs best, attaining
\(1.685\%\pm0.017\%\) mean relative error and
\(1.992\%\pm0.026\%\) global relative error. The Fixed Topological model
follows with corresponding errors of \(1.824\%\pm0.018\%\) and
\(2.110\%\pm0.030\%\).
Although the FNO provides the best mean accuracy, it also exhibits
substantially greater sensitivity to random initialization. Its standard
deviation in mean relative error is approximately \(0.172\) percentage
points, nearly ten times the \(0.017\)-percentage-point standard deviation
of the Adaptive Topological model. A similar difference is observed for the
global relative error, for which the FNO standard deviation is
\(0.201\) percentage points compared with \(0.026\) percentage points for
the Adaptive model. The individual FNO mean errors range from approximately
\(0.682\%\) to \(1.020\%\), whereas the Adaptive Topological errors remain
within the much narrower interval \(1.665\%\)--\(1.698\%\). Thus, the FNO
offers superior average accuracy but noticeably greater seed-to-seed
variability, while the functional-coordinate models provide more consistent
performance across independent training runs. Because only three seeds are
available, this observation should be interpreted as evidence of relative
training stability rather than as a definitive statistical conclusion.
The gain in FNO accuracy is also accompanied by greater computational cost.
Its mean training time is \(2021.8\pm98.9\) s, compared with
\(1012.7\pm38.4\) s for the complete end-to-end Adaptive Topological
pipeline. The FNO therefore requires approximately twice the training time.
Its peak GPU memory is approximately \(323\) MB, compared with about
\(30\) MB for the Adaptive Topological model, corresponding to approximately
\(10.7\times\) greater memory consumption.
Figure~\ref{fig:ns-accuracy-cost} visualizes these accuracy--cost
relationships using the multi-seed means and one-standard-deviation error
bars. The FNO occupies the accuracy-optimal region but has the largest
uncertainty in predictive error and the highest training-time and memory
costs. The Topological DeepONets occupy a different part of the trade-off:
they use a \(32\times\) compressed functional representation, exhibit much
smaller seed-to-seed variation, and require substantially less GPU memory.
Consequently, the FNO is preferable when maximum average accuracy on a fixed
periodic grid is the principal objective, whereas the Topological DeepONets
provide a more compact, memory-efficient, and empirically stable
representation that can be evaluated across discretizations.
\begin{figure}[t]
\centering
\includegraphics[width=\textwidth]{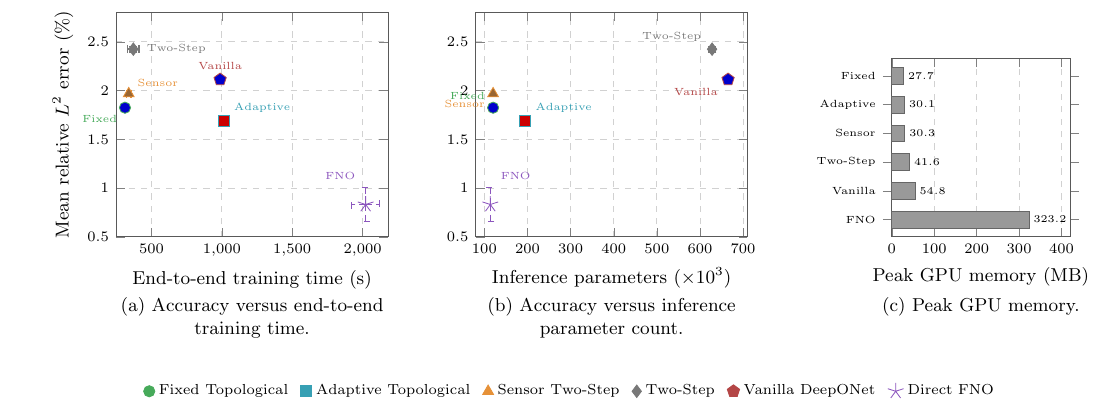}
\caption{%
Accuracy--cost comparison for the fixed-time Navier--Stokes benchmark.
The numerical values are reported in Table~\ref{tab:ns-fno-comparison}.
Markers in panels (a) and (b) show averages over three independent random
seeds, and error bars denote one sample standard deviation.
(a)~Mean relative \(L^2\) error versus end-to-end training time. The
Adaptive Topological time includes both the initial Fixed Topological
training and the subsequent adaptive optimization.
(b)~Mean relative \(L^2\) error versus inference parameter count.
(c)~Peak GPU memory. The FNO achieves the lowest mean error on the uniform
periodic benchmark, but it exhibits substantially greater seed-to-seed
variation and requires the largest training time and memory. The Fixed and
Adaptive Topological models use only \(q=128\) functional coordinates
instead of the full \(4096\)-value field, corresponding to a \(32\times\)
input compression.}
\label{fig:ns-accuracy-cost}
\end{figure}
\subsection{Time-evolving Navier--Stokes operator}
\label{sec:time-evolving-navier-stokes}
We consider the same two-dimensional incompressible Navier--Stokes
vorticity problem~\eqref{eq:ns-vorticity-main} on the periodic unit square
\(D=(0,1)^2\), with \(\nu=10^{-3}\) and prescribed forcing \(f\), but now
learn a trajectory-prediction operator rather than a single-snapshot
forecast.
Let
\begin{equation}
    \mathcal{V}
    =
    L_{\mathrm{per}}^2(D)
\end{equation}
denote the space of square-integrable periodic vorticity fields. The
time-evolving operator maps an input vorticity history to a future
vorticity trajectory:
\begin{equation}
    \mathcal{G}:
    \mathcal{V}^{T_{\mathrm{in}}}
    \longrightarrow
    \mathcal{V}^{T_{\mathrm{out}}},
    \qquad
    \left\{
        \omega(\cdot,t_n)
    \right\}_{n=0}^{T_{\mathrm{in}}-1}
    \longmapsto
    \left\{
        \omega(\cdot,t_n)
    \right\}_{n=T_{\mathrm{start}}}^{T_{\mathrm{end}}-1}.
    \label{eq:time-ns-operator}
\end{equation}
In the present experiment,
\begin{equation}
    T_{\mathrm{in}}=10,
    \qquad
    T_{\mathrm{start}}=10,
    \qquad
    T_{\mathrm{out}}=40,
\end{equation}
so that
\begin{equation}
    \mathcal{G}_{10\rightarrow40}:
    \left\{
        \omega(\cdot,t_0),\ldots,\omega(\cdot,t_9)
    \right\}
    \longmapsto
    \left\{
        \omega(\cdot,t_{10}),\ldots,\omega(\cdot,t_{49})
    \right\}.
    \label{eq:time-ns-10-to-40}
\end{equation}
The dataset contains \(5000\) realizations sampled on a
\(64\times64\) spatial grid at \(50\) stored time levels. We use
\(4000\), \(500\), and \(500\) realizations for training, validation,
and testing, respectively. After discretization, the input and output
trajectories belong to
\begin{equation}
    \mathbb{R}^{64\times64\times10}
    \qquad\text{and}\qquad
    \mathbb{R}^{64\times64\times40},
\end{equation}
respectively.
For the Fixed Topological Two-Step model, each input snapshot is represented
through \(q\) continuous linear functionals
\begin{equation}
    \ell_j\!\left(\omega(\cdot,t_n)\right)
    =
    \left\langle
        \omega(\cdot,t_n),\psi_j
    \right\rangle_{L^2(D)}
    =
    \int_D
        \omega(\bm{x},t_n)\psi_j(\bm{x})
    \,\mathrm{d}\bm{x},
    \qquad
    j=1,\ldots,q,
    \label{eq:time-ns-functional}
\end{equation}
where
\begin{equation}
    \psi_j\in L_{\mathrm{per}}^2(D).
\end{equation}
By the Cauchy--Schwarz inequality,
\begin{equation}
    \left|
        \ell_j(\omega)
    \right|
    \leq
    \|\omega\|_{L^2(D)}
    \|\psi_j\|_{L^2(D)},
\end{equation}
and therefore
\begin{equation}
    \ell_j\in\mathcal{V}'.
\end{equation}
The measurement functions are selected from a prescribed localized
multiscale dictionary containing centered Gaussian, directional
Gaussian-weighted, and Laplacian-of-Gaussian-type atoms at several spatial
scales. The default implementation uses
\begin{equation}
    q=128
\end{equation}
measurements at each of the ten input times. The resulting functional
history is
\begin{equation}
    \mathcal{M}_q(\omega)
    =
    \left[
        \ell_j\!\left(\omega(\cdot,t_n)\right)
    \right]_{
        \substack{
            n=0,\ldots,T_{\mathrm{in}}-1\\
            j=1,\ldots,q
        }
    }
    \in
    \mathbb{R}^{T_{\mathrm{in}}\times q}.
    \label{eq:time-ns-functional-history}
\end{equation}
A regularized dual synthesis map reconstructs a spatial proxy history from
these measurements. The proxy is then supplied to a three-dimensional
Fourier spectral-layer
branch encoder acting jointly on the two spatial coordinates and the input
time coordinate. The Fixed Topological Two-Step pipeline is therefore
\begin{equation}
    \left\{\omega(\cdot,t_n)\right\}_{n=0}^{9}
    \xrightarrow{\;\mathcal{M}_q\;}
    \mathbb{R}^{10\times128}
    \xrightarrow{\;\text{dual synthesis}\;}
    \widetilde{\omega}_q
    \xrightarrow{\;\text{3D FNO branch}\;}
    \widehat{\bm{c}}
    \xrightarrow{\;\text{output basis}\;}
    \widehat{\omega}.
    \label{eq:time-ns-fixed-pipeline}
\end{equation}
For the Adaptive Topological Two-Step model, the fixed measurements are
corrected by learned linear combinations of the complete dictionary
coordinates. If
\begin{equation}
    \bm{h}(t_n)
    =
    \left[
        \ell_1(\omega(\cdot,t_n)),\ldots,
        \ell_K(\omega(\cdot,t_n))
    \right]
\end{equation}
contains all dictionary measurements and \(\mathcal{I}_q\) denotes the
indices of the selected fixed measurements, then
\begin{equation}
    \bm{z}_{\mathrm{ad}}(t_n)
    =
    \bm{h}_{\mathcal{I}_q}(t_n)
    +
    \bm{h}(t_n)\Delta A_\theta,
    \qquad
    \Delta A_\theta\in\mathbb{R}^{K\times q}.
    \label{eq:time-ns-adaptive-measurement}
\end{equation}
The initialization \(\Delta A_\theta=0\) makes the adaptive and fixed
representations identical before adaptive training.
For all Two-Step models, the complete output trajectory is flattened and
projected onto a rank-\(r\) basis obtained from the singular value
decomposition of the centered training trajectories. Writing
\begin{equation}
    Y_c=U\Sigma V^{\mathsf T},
    \qquad
    Q=V_{:,1:r},
\end{equation}
the trajectory is approximated as
\begin{equation}
    \bm{y}
    \approx
    \overline{\bm{y}}
    +
    Q\bm{c}.
    \label{eq:time-ns-output-basis}
\end{equation}
The spectral branch predicts the reduced coefficient vector
\(\widehat{\bm{c}}\), after which the full
\(64\times64\times40\) trajectory is reconstructed using the fixed basis
\(Q\).
Further details on the multiscale dictionary, predictive selection,
dual synthesis, adaptive parameterization, output SVD construction, and
training procedure are provided in
Appendix~\ref{app:time-evolving-ns-details}.
Test errors for the \(10\)-to-\(40\) Navier--Stokes trajectory-prediction problem are summarized in the preceding discussion, whereas \autoref{tab:time-evolving-ns-computational} reports validation, optimization, timing, and model-size statistics.
The Adaptive Topological Two-Step model achieves the best performance
across all four reported error measures, attaining a mean relative
\(L^2\) error of \(4.5729\times10^{-2}\), a global relative \(L^2\)
error of \(4.8087\times10^{-2}\), an RMSE of
\(4.5017\times10^{-2}\), and an MAE of \(3.2172\times10^{-2}\). Relative to the Fixed Topological model, adaptive measurements reduce
the mean relative \(L^2\) error by approximately \(5.3\%\). The
improvement is approximately \(8.2\%\) relative to Sensor Two-Step and
\(4.2\%\) relative to the full-field Two-Step model. Compared with
Vanilla DeepONet, the adaptive model reduces the mean relative
\(L^2\) error by approximately \(48.1\%\).
The Fixed Topological model also outperforms the Sensor Two-Step model
despite using the same \(128\) measurements at each input time. Its mean
relative \(L^2\) error is \(4.8273\times10^{-2}\), compared with
\(4.9801\times10^{-2}\) for the sensor representation. This result
indicates that distributed multiscale functionals provide a more
informative compressed description of the vorticity history than an
equal number of pointwise samples.
The fixed, adaptive, and sensor models use \(128\) measurements for each
of the ten input snapshots, resulting in an effective input dimension of
\(1280\). In contrast, the full-field models use all
\(64\times64\times10=40960\) input values. The compressed models
therefore achieve a \(32\times\) reduction in input dimension. Despite this substantial compression, the Adaptive Topological model
outperforms the full-field Two-Step model, while the Fixed Topological
model remains competitive with it. In particular, the adaptive model
obtains a mean relative \(L^2\) error of \(4.5729\times10^{-2}\),
compared with \(4.7714\times10^{-2}\) for Full Two-Step. This result
suggests that the learned functional coordinates suppress redundant
components of the input history while retaining dynamically relevant
multiscale information.
\autoref{tab:time-evolving-threshold-coverage} and \autoref{fig:ecdf_nst}
report the percentage of
test trajectories whose sample-wise relative \(L^2\) error falls below
selected thresholds. The Adaptive Topological model achieves the highest
coverage at every reported threshold. In particular, \(38.0\%\) of its test
predictions have relative error below \(4\%\), compared with \(28.8\%\) for
the Full Two-Step model, \(27.0\%\) for the Fixed Topological model, and
\(19.6\%\) for the Sensor Two-Step model. No Vanilla DeepONet prediction
falls below the \(4\%\) threshold. At the \(5\%\) threshold, the Adaptive Topological model successfully predicts
\(78.4\%\) of the test trajectories. The corresponding coverage is
\(72.4\%\) for Fixed Topological, \(72.2\%\) for Full Two-Step, and
\(64.4\%\) for Sensor Two-Step, whereas Vanilla DeepONet still has no test
sample below this threshold. These results show that adaptive functional
measurements improve not only the mean prediction error but also the
consistency of the model across individual test realizations. At the more permissive \(10\%\) threshold, all Two-Step-based models attain
at least \(99.0\%\) coverage. The Adaptive Topological model remains the best,
with \(99.4\%\) of the test samples below \(10\%\), followed by Fixed
Topological and Full Two-Step at \(99.2\%\), Sensor Two-Step at \(99.0\%\),
and Vanilla DeepONet at \(90.6\%\). Thus, the topological and Two-Step models
produce substantially fewer high-error trajectories than Vanilla DeepONet.
\autoref{fig:time-evolving-ns-fields-errors-spectrum} provides a
representative comparison of the predicted time-evolving vorticity
fields. The complete-field predictions in the upper row are visually
similar because the dominant large-scale structures are recovered by all
models. To expose the differences that are obscured by the common
vorticity scale, the lower row reports localized absolute-error maps over
the shear-interaction region identified by the dashed boxes. The Adaptive Topological model produces the smallest localized error and
more accurately preserves the position and amplitude of the diagonal
shear structure. The Fixed Topological model remains close to the
adaptive result but exhibits a slightly broader and stronger error band.
The Vanilla DeepONet prediction shows the largest deviation, with a more
pronounced error concentrated along the interface separating the
positive- and negative-vorticity regions. This behavior is consistent
with the larger snapshot-wise and trajectory-wise relative \(L^2\)
errors reported for Vanilla DeepONet.
\begin{table}[t]
\centering
\caption{
Validation, optimization, and model-size statistics for the
time-evolving Navier--Stokes benchmark. Training times are reported in
hours. The adaptive measurement parameters correspond to the trainable
correction matrix \(\Delta A_\theta\).
}
\label{tab:time-evolving-ns-computational}
\resizebox{\linewidth}{!}{
\begin{tabular}{lccccc}
\toprule
Metric
& Fixed Topological
& Adaptive Topological
& Sensor Two-Step
& Full Two-Step
& Vanilla DeepONet
\\
\midrule
Best validation global relative \(L^2\)
& \(4.9488\times10^{-2}\)
& \(\mathbf{4.6427\times10^{-2}}\)
& \(5.1158\times10^{-2}\)
& \(4.8344\times10^{-2}\)
& \(8.7771\times10^{-2}\)
\\
Best epoch
& \(5300\)
& \(3000\)
& \(5400\)
& \(4400\)
& \(9900\)
\\
Training time (h)
& \(23.74\)
& \(12.39\)
& \(24.02\)
& \(22.34\)
& \(41.13\)
\\
Training parameters
& \(50{,}423{,}490\)
& \(50{,}489{,}020\)
& \(50{,}423{,}550\)
& \(50{,}423{,}490\)
& \(50{,}424{,}000\)
\\
Inference parameters
& \(50{,}423{,}490\)
& \(50{,}489{,}020\)
& \(50{,}423{,}550\)
& \(50{,}423{,}490\)
& \(50{,}424{,}000\)
\\
Measurement parameters
& \(0\)
& \(65{,}536\)
& \(0\)
& \(0\)
& \(0\)
\\
\bottomrule
\end{tabular}
}
\end{table}
\begin{figure}[t]
\centering
\includegraphics[width=\linewidth]{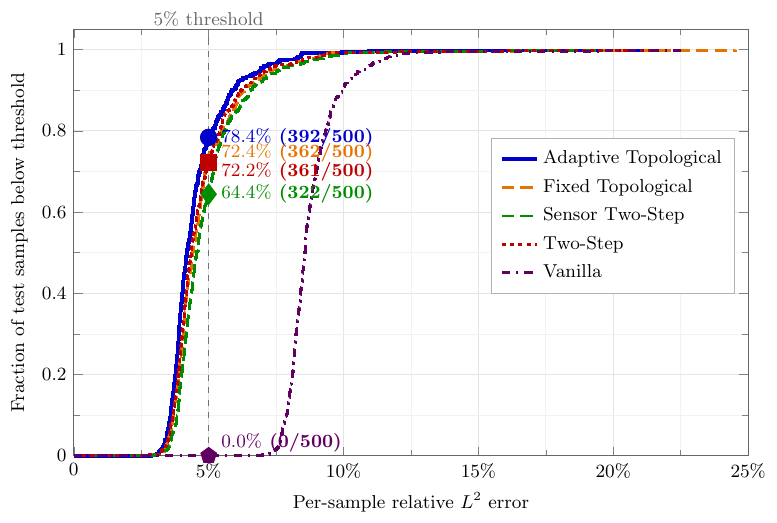}
\caption{
Empirical cumulative distribution functions (ECDFs) of the sample-wise
relative \(L^2\) trajectory errors on the Navier--Stokes test set
(\(N_{\mathrm{test}}=500\)) for the five operator-learning models.
Curves farther toward the upper-left indicate that a larger fraction of
test trajectories is predicted with lower error. The vertical dashed line
marks the \(5\%\) relative-error threshold, and the filled markers indicate
the corresponding cumulative fraction for each model. The percentages of
test samples below the \(4\%\), \(5\%\), and \(10\%\) thresholds are
summarized in Table~\ref{tab:time-evolving-threshold-coverage}.
}
\label{fig:ecdf_nst}
\end{figure}
\begin{table}[t]
\centering
\caption{Percentage of test samples below selected relative
\(L^2\)-error thresholds for the time-evolving Navier--Stokes benchmark.}
\label{tab:time-evolving-threshold-coverage}
\renewcommand{\arraystretch}{1.08}
\setlength{\tabcolsep}{10pt}
\begin{tabular}{lccc}
\toprule
Model & \(4\%\) & \(5\%\) & \(10\%\) \\
\midrule
Adaptive Topological
& \textbf{38.0}
& \textbf{78.4}
& \textbf{99.4}
\\
Fixed Topological
& 27.0
& 72.4
& 99.2
\\
Sensor Two-Step
& 19.6
& 64.4
& 99.0
\\
Full Two-Step
& 28.8
& 72.2
& 99.2
\\
Vanilla DeepONet
& 0.0
& 0.0
& 90.6
\\
\bottomrule
\end{tabular}
\end{table}
\begin{figure}[t]
    \centering
    \includegraphics[width=\textwidth]
    {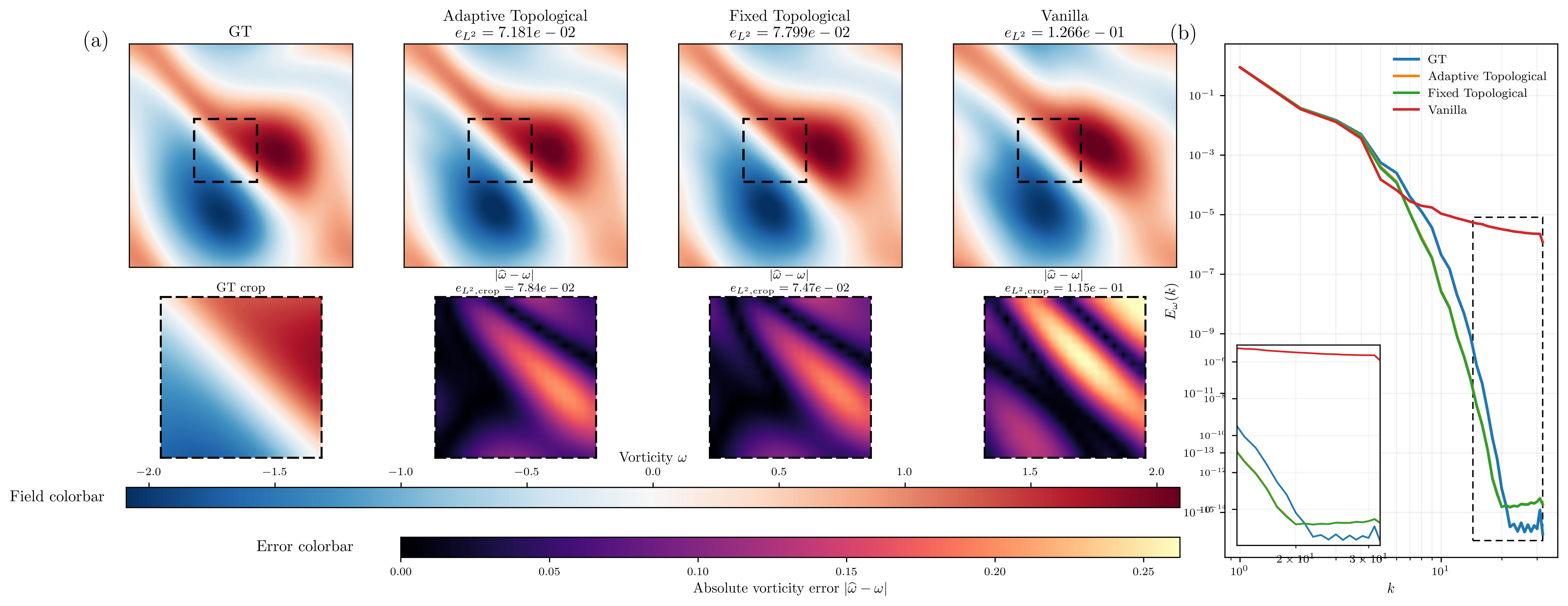}
    \caption{
    Qualitative and spectral comparison for the time-evolving
Navier--Stokes benchmark. Panel~\textbf{(a)} shows the ground-truth
vorticity field and predictions from the Adaptive Topological, Fixed
Topological, and Vanilla DeepONet models at a selected output time. A
common symmetric color scale is used for all full-field panels. The
dashed rectangles identify a localized shear-interaction region, while
the lower row shows the corresponding ground-truth crop and absolute
prediction errors \( |\widehat{\omega}-\omega| \) using a common error
scale. Panel~\textbf{(b)} compares the isotropic vorticity spectra
\(E_{\omega}(k)\), averaged over the complete predicted trajectory for
the same test realization. The inset magnifies the high-wavenumber range
indicated by the dashed box.
    }
    \label{fig:time-evolving-ns-fields-errors-spectrum}
\end{figure}
\subsection{Distribution-valued screened Poisson operator}
\label{sec:distributional-screened-poisson}
To evaluate the proposed framework beyond normed function spaces, we
consider a screened Poisson equation with distribution-valued forcing. Let
\begin{equation}
(-\Delta+I)u=\mu
\quad\text{in }\Omega=(0,1)^2,
\qquad
u=0
\quad\text{on }\partial\Omega,
\end{equation}
with distribution-valued input
\begin{equation}
\mu=\sum_{i=1}^{N_f}a_i\delta_{\bm{x}_i}.
\end{equation}
Thus,
\begin{equation}
\mathcal{G}:\mathcal{M}(\Omega)\rightarrow H_0^1(\Omega),
\qquad
\mu\mapsto u.
\end{equation}
For smooth test functions \(\{\phi_j\}_{j=1}^K\), the functional coordinates are evaluated directly from the source list as
\begin{equation}
\ell_j(\mu)
=
\langle \mu,\phi_j\rangle
=
\sum_{i=1}^{N_f}a_i\phi_j(\bm{x}_i),
\end{equation}
without rasterizing the input onto a common grid.
The source locations are sampled uniformly from \([0.05,0.95]^2\), with amplitudes \(a_i\sim\mathcal{N}(0,1)\). The in-distribution data use \(N_f\in\{1,\ldots,15\}\), while the out-of-distribution test set uses \(N_f\in\{16,\ldots,25\}\). Reference solutions are obtained from a truncated sine-series expansion and evaluated on a \(64\times64\) grid. We compare the proposed functional-measurement models with two baselines
designed for variable-cardinality or discretized inputs. First, the
\emph{DeepSets source-list} model represents each realization by the
unordered set
\[
\mathcal{S}
=
\left\{
(x_{i,1},x_{i,2},a_i)
\right\}_{i=1}^{N_f}.
\]
Following the permutation-invariant Deep Sets architecture
\cite{zaheer2017deep}, the same encoder \(\phi\) is applied to every source,
the resulting features are aggregated by summation, and a second network
\(\rho\) predicts the reduced output coefficients:
\begin{equation}
\widehat{\bm c}
=
\rho\left(
\sum_{i=1}^{N_f}
\phi(x_{i,1},x_{i,2},a_i)
\right).
\end{equation}
This construction directly accommodates a variable number of sources and
is invariant to their ordering, but its latent coordinates are generic
nonlinear learned features rather than prescribed continuous linear
functionals of the input measure. Deep Sets was introduced specifically
for permutation-invariant learning on set-valued inputs.
Second, the \emph{Rasterized Two-Step} baseline deposits the signed Dirac
sources onto a fixed \(32\times32\) Cartesian grid using bilinear weights.
The resulting \(1024\)-dimensional grid vector is passed to a branch network
that predicts coefficients in the same POD output basis used by all other
models. This model follows the coefficient-space Two-Step DeepONet strategy,
in which the output representation and branch regression are trained as
separate stages \cite{lee2024training}. Unlike the functional models, the
rasterized baseline depends on a prescribed input grid and introduces an
additional approximation through source deposition. The Two-Step strategy
was proposed to reduce the difficulty of jointly optimizing the DeepONet
branch and trunk representations.
\begin{table}[t]
\centering
\caption{
In-distribution and out-of-distribution performance for the
distribution-valued screened Poisson benchmark. Models are trained using
\(1\leq N_f\leq15\). The OOD set contains \(16\leq N_f\leq25\) sources.
}
\label{tab:distributional-screened-poisson-id-ood}
\begin{tabular}{lcccc}
\toprule
& \multicolumn{2}{c}{ID: \(1\leq N_f\leq15\)}
& \multicolumn{2}{c}{OOD: \(16\leq N_f\leq25\)} \\
\cmidrule(lr){2-3}
\cmidrule(lr){4-5}
Model
& Mean
& Global
& Mean
& Global \\
\midrule
Fixed Functional
& \(3.547\times10^{-1}\)
& \(3.431\times10^{-1}\)
& \(4.228\times10^{-1}\)
& \(4.590\times10^{-1}\) \\
Adaptive Functional
& \(4.153\times10^{-1}\)
& \(\mathbf{3.118\times10^{-1}}\)
& \(\mathbf{3.980\times10^{-1}}\)
& \(\mathbf{4.237\times10^{-1}}\) \\
DeepSets Source List
& \(1.020\)
& \(8.096\times10^{-1}\)
& \(8.855\times10^{-1}\)
& \(8.223\times10^{-1}\) \\
Rasterized Two-Step
& \(8.050\times10^{-1}\)
& \(6.400\times10^{-1}\)
& \(6.840\times10^{-1}\)
& \(7.075\times10^{-1}\) \\
\bottomrule
\end{tabular}
\end{table}
\begin{figure}[t]
    \centering
    \includegraphics[width=\textwidth]{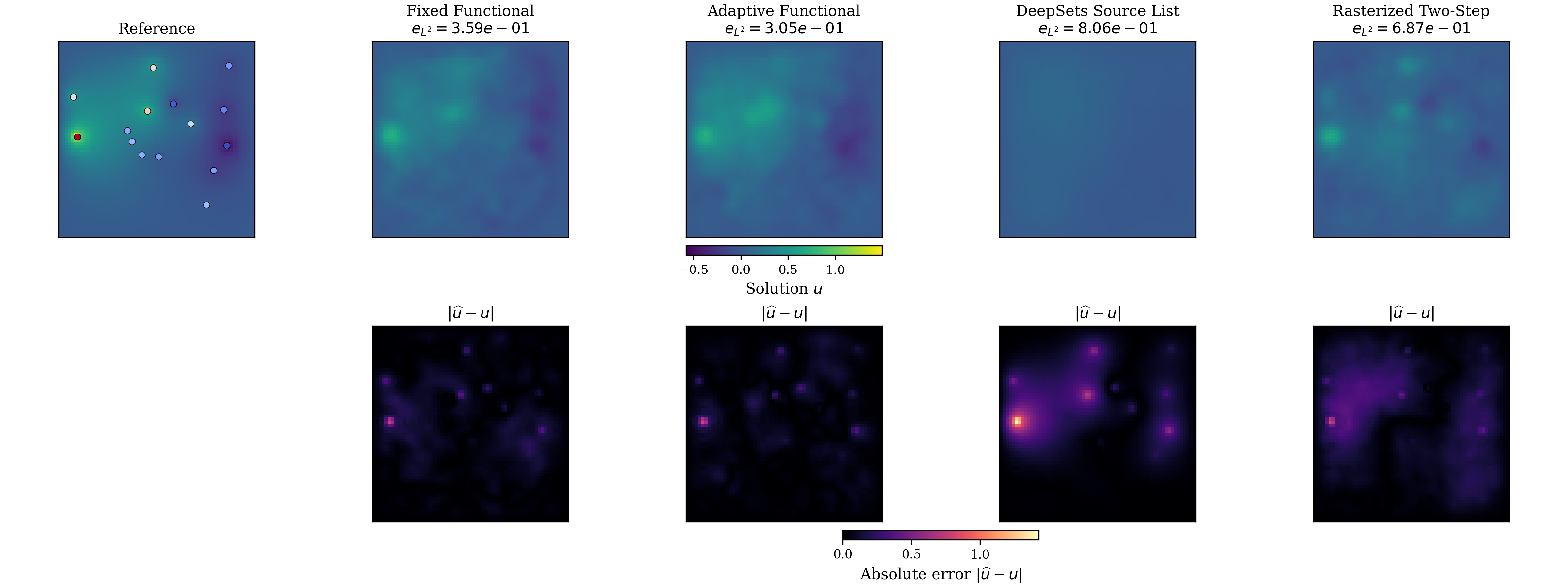}
    \caption{
    Representative prediction for the distribution-valued screened Poisson benchmark.
    The top row shows the reference solution and the predictions from the Fixed Functional,
    Adaptive Functional, DeepSets source-list, and Rasterized Two-Step models. The bottom
    row shows the corresponding absolute-error fields. The functional models recover the
    dominant spatial structures more accurately than the two baseline representations.
    }
    \label{fig:distributional-screened-poisson-predictions}
\end{figure}
The functional-measurement models substantially outperform the DeepSets and rasterized baselines shown in \autoref{tab:distributional-screened-poisson-id-ood}. The Fixed Functional model gives the lowest mean and \(P_{95}\) errors, reducing the mean error by approximately \(55.9\%\) relative to Rasterized Two-Step and \(65.2\%\) relative to DeepSets. The Adaptive Functional model achieves the lowest global relative \(L^2\) error, although its larger standard deviation and \(P_{95}\) indicate a small number of high-error outliers. These results demonstrate the benefit of evaluating continuous-dual measurements directly on variable-cardinality, distribution-valued inputs. On the out-of-distribution test set (\autoref{tab:distributional-screened-poisson-id-ood}), the Adaptive Functional model achieves
the best performance across all reported error statistics, with a mean
relative \(L^2\) error of \(39.80\%\) and a global relative \(L^2\) error
of \(42.37\%\). The Fixed Functional model follows with corresponding
errors of \(42.28\%\) and \(45.90\%\). Both functional models substantially
outperform the Rasterized Two-Step and DeepSets baselines.
Relative to Rasterized Two-Step, the Adaptive Functional model reduces the
OOD mean error by approximately
\[
\frac{0.6840-0.3980}{0.6840}\times100
\approx 41.8\%.
\]
Relative to DeepSets, the reduction is approximately \(55.1\%\). These
results indicate that the continuous-functional representation generalizes
more effectively to measures whose cardinality exceeds the entire training
range. Figure~\ref{fig:distributional-screened-poisson-predictions} shows a representative
test realization. The Fixed and Adaptive Functional models reproduce the principal
localized solution structures, with sample-wise relative errors of \(35.9\%\) and
\(30.5\%\), respectively. In contrast, the DeepSets source-list prediction is nearly
uniform, while the Rasterized Two-Step model captures only part of the spatial response,
leading to substantially larger errors.
\section{Summary} \label{sec:summary}
We developed fixed and adaptive Topological DeepONets as computational
realizations of the operator-learning framework of
Ismailov~\cite{ismailov2026topological}. In contrast to conventional
DeepONets, whose branch inputs are point evaluations tied to a prescribed
sensor grid, the proposed models encode the input through continuous linear
functionals drawn from the dual of a Hausdorff locally convex space. Both
variants were combined with an enhanced Two-Step training
strategy~\cite{lee2024training}: a weighted singular value decomposition
constructs a rank-stable, discretely orthonormal output basis, and the branch
network is trained directly in the corresponding coefficient space. In the
adaptive formulation, the measurement functionals and coefficient map are
learned jointly, while a training-only reconstruction decoder and soft
regularization reduce information loss and feature collapse without
increasing inference-time complexity.
On the theoretical side,
Theorem~\ref{thm:discrete-topological-error} decomposes the discrete
approximation error into measurement-reconstruction, output-basis truncation,
and neural-approximation contributions, while
Corollary~\ref{cor:barron-rate} provides an explicit Barron-rate refinement
of the network term. The estimate was verified numerically for the
antiderivative operator: all \(400\) test samples satisfied the samplewise
bound, with maximum sharpness ratios of \(0.94\) and \(0.97\) for the
adaptive and fixed models, respectively. We further considered a benchmark
posed on a genuinely non-normable locally convex input space. This experiment
demonstrates that the proposed construction remains applicable when the input
topology is generated by a family of seminorms and cannot be represented by a
single norm, thereby providing a direct computational illustration of the
extension beyond the Banach-space setting.
The computational experiments support several conclusions. First,
distributed functional measurements outperform point sensors under matched
representation budgets. On the Darcy problem, the Fixed Topological DeepONet
reduces the global relative \(L^2\) error from \(10.39\%\) to \(5.88\%\)
relative to a parameter-matched point-sensor model using the same
\(32\)-dimensional input. In the heterogeneous-resolution Darcy experiment,
the functional models also retain nearly resolution-independent mean errors
of approximately \(5.5\%\)--\(5.6\%\) on previously unseen grids and degrade
substantially less than interpolation-based baselines under missing
observations. Second, learning the measurement functionals produces further
gains at essentially unchanged inference cost. The adaptive model attains the
lowest error on the antiderivative benchmark, with mean relative \(L^2\)
error \(2.12\times10^{-2}\), and outperforms the Two-Step baseline on
\(72\%\) of the test functions. It also achieves the highest coverage of
accurate predictions among the DeepONet-based models on both
Navier--Stokes problems, with \(76.0\%\) of fixed-time samples below
\(2\%\) error and \(78.4\%\) of time-evolving trajectories below \(5\%\)
error. Third, the functional coordinates retain dynamically relevant
information despite substantial compression: with a \(32\times\) reduction
in input dimension, the adaptive model surpasses the full-field Two-Step
model on the time-evolving trajectory-prediction problem, attaining mean
relative \(L^2\) errors of \(4.57\times10^{-2}\) and
\(4.77\times10^{-2}\), respectively.
The comparison with a direct Fourier neural operator
(FNO)~\cite{li2021fourier} clarifies the scope of these advantages. On the
uniform periodic fixed-time Navier--Stokes benchmark, the FNO achieves the
lowest three-seed mean relative \(L^2\) error,
\(0.832\%\pm0.172\%\), whereas the Adaptive Topological DeepONet attains
\(1.685\%\pm0.017\%\) using only \(128\) functional coordinates. Thus, the
FNO provides higher average accuracy in a setting strongly aligned with its
spectral inductive bias, but it exhibits substantially larger seed-to-seed
variability, uses the complete \(64\times64\) input field, requires
approximately twice the end-to-end training time, and consumes about
\(10.7\times\) more peak GPU memory. The principal advantage of the proposed
framework is therefore not universal superiority over grid-specific neural
operators, but the construction of compact, interpretable,
operator-adapted, and discretization-portable coordinates in the continuous
dual \(\mathcal V'\), including for input spaces that are not normable.
Future work includes physics-informed \cite{wang2021physicsinformed,li2021pino} extensions of the topological
framework, principled and data-driven design of measurement dictionaries,
applications to variable-resolution and multimodal experimental data, and an
extension of the discrete error analysis of
Section~\ref{sec:theory} to fully infinite-dimensional locally convex
settings.
\section*{Acknowledgments}
KS and GEK acknowledge partial support from the U.S. Department of Energy (DOE), Office of Science, Advanced Scientific Computing Research (ASCR) program, through the Scientific Discovery through Advanced Computing (SciDAC) Institute "LEADS: LEarning-Accelerated Domain Science" (Subcontract 831126 under DE-AC05-76RL01830). GEK further acknowledges support from the ONR Vannevar Bush Faculty Fellowship (N00014-22-1-2795).
\section*{Code and Reproducibility}
The complete implementation, including data-generation scripts,
training configurations, and evaluation code for all four benchmarks, will be made publicly available in a permanent repository upon publication; the archival URL will be inserted in the final version.
\newpage
\bibliographystyle{unsrtnat}
\bibliography{ref}

@article{ismailov2026topological,
  title={Topological DeepONets and a generalization of the Chen-Chen operator approximation theorem},
  author={Ismailov, Vugar},
  journal={arXiv preprint arXiv:2603.11972},
  year={2026}
}

@article{lee2024training,
  title={On the training and generalization of deep operator networks},
  author={Lee, Sanghyun and Shin, Yeonjong},
  journal={SIAM Journal on Scientific Computing},
  volume={46},
  number={4},
  pages={C273--C296},
  year={2024},
  publisher={SIAM}
}

@article{lu2021deeponet,
  author  = {Lu, Lu and Jin, Pengzhan and Pang, Guofei and Zhang, Zhongqiang
             and Karniadakis, George Em},
  title   = {Learning Nonlinear Operators via {DeepONet} Based on the
             Universal Approximation Theorem of Operators},
  journal = {Nature Machine Intelligence},
  volume  = {3},
  number  = {3},
  pages   = {218--229},
  year    = {2021},
  doi     = {10.1038/s42256-021-00302-5}
}

@inproceedings{li2021fourier,
  author    = {Li, Zongyi and Kovachki, Nikola and Azizzadenesheli, Kamyar
               and Liu, Burigede and Bhattacharya, Kaushik and Stuart, Andrew
               and Anandkumar, Anima},
  title     = {Fourier Neural Operator for Parametric Partial Differential Equations},
  booktitle = {International Conference on Learning Representations},
  year      = {2021},
  url       = {https://openreview.net/forum?id=c8P9NQVtmnO}
}

@article{chen1995operator,
  author  = {Chen, Tianping and Chen, Hong},
  title   = {Universal Approximation to Nonlinear Operators by Neural
             Networks with Arbitrary Activation Functions and Its Application
             to Dynamical Systems},
  journal = {IEEE Transactions on Neural Networks},
  volume  = {6},
  number  = {4},
  pages   = {911--917},
  year    = {1995},
  doi     = {10.1109/72.392253}
}

@article{venturi2023svd,
  author  = {Venturi, Simone and Casey, Tiernan},
  title   = {{SVD} Perspectives for Augmenting {DeepONet} Flexibility and
             Interpretability},
  journal = {Computer Methods in Applied Mechanics and Engineering},
  volume  = {403},
  pages   = {115718},
  year    = {2023},
  doi     = {10.1016/j.cma.2022.115718}
}

@article{kovachki2023neuraloperator,
  author  = {Kovachki, Nikola B. and Li, Zongyi and Liu, Burigede
             and Azizzadenesheli, Kamyar and Bhattacharya, Kaushik
             and Stuart, Andrew M. and Anandkumar, Anima},
  title   = {Neural Operator: Learning Maps Between Function Spaces with
             Applications to {PDEs}},
  journal = {Journal of Machine Learning Research},
  volume  = {24},
  number  = {89},
  pages   = {1--97},
  year    = {2023}
}

@article{wang2021physicsinformed,
  author        = {Wang, Sifan and Wang, Hanwen and Perdikaris, Paris},
  title         = {Learning the Solution Operator of Parametric Partial
                   Differential Equations with Physics-Informed {DeepONets}},
  journal       = {Science Advances},
  volume        = {7},
  number        = {40},
  pages         = {eabi8605},
  year          = {2021},
  doi           = {10.1126/sciadv.abi8605}
}

@article{jin2022mionet,
  author  = {Jin, Pengzhan and Meng, Shuai and Lu, Lu},
  title   = {{MIONet}: Learning Multiple-Input Operators via Tensor Product},
  journal = {SIAM Journal on Scientific Computing},
  volume  = {44},
  number  = {6},
  pages   = {A3490--A3514},
  year    = {2022},
  doi     = {10.1137/22M1477751}
}

@article{liu2022multiscale,
  author        = {Liu, Lizuo and Cai, Wei},
  title         = {Multiscale {DeepONet} for Nonlinear Operators in
                   Oscillatory Function Spaces for Building Seismic Wave
                   Responses},
  journal       = {Engineering Structures},
  volume        = {276},
  pages         = {115221},
  year          = {2023}
}

@article{shukla2024deep,
  title={Deep neural operators as accurate surrogates for shape optimization},
  author={Shukla, Khemraj and Oommen, Vivek and Peyvan, Ahmad and Penwarden, Michael and Plewacki, Nicholas and Bravo, Luis and Ghoshal, Anindya and Kirby, Robert M and Karniadakis, George Em},
  journal={Engineering Applications of Artificial Intelligence},
  volume={129},
  pages={107615},
  year={2024},
  publisher={Elsevier}
}

@article{oommen2022learning,
  title={Learning two-phase microstructure evolution using neural operators and autoencoder architectures},
  author={Oommen, Vivek and Shukla, Khemraj and Goswami, Somdatta and Dingreville, R{\'e}mi and Karniadakis, George Em},
  journal={npj Computational Materials},
  volume={8},
  number={1},
  pages={190},
  year={2022},
  publisher={Nature Publishing Group UK London}
}

@article{oommen2024rethinking,
  title={Rethinking materials simulations: Blending direct numerical simulations with neural operators},
  author={Oommen, Vivek and Shukla, Khemraj and Desai, Saaketh and Dingreville, R{\'e}mi and Karniadakis, George Em},
  journal={npj Computational Materials},
  volume={10},
  number={1},
  pages={145},
  year={2024},
  publisher={Nature Publishing Group UK London}
}

@article{de2025deep,
  title={Deep operator neural network model predictive control},
  author={de Jong, Thomas Oliver and Shukla, Khemraj and Lazar, Mircea},
  journal={IEEE Open Journal of Control Systems},
  year={2025},
  publisher={IEEE}
}

@article{laudato2025neural,
  title={Neural operator modeling of platelet geometry and stress in shear flow},
  author={Laudato, Marco and Manzari, Luca and Shukla, Khemraj},
  journal={arXiv preprint arXiv:2503.12074},
  year={2025}
}

@article{shukla2024deepI,
  title={Deep operator learning-based surrogate models for aerothermodynamic analysis of AEDC hypersonic waverider},
  author={Shukla, Khemraj and Ratchford, Jasmine and Bravo, Luis and Oommen, Vivek and Plewacki, Nicholas and Ghoshal, Anindya and Karniadakis, George},
  journal={arXiv preprint arXiv:2405.13234},
  year={2024}
}

@article{prasthofer2022vidon,
  title         = {Variable-Input Deep Operator Networks},
  author        = {Prasthofer, Michael and De Ryck, Tim and Mishra, Siddhartha},
  journal       = {arXiv preprint arXiv:2205.11404},
  year          = {2022},
  doi           = {10.48550/arXiv.2205.11404},
  eprint        = {2205.11404},
  archivePrefix = {arXiv},
  primaryClass  = {cs.LG}
}

@article{zhang2022belnet,
  title         = {{BelNet}: Basis Enhanced Learning, a Mesh-Free Neural Operator},
  author        = {Zhang, Zecheng and Leung, Wing Tat and Schaeffer, Hayden},
  journal       = {arXiv preprint arXiv:2212.07336},
  year          = {2022},
  doi           = {10.48550/arXiv.2212.07336},
  eprint        = {2212.07336},
  archivePrefix = {arXiv},
  primaryClass  = {math.NA}
}

@article{zhang2023discretization,
  title         = {A Discretization-Invariant Extension and Analysis of Some
                   Deep Operator Networks},
  author        = {Zhang, Zecheng and Leung, Wing Tat and Schaeffer, Hayden},
  journal       = {arXiv preprint arXiv:2307.09738},
  year          = {2023},
  doi           = {10.48550/arXiv.2307.09738},
  eprint        = {2307.09738},
  archivePrefix = {arXiv},
  primaryClass  = {math.NA}
}

@article{lanthaler2021error,
  title         = {Error Estimates for {DeepONets}: A Deep Learning Framework
                   in Infinite Dimensions},
  author        = {Lanthaler, Samuel and Mishra, Siddhartha and
                   Karniadakis, George Em},
  journal       = {arXiv preprint arXiv:2102.09618},
  year          = {2021},
  doi           = {10.48550/arXiv.2102.09618},
  eprint        = {2102.09618},
  archivePrefix = {arXiv},
  primaryClass  = {math.NA}
}

@inproceedings{seidman2022nomad,
  title     = {{NOMAD}: Nonlinear Manifold Decoders for Operator Learning},
  author    = {Seidman, Jacob H. and Kissas, Georgios and
               Perdikaris, Paris and Pappas, George J.},
  booktitle = {Advances in Neural Information Processing Systems},
  volume    = {35},
  year      = {2022}
}

@article{shukla2026uncertainty,
  title={Uncertainty Quantification in PINNs for Turbulent Flows: Bayesian Inference and Repulsive Ensembles},
  author={Shukla, Khemraj and Zou, Zongren and Kaeufer, Theo and Triantafyllou, Michael and Karniadakis, George Em},
  journal={arXiv preprint arXiv:2604.17156},
  year={2026}
}

@inproceedings{zaheer2017deep,
  title     = {Deep Sets},
  author    = {Zaheer, Manzil and Kottur, Satwik and Ravanbakhsh, Siamak
               and Poczos, Barnabas and Salakhutdinov, Ruslan
               and Smola, Alexander J.},
  booktitle = {Advances in Neural Information Processing Systems},
  volume    = {30},
  pages     = {3391--3401},
  year      = {2017}
}

@article{raissi2019physics,
  title   = {Physics-Informed Neural Networks: A Deep Learning Framework for Solving Forward and Inverse Problems Involving Nonlinear Partial Differential Equations},
  author  = {Raissi, Maziar and Perdikaris, Paris and Karniadakis, George Em},
  journal = {Journal of Computational Physics},
  volume  = {378},
  pages   = {686--707},
  year    = {2019},
  doi     = {10.1016/j.jcp.2018.10.045}
}

@article{karniadakis2021physics,
  title   = {Physics-Informed Machine Learning},
  author  = {Karniadakis, George Em and Kevrekidis, Ioannis G. and Lu, Lu and Perdikaris, Paris and Wang, Sifan and Yang, Liu},
  journal = {Nature Reviews Physics},
  volume  = {3},
  pages   = {422--440},
  year    = {2021},
  doi     = {10.1038/s42254-021-00314-5}
}

@article{rackauckas2020universal,
  title   = {Universal Differential Equations for Scientific Machine Learning},
  author  = {Rackauckas, Christopher and Ma, Yingbo and Martensen, Julius and Warner, Collin and Zubov, Kirill and Supekar, Rohit and Skinner, Dominic and Ramadhan, Ali and Edelman, Alan},
  journal = {arXiv preprint arXiv:2001.04385},
  year    = {2020},
  eprint  = {2001.04385},
  archivePrefix = {arXiv}
}

@article{li2021pino,
  title   = {Physics-Informed Neural Operator for Learning Partial Differential Equations},
  author  = {Li, Zongyi and Zheng, Hongkai and Kovachki, Nikola and Jin, David and Chen, Haoxuan and Liu, Burigede and Azizzadenesheli, Kamyar and Anandkumar, Anima},
  journal = {arXiv preprint arXiv:2111.03794},
  year    = {2021},
  eprint  = {2111.03794},
  archivePrefix = {arXiv}
}

@inproceedings{gupta2021multiwavelet,
  title     = {Multiwavelet-Based Operator Learning for Differential Equations},
  author    = {Gupta, Gaurav and Xiao, Xiongye and Bogdan, Paul},
  booktitle = {Advances in Neural Information Processing Systems},
  volume    = {34},
  pages     = {24048--24062},
  year      = {2021}
}

@inproceedings{hao2023gnot,
  title     = {{GNOT}: A General Neural Operator Transformer for Operator Learning},
  author    = {Hao, Zhongkai and Wang, Zhengyi and Su, Hang and Ying, Chengyang and Dong, Yinpeng and Liu, Songming and Cheng, Ze and Song, Jun and Zhu, Jun},
  booktitle = {Proceedings of the 40th International Conference on Machine Learning},
  series    = {Proceedings of Machine Learning Research},
  volume    = {202},
  pages     = {12556--12569},
  year      = {2023}
}

@inproceedings{pfaff2021meshgraphnets,
  title     = {Learning Mesh-Based Simulation with Graph Networks},
  author    = {Pfaff, Tobias and Fortunato, Meire and Sanchez-Gonzalez, Alvaro and Battaglia, Peter W.},
  booktitle = {International Conference on Learning Representations},
  year      = {2021}
}

@article{cuomo2022scientific,
  title   = {Scientific Machine Learning through Physics-Informed Neural Networks: Where We Are and What's Next},
  author  = {Cuomo, Salvatore and Schiano Di Cola, Vincenzo and Giampaolo, Fabio and Rozza, Gianluigi and Raissi, Maziar and Piccialli, Francesco},
  journal = {Journal of Scientific Computing},
  volume  = {92},
  number  = {3},
  pages   = {88},
  year    = {2022},
  doi     = {10.1007/s10915-022-01939-z}
}

@inproceedings{serrano2023coral,
  title     = {Operator Learning with Neural Fields: Tackling {PDEs} on General Geometries},
  author    = {Serrano, Louis and Le Boudec, Lise and Koupa{\"i}, Armand Kassa{\"i} and Wang, Thomas X. and Yin, Yuan and Vittaut, Jean-No{\"e}l and Gallinari, Patrick},
  booktitle = {Advances in Neural Information Processing Systems},
  volume    = {36},
  year      = {2023}
}
\appendix
\section{Details of the antiderivative convergence study}
\label{app:antiderivative-details}
\subsection{Discretization and data generation}
The interval \([0,1]\) is discretized using
\[
    m=256
\]
uniformly spaced points
\[
    x_j=\frac{j}{m-1},
    \qquad
    j=0,\ldots,m-1.
\]
Each input function is represented by
\[
    v_h
    =
    \bigl(
        v(x_0),\ldots,v(x_{m-1})
    \bigr)^\top
    \in\mathbb R^m.
\]
The discrete antiderivative is computed using the cumulative composite
trapezoidal rule:
\[
    \bigl[\mathcal G_h(v_h)\bigr]_0=0,
\]
and
\[
    \bigl[\mathcal G_h(v_h)\bigr]_j
    =
    \sum_{k=0}^{j-1}
    \frac{x_{k+1}-x_k}{2}
    \left(
        v(x_k)+v(x_{k+1})
    \right),
    \qquad
    j=1,\ldots,m-1.
\]
The input functions are generated as random truncated Fourier series,
\[
    v^{(i)}(x)
    =
    a_0^{(i)}
    +
    \sum_{k=1}^{K}
    \left[
        a_k^{(i)}\cos(2\pi kx)
        +
        b_k^{(i)}\sin(2\pi kx)
    \right],
\]
with
\[
    K=40.
\]
The coefficients are sampled independently according to
\[
    a_k^{(i)},b_k^{(i)}
    \sim
    \mathcal N(0,k^{-2}),
    \qquad
    k=1,\ldots,K,
\]
Here $k^{-2}$ denotes the variance, so the corresponding standard
deviation is $k^{-1}$.
while
\[
    a_0^{(i)}
    \sim
    \mathcal N(0,0.25^2).
\]
Each realization is scaled by
\[
    \alpha_i\sim\mathcal U(0.7,1.3),
\]
so that the final input is
\[
    \widetilde v^{(i)}(x)
    =
    \alpha_i v^{(i)}(x).
\]
For each input, the corresponding output is
\[
    y_h^{(i)}
    =
    \mathcal G_h(v_h^{(i)}).
\]
A total of \(2800\) aligned input--output pairs are generated and divided
into
\[
    N_{\mathrm{train}}=2000,
    \qquad
    N_{\mathrm{val}}=400,
    \qquad
    N_{\mathrm{test}}=400.
\]
The same random permutation is applied to both the input and output
arrays, and the resulting dataset is reused for every value of \(q\).
\subsection{Fixed and adaptive measurements}
For the fixed Topological DeepONet, the measurement map is
\[
    M_q:\mathbb R^m\longrightarrow\mathbb R^q,
    \qquad
    M_qv_h
    =
    \Phi_q^\top v_h,
\]
where
\[
    \Phi_q
    =
    \begin{bmatrix}
        \phi_{1,h} & \cdots & \phi_{q,h}
    \end{bmatrix}
    \in\mathbb R^{m\times q}
\]
contains the first \(q\) discrete Legendre modes after QR
orthonormalization:
\[
    \Phi_q^\top\Phi_q=I_q.
\]
The corresponding reconstruction is
\[
    R_q(M_qv_h)
    =
    \Phi_q\Phi_q^\top v_h.
\]
The adaptive measurement basis is initialized with \(\Phi_q\) and updated
according to
\[
    \widetilde\Phi_q
    =
    \operatorname{qf}
    \left(
        \Phi_q+\Delta\Phi_q
    \right),
\]
where \(\operatorname{qf}\) denotes the orthonormal factor from a reduced
QR decomposition. The correction is regularized using
\[
    \lambda_{\mathrm{drift}}
    \|\Delta\Phi_q\|_F^2,
    \qquad
    \lambda_{\mathrm{drift}}=10^{-4}.
\]
\subsection{Reduced output representation and training}
Let
\[
    \overline y_h
    =
    \frac{1}{N_{\mathrm{train}}}
    \sum_{i=1}^{N_{\mathrm{train}}}
    y_h^{(i)}
\]
be the training-output mean. Let
\[
    Q_r\in\mathbb R^{m\times r},
    \qquad
    Q_r^\top Q_r=I_r,
\]
contain the first \(r=64\) POD modes of the centered training-output
matrix. The exact reduced coefficients are
\[
    c_r^{(i)}
    =
    Q_r^\top
    \left(
        y_h^{(i)}-\overline y_h
    \right).
\]
The learned operator is
\[
    \widehat{\mathcal G}_{h,\theta}(v_h)
    =
    \overline y_h
    +
    Q_r b_\theta(M_qv_h),
\]
where
\[
    b_\theta:\mathbb R^q\longrightarrow\mathbb R^r
\]
is a feed-forward network with four hidden layers, width \(128\), and
GELU activations.
The fixed model is trained with AdamW using learning rate
\[
    2\times10^{-3}.
\]
The adaptive model is initialized from the trained fixed branch network
and optimized with learning rate
\[
    2\times10^{-4}.
\]
The batch size is \(64\), and the checkpoint with the smallest global
relative \(L^2\) validation error is retained.
\subsection{Measurement-dimension convergence protocol}
The measurement dimension is varied according to
\[
    q\in\{8,16,32,64,128\}.
\]
The spatial discretization, POD rank, network architecture, dataset, and
optimization parameters are held fixed. Therefore, the experiment is a
convergence study with respect to the measurement dimension \(q\), not a
spatial mesh-convergence study.
\subsection{Error decomposition}
For a test input \(v_h^{(i)}\), define the reconstructed input
\[
    \widetilde v_{h,q}^{(i)}
    =
    R_q(M_qv_h^{(i)})
\]
and its exact discrete output
\[
    \widetilde y_{h,q}^{(i)}
    =
    \mathcal G_h
    \left(
        \widetilde v_{h,q}^{(i)}
    \right).
\]
The propagated measurement error is
\[
    E_{\mathrm{meas}}^{(i)}
    =
    \frac{
        \left\|
            y_h^{(i)}
            -
            \widetilde y_{h,q}^{(i)}
        \right\|_2
    }{
        \|y_h^{(i)}\|_2
    }.
\]
The output projection error is
\[
    E_{\mathrm{out}}^{(i)}
    =
    \frac{
        \left\|
            \widetilde y_{h,q}^{(i)}
            -
            \overline y_h
            -
            Q_rQ_r^\top
            \left(
                \widetilde y_{h,q}^{(i)}
                -
                \overline y_h
            \right)
        \right\|_2
    }{
        \|y_h^{(i)}\|_2
    }.
\]
The neural approximation error is
\[
    E_{\mathrm{NN}}^{(i)}
    =
    \frac{
        \left\|
            \overline y_h
            +
            Q_rQ_r^\top
            \left(
                \widetilde y_{h,q}^{(i)}
                -
                \overline y_h
            \right)
            -
            \widehat{\mathcal G}_{h,\theta}
            \left(
                v_h^{(i)}
            \right)
        \right\|_2
    }{
        \|y_h^{(i)}\|_2
    }.
\]
The total prediction error is
\[
    E_{\mathrm{tot}}^{(i)}
    =
    \frac{
        \left\|
            y_h^{(i)}
            -
            \widehat{\mathcal G}_{h,\theta}
            \left(
                v_h^{(i)}
            \right)
        \right\|_2
    }{
        \|y_h^{(i)}\|_2
    }.
\]
The computable upper bound is
\[
    E_{\mathrm{bound}}^{(i)}
    =
    E_{\mathrm{meas}}^{(i)}
    +
    E_{\mathrm{out}}^{(i)}
    +
    E_{\mathrm{NN}}^{(i)}.
\]
The triangle inequality gives
\[
    E_{\mathrm{tot}}^{(i)}
    \leq
    E_{\mathrm{bound}}^{(i)}.
\]
The reported convergence curves are test-set means,
\[
    \overline E_\alpha(q)
    =
    \frac{1}{N_{\mathrm{test}}}
    \sum_{i=1}^{N_{\mathrm{test}}}
    E_\alpha^{(i)}(q),
\]
where
\[
    \alpha
    \in
    \{
        \mathrm{meas},
        \mathrm{out},
        \mathrm{NN},
        \mathrm{tot},
        \mathrm{bound}
    \}.
\]
The sharpness of the sample-wise estimate is measured by
\[
    \rho_{\max}
    =
    \max_i
    \frac{
        E_{\mathrm{tot}}^{(i)}
    }{
        E_{\mathrm{bound}}^{(i)}
    }.
\]
\section{Additional details for the Darcy experiment}
\label{app:darcy-sensor-interpretation}
\subsection{Discrete data representation}
Each permeability and pressure field is stored as a vector in
\(\mathbb{R}^{7225}\), obtained by flattening an \(85\times85\) grid.
The data matrices have dimensions
\begin{equation}
\begin{aligned}
A_{\mathrm{train}},U_{\mathrm{train}}
&\in\mathbb{R}^{800\times7225},\\
A_{\mathrm{val}},U_{\mathrm{val}}
&\in\mathbb{R}^{100\times7225},\\
A_{\mathrm{test}},U_{\mathrm{test}}
&\in\mathbb{R}^{100\times7225}.
\end{aligned}
\end{equation}
All normalization statistics are computed using only the training set.
\subsection{Reduced output basis}
Let
\begin{equation}
U_{\mathrm{train}}
=
W\Sigma V^{\mathsf T}
\end{equation}
be the singular-value decomposition of the training pressure snapshots.
The first \(r\) right singular vectors define
\begin{equation}
Q=V_{:,1:r}.
\end{equation}
Each solution is represented by
\begin{equation}
\bm{c}^{(i)}
=
Q^{\mathsf T}\bm{u}^{(i)},
\end{equation}
and reconstructed as
\begin{equation}
\widehat{\bm{u}}^{(i)}
=
Q\widehat{\bm{c}}^{(i)}.
\end{equation}
Thus, all Two-Step models predict the same reduced pressure coefficients and
use the same fixed output basis for reconstruction.
\subsection{Input representations}
The full-field Two-Step model uses the complete discretized permeability field,
\begin{equation}
\bm{a}\in\mathbb{R}^{7225}.
\end{equation}
The Sensor Two-Step model uses \(q=32\) pointwise observations,
\begin{equation}
\bm{s}(a)
=
\left[
a(\bm{x}_1),\ldots,a(\bm{x}_{32})
\right]^{\mathsf T}.
\end{equation}
The Fixed Topological DeepONet uses the same number of functional coordinates,
\begin{equation}
\bm{\ell}(a)
=
\left[
\ell_1(a),\ldots,\ell_{32}(a)
\right]^{\mathsf T},
\end{equation}
with
\begin{equation}
\ell_j(a)
=
\int_{\Omega}
a(\bm{x})\phi_j(\bm{x})\,\mathrm{d}\bm{x}.
\label{eq:app-darcy-functional}
\end{equation}
The continuum input space for the Darcy operator is
\(\mathcal V=L^\infty(\Omega)\), restricted to the uniformly elliptic
admissible set \(\mathcal A\) defined in~\eqref{eq:darcy-spaces}. Taking
\(\phi_j\in L^1(\Omega)\), each measurement is a continuous linear
functional on this ambient space because
\begin{equation}
\left|
\ell_j(a)
\right|
\leq
\|a\|_{L^\infty(\Omega)}
\|\phi_j\|_{L^1(\Omega)}.
\end{equation}
The fact that the sampled coefficient fields also lie in $L^2(\Omega)$ is
used only at the level of their finite-dimensional numerical representation;
it does not change the continuum topology assigned to the Darcy operator.
On the discrete grid, the functional is approximated by
\begin{equation}
\ell_j(a)
\approx
\sum_{m=1}^{7225}
w_m a(\bm{x}_m)\phi_j(\bm{x}_m),
\label{eq:app-darcy-discrete-functional}
\end{equation}
where \(w_m\) are quadrature weights associated with the grid points
\(\bm{x}_m\).
\subsection{Functional dictionaries}
The implementation supports Legendre, cosine, Chebyshev, and Lagrange
measurement dictionaries. In the default Darcy configuration, the fixed
measurements are constructed from a total-degree Legendre dictionary.
After mapping the physical coordinates to the reference square
\(\widehat{\Omega}=[-1,1]^2\), the Legendre atoms are defined by
\begin{equation}
\phi_{ij}^{\mathrm{leg}}(x,y)
=
\widehat{P}_i(x)\widehat{P}_j(y),
\qquad
i+j\leq p,
\label{eq:app-darcy-legendre}
\end{equation}
where
\begin{equation}
\widehat{P}_n(x)
=
\sqrt{\frac{2n+1}{2}}\,P_n(x)
\end{equation}
is the normalized Legendre polynomial of degree \(n\).
The alternative tensor-product dictionaries are
\begin{equation}
\begin{aligned}
\phi_{ij}^{\mathrm{cos}}(x,y)
&=
c_i(x)c_j(y),\\
\phi_{ij}^{\mathrm{cheb}}(x,y)
&=
T_i(x)T_j(y),\\
\phi_{ij}^{\mathrm{lag}}(x,y)
&=
L_i(x)L_j(y),
\end{aligned}
\label{eq:app-darcy-alternative-dictionaries}
\end{equation}
where
\begin{equation}
c_0(x)=1,
\qquad
c_k(x)=\sqrt{2}\cos(k\pi x),
\end{equation}
\(T_i\) denotes a first-kind Chebyshev polynomial, and \(L_i\) denotes a
Lagrange cardinal polynomial constructed on either Chebyshev or uniformly
spaced interpolation nodes.
For the polynomial dictionaries, the index set may be chosen either as a
total-degree set,
\begin{equation}
\mathcal{I}_p
=
\left\{
(i,j)\in\mathbb{N}_0^2:
i+j\leq p
\right\},
\end{equation}
or as a full tensor-product set. The total-degree construction is used in the
default experiment.
\subsection{Discrete weighted orthonormalization}
Let
\begin{equation}
\Phi_{mj}
=
\phi_j(\bm{x}_m)
\end{equation}
denote the discrete dictionary matrix, and let
\begin{equation}
W_q
=
\operatorname{diag}(w_1,\ldots,w_{7225})
\end{equation}
contain the quadrature weights. To improve conditioning and permit a fair
comparison among different basis families, the raw dictionary is
orthonormalized with respect to the weighted discrete inner product.
The weighted Gram matrix is
\begin{equation}
G
=
\Phi^{\mathsf T}W_q\Phi,
\end{equation}
and the orthonormalized dictionary is
\begin{equation}
\Phi_{\mathrm{orth}}
=
\Phi G^{-1/2}.
\label{eq:app-darcy-orthonormalization}
\end{equation}
The corresponding measurement matrix is
\begin{equation}
M_{\mathrm{base}}
=
W_q\Phi_{\mathrm{orth}}.
\end{equation}
For a discretized permeability vector \(\bm{a}\), the complete dictionary
coordinate vector is therefore
\begin{equation}
\bm{\alpha}(a)
=
M_{\mathrm{base}}^{\mathsf T}\bm{a}.
\label{eq:app-darcy-full-coordinates}
\end{equation}
\subsection{Active dictionary and fixed measurements}
The full dictionary may contain more functions than are needed by the branch
network. The dictionary coordinates are first ranked according to their
variance over the training set. An active subset is then retained so that it
captures a prescribed fraction of the total measurement variance, subject to
a minimum active dimension.
Let
\begin{equation}
\bm{\alpha}_{\mathrm{act}}(a)
\in
\mathbb{R}^{K_{\mathrm{act}}}
\end{equation}
denote the standardized coordinates associated with the active dictionary.
The Fixed Topological DeepONet uses the first \(q=32\) selected coordinates,
\begin{equation}
\bm{\ell}_{\mathrm{fix}}(a)
=
\left[
\alpha_1(a),\ldots,\alpha_{32}(a)
\right]^{\mathsf T}.
\end{equation}
These coordinates are fixed before neural-network training.
\subsection{Adaptive functional measurements}
The Adaptive Topological DeepONet begins from the active dictionary
coordinates and learns \(q=32\) linear combinations,
\begin{equation}
\bm{z}_{\mathrm{ad}}(a)
=
A_{\theta}
\bm{\alpha}_{\mathrm{act}}(a),
\qquad
A_{\theta}
\in
\mathbb{R}^{32\times K_{\mathrm{act}}}.
\label{eq:app-darcy-adaptive-map}
\end{equation}
The \(k\)-th learned coordinate can be written as
\begin{equation}
\begin{aligned}
z_k^{\mathrm{ad}}(a)
&=
\sum_{j=1}^{K_{\mathrm{act}}}
(A_{\theta})_{kj}\alpha_j(a)\\
&=
\int_{\Omega}
a(\bm{x})
\psi_k^{\theta}(\bm{x})
\,\mathrm{d}\bm{x},
\end{aligned}
\end{equation}
where
\begin{equation}
\psi_k^{\theta}(\bm{x})
=
\sum_{j=1}^{K_{\mathrm{act}}}
(A_{\theta})_{kj}\phi_j(\bm{x})
\in
\operatorname{span}
\left\{
\phi_1,\ldots,\phi_{K_{\mathrm{act}}}
\right\}.
\end{equation}
Thus, every adaptive coordinate remains a continuous linear functional in
\(\mathcal{V}'\), while its measurement function is learned within the span
of the active dictionary.
A training-only decoder reconstructs the active dictionary coordinates from
\(\bm{z}_{\mathrm{ad}}\). The associated reconstruction loss discourages the
learned map from discarding excessive input information. In addition, a Gram
regularization term promotes nondegenerate and approximately decorrelated
learned measurements.
\subsection{Two-Step branch prediction}
For the full-field, sensor, fixed-functional, and adaptive-functional
Two-Step models, the branch network predicts the standardized reduced
coefficients
\begin{equation}
\widehat{\bm{c}}(a)
\in
\mathbb{R}^{r}.
\end{equation}
The predicted pressure field is reconstructed using
\begin{equation}
\widehat{\bm{u}}(a)
=
Q\widehat{\bm{c}}(a).
\end{equation}
The corresponding input-to-output pipelines are
\begin{equation}
\bm{a}
\longrightarrow
\widehat{\bm{c}}
\longrightarrow
\widehat{\bm{u}}
\end{equation}
for the full-field Two-Step model,
\begin{equation}
\bm{s}(a)
\longrightarrow
\widehat{\bm{c}}
\longrightarrow
\widehat{\bm{u}}
\end{equation}
for the sensor model,
\begin{equation}
\bm{\ell}_{\mathrm{fix}}(a)
\longrightarrow
\widehat{\bm{c}}
\longrightarrow
\widehat{\bm{u}}
\end{equation}
for the fixed Topological DeepONet, and
\begin{equation}
\bm{\alpha}_{\mathrm{act}}(a)
\longrightarrow
A_\theta\bm{\alpha}_{\mathrm{act}}(a)
\longrightarrow
\widehat{\bm{c}}
\longrightarrow
\widehat{\bm{u}}
\end{equation}
for the adaptive model.
\subsection{Evaluation metrics}
For the \(i\)-th test sample, the relative \(L^2\) error is
\begin{equation}
e_i
=
\frac{
\left\|
\widehat{\bm{u}}^{(i)}-\bm{u}^{(i)}
\right\|_2
}{
\left\|
\bm{u}^{(i)}
\right\|_2
}.
\end{equation}
The mean relative \(L^2\) error is
\begin{equation}
E_{\mathrm{mean}}
=
\frac{1}{N_{\mathrm{test}}}
\sum_{i=1}^{N_{\mathrm{test}}}e_i.
\end{equation}
The global relative \(L^2\) error is
\begin{equation}
E_{\mathrm{global}}
=
\frac{
\left\|
U_{\mathrm{test}}-\widehat U_{\mathrm{test}}
\right\|_F
}{
\left\|
U_{\mathrm{test}}
\right\|_F
}.
\end{equation}
The root-mean-square error is
\begin{equation}
\mathrm{RMSE}
=
\sqrt{
\frac{1}{N_{\mathrm{test}}N_h}
\sum_{i=1}^{N_{\mathrm{test}}}
\sum_{m=1}^{N_h}
\left(
\widehat u_m^{(i)}-u_m^{(i)}
\right)^2
},
\end{equation}
where \(N_h=7225\) is the number of grid points. The mean absolute error is
\begin{equation}
\mathrm{MAE}
=
\frac{1}{N_{\mathrm{test}}N_h}
\sum_{i=1}^{N_{\mathrm{test}}}
\sum_{m=1}^{N_h}
\left|
\widehat u_m^{(i)}-u_m^{(i)}
\right|.
\end{equation}
\subsection{Interpretation of the sensor comparison}
Sensor Two-Step and Fixed Topological DeepONet are matched in input
dimension, branch architecture, output basis, and inference parameter count.
Their comparison therefore isolates the effect of replacing local point
evaluations with global functional coordinates.
The comparison should be interpreted as an equal
\emph{representation-budget} comparison, rather than necessarily an equal
physical sensing-cost comparison, because a global functional may require
access to the full discretized permeability field.
The adaptive model uses the same final coordinate dimension as the fixed and
sensor models, but its coordinates are learned as linear combinations of the
active functional dictionary. Consequently, its comparison with the fixed
model evaluates whether adapting the measurement functions within a common
dictionary span improves the reduced input representation.
\section{Additional Details for the Fixed-Time Navier--Stokes Benchmark}
\label{app:fixed-time-details}
\subsection{Governing equations}
\label{subsec:ns-governing-equations}
We consider the two-dimensional incompressible Navier--Stokes equations
on the periodic unit square
\[
    D=(0,1)^2,
\]
written in vorticity form as
\begin{equation}
    \frac{\partial \omega}{\partial t}
    +
    \bm{u}\cdot\nabla\omega
    =
    \nu\Delta\omega + f,
    \qquad
    (\bm{x},t)\in D\times(0,T],
    \label{eq:fixed-time-vorticity}
\end{equation}
subject to the incompressibility constraint
\begin{equation}
    \nabla\cdot\bm{u}=0.
    \label{eq:fixed-time-incompressibility}
\end{equation}
Here,
\(\omega=\omega(\bm{x},t)\) denotes the scalar vorticity,
\(\bm{u}=(u,v)\) is the velocity field,
\(\nu>0\) is the kinematic viscosity,
and \(f=f(\bm{x})\) is a prescribed forcing term.
The velocity and vorticity are related by
\begin{equation}
    \omega
    =
    \frac{\partial v}{\partial x}
    -
    \frac{\partial u}{\partial y}.
    \label{eq:fixed-time-vorticity-definition}
\end{equation}
Introducing the streamfunction \(\psi\), one may write
\begin{equation}
    \bm{u}
    =
    \nabla^\perp\psi
    =
    \begin{pmatrix}
        \partial_y\psi\\
        -\partial_x\psi
    \end{pmatrix},
    \qquad
    -\Delta\psi=\omega.
    \label{eq:fixed-time-streamfunction}
\end{equation}
Equivalently,
\begin{equation}
    \bm{u}
    =
    \nabla^\perp(-\Delta)^{-1}\omega.
    \label{eq:fixed-time-biot-savart}
\end{equation}
Periodic boundary conditions are imposed in both spatial directions:
\begin{equation}
    \omega(x+1,y,t)=\omega(x,y,t),
    \qquad
    \omega(x,y+1,t)=\omega(x,y,t),
    \label{eq:fixed-time-periodic-bc}
\end{equation}
together with the initial condition
\begin{equation}
    \omega(\bm{x},0)=\omega_0(\bm{x}).
    \label{eq:fixed-time-initial-condition}
\end{equation}
\subsection{Fixed-time operator-learning problem}
\label{subsec:fixed-time-operator-problem}
The objective is to learn a fixed-time solution operator rather than an
entire temporal trajectory. Let \(t_{\mathrm{in}}\) denote the input time
and let \(t_{\mathrm{out}}>t_{\mathrm{in}}\) denote the target time. The
exact solution operator is
\begin{equation}
    \mathcal{G}_{t_{\mathrm{in}}\rightarrow t_{\mathrm{out}}}:
    \omega(\cdot,t_{\mathrm{in}})
    \longmapsto
    \omega(\cdot,t_{\mathrm{out}}).
    \label{eq:fixed-time-solution-operator}
\end{equation}
In the present experiment,
\begin{equation}
    t_{\mathrm{in}}=t_0,
    \qquad
    t_{\mathrm{out}}=t_{10},
\end{equation}
and therefore
\begin{equation}
    \mathcal{G}_{0\rightarrow 10}:
    \omega(\cdot,t_0)
    \longmapsto
    \omega(\cdot,t_{10}).
    \label{eq:fixed-time-ten-step-map}
\end{equation}
Because the implementation uses zero-based snapshot indices,
\texttt{input-index=0} and \texttt{target-index=10} define a
ten-snapshot-interval forecast.
For the \(i\)-th realization, we define
\begin{equation}
    a^{(i)}(\bm{x})
    =
    \omega^{(i)}(\bm{x},t_0),
    \qquad
    u^{(i)}(\bm{x})
    =
    \omega^{(i)}(\bm{x},t_{10}).
    \label{eq:fixed-time-input-output}
\end{equation}
The supervised dataset is
\begin{equation}
    \mathcal{D}
    =
    \left\{
        \left(a^{(i)},u^{(i)}\right)
    \right\}_{i=1}^{N},
    \qquad
    u^{(i)}
    =
    \mathcal{G}_{0\rightarrow 10}
    \left(a^{(i)}\right).
    \label{eq:fixed-time-supervised-dataset}
\end{equation}
A neural operator \(\mathcal{G}_{\theta}\) is trained such that
\begin{equation}
    \widehat{u}^{(i)}
    =
    \mathcal{G}_{\theta}
    \left(a^{(i)}\right)
    \approx
    u^{(i)}.
    \label{eq:fixed-time-learned-operator}
\end{equation}
\subsection{Dataset and preprocessing}
\label{subsec:fixed-time-data}
We use the standard two-dimensional Navier--Stokes vorticity dataset
employed in neural-operator studies~\cite{li2021fourier}. Each realization is
stored on a uniform periodic grid with
\[
    N_x=N_y=64
\]
and contains 50 temporal snapshots. The complete data tensor may be written as
\begin{equation}
    \mathbf{W}
    \in
    \mathbb{R}^{5000\times64\times64\times50},
    \label{eq:fixed-time-data-tensor}
\end{equation}
up to the axis ordering used by the MATLAB or HDF5 data file.
For the fixed-time experiment, only the input and target snapshots are
extracted:
\begin{align}
    \mathbf{X}^{(i)}
    &=
    \mathbf{W}^{(i)}(:,:,0),
    \\
    \mathbf{Y}^{(i)}
    &=
    \mathbf{W}^{(i)}(:,:,10).
\end{align}
Thus,
\begin{equation}
    \mathbf{X},\mathbf{Y}
    \in
    \mathbb{R}^{N\times64\times64}.
\end{equation}
After vectorization,
\begin{equation}
    \mathbf{x}^{(i)}
    =
    \operatorname{vec}
    \left(
        \mathbf{X}^{(i)}
    \right)
    \in
    \mathbb{R}^{4096},
    \qquad
    \mathbf{y}^{(i)}
    =
    \operatorname{vec}
    \left(
        \mathbf{Y}^{(i)}
    \right)
    \in
    \mathbb{R}^{4096}.
    \label{eq:fixed-time-vectorized-data}
\end{equation}
The 5000 realizations are partitioned into
\begin{equation}
    N_{\mathrm{train}}=4000,
    \qquad
    N_{\mathrm{val}}=500,
    \qquad
    N_{\mathrm{test}}=500.
    \label{eq:fixed-time-split}
\end{equation}
The viscosity is
\[
    \nu=10^{-3}.
\]
\begin{table}[t]
    \centering
    \caption{Configuration of the fixed-time Navier--Stokes benchmark.}
    \label{tab:fixed-time-navier-stokes}
    \begin{tabular}{ll}
        \toprule
        Quantity & Value \\
        \midrule
        Governing equation
        & 2D forced incompressible Navier--Stokes \\
        State variable
        & Vorticity \(\omega\) \\
        Spatial domain
        & \(D=(0,1)^2\), periodic \\
        Viscosity
        & \(\nu=10^{-3}\) \\
        Number of realizations
        & \(5000\) \\
        Spatial resolution
        & \(64\times64\) \\
        Stored time snapshots
        & \(50\) \\
        Input snapshot index
        & \(0\) \\
        Target snapshot index
        & \(10\) \\
        Forecast horizon
        & \(10\) snapshot intervals \\
        Training/validation/test split
        & \(4000/500/500\) \\
        Full-field input dimension
        & \(4096\) \\
        Output dimension
        & \(4096\) \\
        Number of compressed measurements
        & \(q=128\) \\
        Compression ratio
        & \(4096/128=32\) \\
        \bottomrule
    \end{tabular}
\end{table}
\subsection{Reduced output representation}
\label{subsec:fixed-time-output-svd}
For the Two-Step models, the target fields are represented in a
low-dimensional, data-adaptive basis. Let
\begin{equation}
    \overline{\mathbf{y}}
    =
    \frac{1}{N_{\mathrm{train}}}
    \sum_{i=1}^{N_{\mathrm{train}}}
    \mathbf{y}^{(i)}
\end{equation}
denote the mean target field, and define the centered output matrix
\begin{equation}
    \mathbf{Y}_{c}
    =
    \begin{bmatrix}
        (\mathbf{y}^{(1)}-\overline{\mathbf{y}})^{\top}\\
        \vdots\\
        (\mathbf{y}^{(N_{\mathrm{train}})}
        -\overline{\mathbf{y}})^{\top}
    \end{bmatrix}.
\end{equation}
The singular value decomposition is
\begin{equation}
    \mathbf{Y}_{c}
    =
    \mathbf{U}\mathbf{\Sigma}\mathbf{V}^{\top}.
    \label{eq:fixed-time-output-svd}
\end{equation}
The first \(r\) right singular vectors define
\begin{equation}
    \mathbf{Q}
    =
    \begin{bmatrix}
        \mathbf{v}_{1} & \cdots & \mathbf{v}_{r}
    \end{bmatrix}
    \in
    \mathbb{R}^{4096\times r}.
    \label{eq:fixed-time-output-basis}
\end{equation}
The rank is chosen according to
\begin{equation}
    \frac{
        \sum_{j=1}^{r}\sigma_j^2
    }{
        \sum_j\sigma_j^2
    }
    \geq
    \eta_{\mathrm{SVD}},
    \label{eq:fixed-time-svd-energy}
\end{equation}
where \(\eta_{\mathrm{SVD}}\in(0,1)\) is a prescribed energy
threshold, subject to a prescribed maximum rank.
Each target field is approximated as
\begin{equation}
    \mathbf{y}^{(i)}
    \approx
    \overline{\mathbf{y}}
    +
    \mathbf{Q}\mathbf{c}^{(i)},
    \label{eq:fixed-time-output-reconstruction}
\end{equation}
where
\begin{equation}
    \mathbf{c}^{(i)}
    =
    \mathbf{Q}^{\top}
    \left(
        \mathbf{y}^{(i)}
        -
        \overline{\mathbf{y}}
    \right).
    \label{eq:fixed-time-output-coefficients}
\end{equation}
Hence, the Two-Step models learn the reduced map
\begin{equation}
    \widehat{\mathbf{c}}^{(i)}
    =
    \mathcal{B}_{\theta}
    \left(
        \mathcal{M}
        \left(
            a^{(i)}
        \right)
    \right),
    \label{eq:fixed-time-reduced-map}
\end{equation}
where \(\mathcal{M}\) denotes the input representation associated with
the particular model.
\subsection{Benchmark models}
\label{subsec:fixed-time-models}
\paragraph{Full-field Two-Step.}
The Full-field Two-Step model receives the complete input field:
\begin{equation}
    \mathcal{M}_{\mathrm{full}}(a)
    =
    a
    \in
    \mathbb{R}^{64\times64}.
\end{equation}
The branch predicts
\begin{equation}
    \widehat{\mathbf{c}}
    =
    \mathcal{B}^{\mathrm{full}}_{\theta}(a),
\end{equation}
and the output is reconstructed through
\begin{equation}
    \widehat{\mathbf{y}}
    =
    \overline{\mathbf{y}}
    +
    \mathbf{Q}\widehat{\mathbf{c}}.
\end{equation}
\paragraph{Sensor Two-Step.}
Let
\begin{equation}
    \mathcal{S}
    =
    \left\{
        \bm{x}_{s_1},\ldots,\bm{x}_{s_q}
    \right\}
\end{equation}
be a prescribed set of \(q=128\) sensor locations. The measurements are
\begin{equation}
    z_j^{\mathrm{sens}}
    =
    a(\bm{x}_{s_j}),
    \qquad
    j=1,\ldots,q.
    \label{eq:fixed-time-sensor-measurements}
\end{equation}
Thus,
\begin{equation}
    \mathcal{M}_{\mathrm{sens}}(a)
    =
    \begin{bmatrix}
        a(\bm{x}_{s_1}) &
        \cdots &
        a(\bm{x}_{s_q})
    \end{bmatrix}^{\top}
    \in
    \mathbb{R}^{q}.
\end{equation}
The corresponding compression ratio is
\begin{equation}
    \mathrm{CR}
    =
    \frac{4096}{128}
    =
    32.
\end{equation}
When a Fourier spectral-layer branch is employed, the point values are
placed on a sparse
grid,
\begin{equation}
    S(\bm{x})
    =
    \begin{cases}
        a(\bm{x}),
        & \bm{x}\in\mathcal{S},
        \\
        0,
        & \text{otherwise},
    \end{cases}
\end{equation}
and are accompanied by the observation mask
\begin{equation}
    M(\bm{x})
    =
    \begin{cases}
        1,
        & \bm{x}\in\mathcal{S},
        \\
        0,
        & \text{otherwise}.
    \end{cases}
\end{equation}
\paragraph{Fixed Topological Two-Step.}
Let
\begin{equation}
    \Psi
    =
    \left\{
        \psi_k
    \right\}_{k=1}^{K}
\end{equation}
denote a dictionary of localized multiscale test functions. For each
input field \(a\), define
\begin{equation}
    h_k(a)
    =
    \left\langle
        a,\psi_k
    \right\rangle
    =
    \int_D
    a(\bm{x})
    \psi_k(\bm{x})
    \,d\bm{x}.
    \label{eq:fixed-time-dictionary-functional}
\end{equation}
A subset of \(q=128\) predictive atoms is selected,
\begin{equation}
    \mathcal{I}
    =
    \{i_1,\ldots,i_q\},
\end{equation}
and the fixed measurements are
\begin{equation}
    z_j^{\mathrm{fix}}
    =
    \left\langle
        a,\psi_{i_j}
    \right\rangle,
    \qquad
    j=1,\ldots,q.
    \label{eq:fixed-time-fixed-functionals}
\end{equation}
The fixed measurement operator is therefore
\begin{equation}
    \mathcal{M}_{\mathrm{fix}}(a)
    =
    \begin{bmatrix}
        \langle a,\psi_{i_1}\rangle\\
        \vdots\\
        \langle a,\psi_{i_q}\rangle
    \end{bmatrix}
    \in
    \mathbb{R}^{q}.
    \label{eq:fixed-time-fixed-measurement-map}
\end{equation}
For a fair comparison, these measurements are synthesized into
a spatial proxy,
\begin{equation}
    \widetilde{a}_{\mathrm{fix}}(\bm{x})
    =
    \sum_{j=1}^{q}
    z_j^{\mathrm{fix}}
    \widetilde{\psi}_{i_j}(\bm{x}),
    \label{eq:fixed-time-fixed-proxy}
\end{equation}
where \(\widetilde{\psi}_{i_j}\) denotes the associated synthesis atom.
\paragraph{Adaptive Topological Two-Step.}
Let
\begin{equation}
    \bm{h}(a)
    =
    \begin{bmatrix}
        h_1(a) & \cdots & h_K(a)
    \end{bmatrix}
\end{equation}
contain all dictionary measurements, and let
\begin{equation}
    \mathbf{A}_0
    \in
    \mathbb{R}^{K\times q}
\end{equation}
denote the fixed selection matrix. The adaptive measurement matrix is
\begin{equation}
    \mathbf{A}_{\theta}
    =
    \mathbf{A}_{0}
    +
    \Delta\mathbf{A}_{\theta}.
    \label{eq:fixed-time-adaptive-matrix}
\end{equation}
The adaptive measurements are
\begin{equation}
    \bm{z}^{\mathrm{ad}}
    =
    \bm{h}(a)\mathbf{A}_{\theta}
    \in
    \mathbb{R}^{q}.
    \label{eq:fixed-time-adaptive-measurements}
\end{equation}
At initialization,
\begin{equation}
    \Delta\mathbf{A}_{\theta}=0,
\end{equation}
so the adaptive model coincides exactly with the fixed model at epoch zero.
Equivalently, the learned functionals are
\begin{equation}
    \phi_j^{\theta}
    =
    \sum_{k=1}^{K}
    (\mathbf{A}_{\theta})_{kj}\psi_k,
\end{equation}
with
\begin{equation}
    z_j^{\mathrm{ad}}
    =
    \langle a,\phi_j^{\theta}\rangle.
\end{equation}
A drift penalty controls departure from the fixed measurement system:
\begin{equation}
    \mathcal{L}_{\mathrm{drift}}
    =
    \lambda_{\mathrm{drift}}
    \left\|
        \Delta\mathbf{A}_{\theta}
    \right\|_F^2.
    \label{eq:fixed-time-drift-loss}
\end{equation}
\paragraph{Vanilla DeepONet.}
The Vanilla DeepONet~\cite{lu2021deeponet} uses a branch network to encode
the input field and a trunk network to encode the output coordinate:
\begin{equation}
    \bm{b}_{\theta}(a)
    =
    \begin{bmatrix}
        b_1(a),\ldots,b_p(a)
    \end{bmatrix},
\end{equation}
and
\begin{equation}
    \bm{t}_{\eta}(\bm{x})
    =
    \begin{bmatrix}
        t_1(\bm{x}),\ldots,t_p(\bm{x})
    \end{bmatrix}.
\end{equation}
The predicted target field is
\begin{equation}
    \widehat{u}(\bm{x};a)
    =
    \sum_{\ell=1}^{p}
    b_{\ell}(a)t_{\ell}(\bm{x})
    +
    b_0.
    \label{eq:fixed-time-vanilla-deeponet}
\end{equation}
In the common-backbone comparison, the branch is implemented using a
2D Fourier spectral-layer encoder, while the trunk is a coordinate
multilayer perceptron.
\subsection{2D Fourier spectral-layer branch architecture}
\label{subsec:fixed-time-fno}
For models using a Fourier spectral-layer branch~\cite{li2021fourier},
one spectral layer is
written schematically as
\begin{equation}
    v_{\ell+1}(\bm{x})
    =
    \sigma
    \left[
        W_{\ell}v_{\ell}(\bm{x})
        +
        \mathcal{F}^{-1}
        \left(
            R_{\ell}(\bm{k})
            \mathcal{F}[v_{\ell}](\bm{k})
        \right)(\bm{x})
    \right],
    \label{eq:fixed-time-fno-layer}
\end{equation}
where
\(\mathcal{F}\) denotes the discrete Fourier transform,
\(R_{\ell}\) is a learned complex spectral multiplier,
\(W_{\ell}\) is a local linear transformation,
and \(\sigma\) is a nonlinear activation.
For the present fixed-time problem, the branch acts on the two spatial
dimensions only.
\subsection{Training objective}
\label{subsec:fixed-time-training-loss}
Let \(\widehat{\mathbf{c}}^{(i)}\) and \(\mathbf{c}^{(i)}\) denote the
predicted and exact output coefficients. The coefficient-space mean
squared error is
\begin{equation}
    \mathcal{L}_{\mathrm{coef}}
    =
    \frac{1}{Br}
    \sum_{i=1}^{B}
    \left\|
        \widehat{\mathbf{c}}^{(i)}
        -
        \mathbf{c}^{(i)}
    \right\|_2^2,
    \label{eq:fixed-time-coefficient-loss}
\end{equation}
and the relative coefficient loss is
\begin{equation}
    \mathcal{L}_{\mathrm{rel}}
    =
    \frac{1}{B}
    \sum_{i=1}^{B}
    \frac{
        \left\|
            \widehat{\mathbf{c}}^{(i)}
            -
            \mathbf{c}^{(i)}
        \right\|_2
    }{
        \left\|
            \mathbf{c}^{(i)}
        \right\|_2
        +
        \varepsilon
    }.
    \label{eq:fixed-time-relative-coefficient-loss}
\end{equation}
The Two-Step objective is
\begin{equation}
    \mathcal{L}_{\mathrm{TwoStep}}
    =
    \mathcal{L}_{\mathrm{coef}}
    +
    \lambda_{\mathrm{rel}}
    \mathcal{L}_{\mathrm{rel}}.
    \label{eq:fixed-time-twostep-loss}
\end{equation}
For the adaptive model,
\begin{equation}
    \mathcal{L}_{\mathrm{Adaptive}}
    =
    \mathcal{L}_{\mathrm{TwoStep}}
    +
    \lambda_{\mathrm{drift}}
    \left\|
        \Delta\mathbf{A}_{\theta}
    \right\|_F^2.
    \label{eq:fixed-time-adaptive-loss}
\end{equation}
\subsection{Evaluation metrics}
\label{subsec:fixed-time-metrics}
The principal metric is the global relative \(L^2\) error:
\begin{equation}
    \varepsilon_{\mathrm{global}}
    =
    \frac{
        \left\|
            \widehat{\mathbf{Y}}
            -
            \mathbf{Y}
        \right\|_F
    }{
        \left\|
            \mathbf{Y}
        \right\|_F
    }.
    \label{eq:fixed-time-global-error}
\end{equation}
For the \(i\)-th test sample, the relative \(L^2\) error is
\begin{equation}
    \varepsilon_i
    =
    \frac{
        \left\|
            \widehat{\mathbf{y}}^{(i)}
            -
            \mathbf{y}^{(i)}
        \right\|_2
    }{
        \left\|
            \mathbf{y}^{(i)}
        \right\|_2
    }.
    \label{eq:fixed-time-sample-error}
\end{equation}
The mean sample-wise relative error is
\begin{equation}
    \overline{\varepsilon}
    =
    \frac{1}{N_{\mathrm{test}}}
    \sum_{i=1}^{N_{\mathrm{test}}}
    \varepsilon_i.
    \label{eq:fixed-time-mean-error}
\end{equation}
The comparison is designed to determine whether distributed linear
functionals preserve more predictive information than point sensors at
the same feature dimension \(q=128\), whether adaptive functionals
improve on their fixed counterparts, and how closely compressed-input
models approach models that observe the complete \(64\times64\) field.
\section{Additional details for the time-evolving Navier--Stokes experiment}
\label{app:time-evolving-ns-details}
\subsection{Discrete trajectory representation}
Each vorticity snapshot is stored on a \(64\times64\) periodic grid and is
therefore represented by a vector in
\[
    \mathbb{R}^{4096}.
\]
Each realization contains \(50\) stored snapshots. The first
\(T_{\mathrm{in}}=10\) snapshots define the input history,
\begin{equation}
    \bm{x}^{(i)}
    =
    \left[
        \omega^{(i)}(\cdot,t_0),
        \ldots,
        \omega^{(i)}(\cdot,t_9)
    \right],
\end{equation}
while the remaining \(T_{\mathrm{out}}=40\) snapshots define the target
trajectory,
\begin{equation}
    \bm{y}^{(i)}
    =
    \left[
        \omega^{(i)}(\cdot,t_{10}),
        \ldots,
        \omega^{(i)}(\cdot,t_{49})
    \right].
\end{equation}
Thus,
\begin{equation}
    \bm{x}^{(i)}
    \in
    \mathbb{R}^{64\times64\times10},
    \qquad
    \bm{y}^{(i)}
    \in
    \mathbb{R}^{64\times64\times40}.
\end{equation}
The realizations are divided into
\begin{equation}
    N_{\mathrm{train}}=4000,
    \qquad
    N_{\mathrm{val}}=500,
    \qquad
    N_{\mathrm{test}}=500.
\end{equation}
The implementation constructs the input and output windows directly from the
first ten and subsequent forty snapshots, respectively.
\subsection{Input and output function spaces}
Let
\begin{equation}
    \mathcal{V}
    =
    L_{\mathrm{per}}^2(D),
    \qquad
    D=(0,1)^2,
\end{equation}
denote the space of square-integrable periodic vorticity fields. The
time-evolving operator is
\begin{equation}
    \mathcal{G}:
    \mathcal{V}^{10}
    \longrightarrow
    \mathcal{V}^{40},
\end{equation}
with
\begin{equation}
    \mathcal{G}
    \left(
        \omega(\cdot,t_0),
        \ldots,
        \omega(\cdot,t_9)
    \right)
    =
    \left(
        \omega(\cdot,t_{10}),
        \ldots,
        \omega(\cdot,t_{49})
    \right).
\end{equation}
Equivalently, using discrete-time sequence spaces,
\begin{equation}
    \mathcal{G}:
    \ell^2
    \left(
        \{t_0,\ldots,t_9\};
        L_{\mathrm{per}}^2(D)
    \right)
    \longrightarrow
    \ell^2
    \left(
        \{t_{10},\ldots,t_{49}\};
        L_{\mathrm{per}}^2(D)
    \right).
\end{equation}
\subsection{Training-data standardization}
The input histories are standardized using the scalar mean and standard
deviation computed from the training inputs:
\begin{equation}
    \omega_{\mathrm{std}}
    =
    \frac{
        \omega-\mu_{\mathrm{in}}
    }{
        \sigma_{\mathrm{in}}
    }.
\end{equation}
The same \(\mu_{\mathrm{in}}\) and \(\sigma_{\mathrm{in}}\) are used for
training, validation, and testing. No validation or test statistics are used
during normalization. The same scalar standardization procedure is applied
to the synthesized topological proxy fields.
\subsection{Multiscale functional dictionary}
For each input snapshot, the topological representations are constructed
from a dictionary
\begin{equation}
    \mathcal{D}_K
    =
    \left\{
        \psi_1,\ldots,\psi_K
    \right\}
    \subset
    L_{\mathrm{per}}^2(D).
\end{equation}
The default dictionary size is
\begin{equation}
    K
    =
    \gamma q,
    \qquad
    \gamma=4,
    \qquad
    q=128,
\end{equation}
so that \(K=512\).
Let
\[
    \bm{c}=(c_x,c_y)
\]
denote a dictionary center and let
\[
    d_x(x,c_x)
    =
    \min
    \left(
        |x-c_x|,
        1-|x-c_x|
    \right),
\]
\[
    d_y(y,c_y)
    =
    \min
    \left(
        |y-c_y|,
        1-|y-c_y|
    \right)
\]
denote periodic coordinate distances. Define
\begin{equation}
    r_{\mathrm{per}}^2
    =
    d_x^2+d_y^2.
\end{equation}
For each grid-relative scale
\begin{equation}
    s
    \in
    \mathcal{S}
    =
    \{1.5,3,6,12\},
\end{equation}
the corresponding normalized spatial width is
\begin{equation}
    \sigma_s
    =
    \max
    \left(
        \frac{s}{\max(N_x,N_y)},
        \frac{1}{\max(N_x,N_y)}
    \right).
    \label{eq:app-ns-scale}
\end{equation}
For the \(64\times64\) grid,
\begin{equation}
    \sigma_s
    \in
    \left\{
        \frac{1.5}{64},
        \frac{3}{64},
        \frac{6}{64},
        \frac{12}{64}
    \right\}.
\end{equation}
The Gaussian envelope is
\begin{equation}
    g_{\bm{c},\sigma}(\bm{x})
    =
    \exp
    \left(
        -\frac{
            r_{\mathrm{per}}^2
        }{
            2\sigma^2
        }
    \right).
\end{equation}
The dictionary contains the following atom families:
\begin{equation}
\begin{aligned}
    \psi_{\bm{c},\sigma}^{(0)}(\bm{x})
    &=
    g_{\bm{c},\sigma}(\bm{x}),
    \\
    \psi_{\bm{c},\sigma}^{(x)}(\bm{x})
    &=
    \frac{
        d_x(x,c_x)
    }{
        \sigma
    }
    g_{\bm{c},\sigma}(\bm{x}),
    \\
    \psi_{\bm{c},\sigma}^{(y)}(\bm{x})
    &=
    \frac{
        d_y(y,c_y)
    }{
        \sigma
    }
    g_{\bm{c},\sigma}(\bm{x}),
    \\
    \psi_{\bm{c},\sigma}^{(\Delta)}(\bm{x})
    &=
    \left(
        \frac{
            r_{\mathrm{per}}^2
        }{
            \sigma^2
        }
        -2
    \right)
    g_{\bm{c},\sigma}(\bm{x}).
\end{aligned}
\label{eq:app-ns-dictionary-families}
\end{equation}
A constant atom is also included.
The implementation labels the second and third families as
\texttt{dx} and \texttt{dy}. Since unsigned periodic distances are used,
these should be interpreted as directionally weighted Gaussian atoms rather
than exact signed Gaussian derivatives. The dictionary construction and
scale conversion are implemented directly in the code.
\subsection{Atom normalization}
Each nonconstant atom is centered by subtracting its discrete mean:
\begin{equation}
    \psi_j
    \leftarrow
    \psi_j
    -
    \frac{1}{N_h}
    \sum_{m=1}^{N_h}
    \psi_j(\bm{x}_m),
\end{equation}
where \(N_h=4096\). It is then normalized using the discrete Euclidean norm:
\begin{equation}
    \psi_j
    \leftarrow
    \frac{
        \psi_j
    }{
        \left(
            \sum_{m=1}^{N_h}
            \psi_j(\bm{x}_m)^2
        \right)^{1/2}
    }.
\end{equation}
The assembled dictionary matrix is normalized columnwise once more for
numerical robustness.
\subsection{Functional measurements}
Let
\begin{equation}
    \Psi
    =
    \begin{bmatrix}
        \bm{\psi}_1 & \cdots & \bm{\psi}_K
    \end{bmatrix}
    \in
    \mathbb{R}^{4096\times K}
\end{equation}
denote the discrete dictionary matrix (denoted \(\Psi\) here to
avoid a clash with the mini-batch size \(B\) used in the training losses).
For the \(n\)-th input snapshot,
the complete dictionary-coordinate vector is
\begin{equation}
    \bm{h}(t_n)
    =
    \Psi^{\mathsf T}
    \bm{\omega}(t_n)
    \in
    \mathbb{R}^{K}.
\end{equation}
Equivalently,
\begin{equation}
    h_j(t_n)
    =
    \left\langle
        \omega(\cdot,t_n),
        \psi_j
    \right\rangle_h,
\end{equation}
where \(\langle\cdot,\cdot\rangle_h\) denotes the discrete grid inner
product used by the implementation.
For a batch of \(N\) trajectories, the full measurement tensor has shape
\begin{equation}
    H
    \in
    \mathbb{R}^{N\times T_{\mathrm{in}}\times K}.
\end{equation}
The implementation computes these measurements by multiplying each flattened
input snapshot by the dictionary matrix.
\subsection{Predictive selection of fixed measurements}
The fixed topological model selects \(q=128\) atoms from the complete
dictionary. Each dictionary coordinate is standardized over the training
set and scored according to its cross-covariance with the reduced output
coefficients.
Let
\begin{equation}
    \widetilde H
    \in
    \mathbb{R}^{N_{\mathrm{train}}
    \times T_{\mathrm{in}}\times K}
\end{equation}
denote the standardized measurement tensor and let
\begin{equation}
    C
    \in
    \mathbb{R}^{N_{\mathrm{train}}\times r}
\end{equation}
contain the centered output coefficients. For atom \(j\) and input time
\(t_n\), define
\begin{equation}
    \bm{\gamma}_{n,j}
    =
    \frac{1}{N_{\mathrm{train}}}
    \widetilde H_{:,n,j}^{\mathsf T}C.
\end{equation}
The time-aggregated predictive score is
\begin{equation}
    s_j
    =
    \left[
        \frac{1}{T_{\mathrm{in}}}
        \sum_{n=0}^{T_{\mathrm{in}}-1}
        \left\|
            \bm{\gamma}_{n,j}
        \right\|_2^2
    \right]^{1/2}.
    \label{eq:app-ns-predictive-score}
\end{equation}
The \(q\) atoms with the largest scores define the fixed measurement index
set
\begin{equation}
    \mathcal{I}_q.
\end{equation}
This selection is performed using training data only.
\subsection{Dual synthesis and proxy reconstruction}
Let
\begin{equation}
    \Psi_q
    =
    \Psi_{:,\mathcal{I}_q}
    \in
    \mathbb{R}^{4096\times q}
\end{equation}
contain the selected atoms. The regularized dual synthesis matrix is
\begin{equation}
    D_q
    =
    \Psi_q
    \left(
        \Psi_q^{\mathsf T}\Psi_q
        +
        \varepsilon_{\mathrm{dual}}I
    \right)^{-1},
    \qquad
    \varepsilon_{\mathrm{dual}}=10^{-5}.
    \label{eq:app-ns-dual-synthesis}
\end{equation}
For the selected measurements
\begin{equation}
    \bm{z}(t_n)
    =
    \bm{h}_{\mathcal{I}_q}(t_n)
    \in
    \mathbb{R}^{q},
\end{equation}
the synthesized proxy snapshot is
\begin{equation}
    \widetilde{\bm{\omega}}_q(t_n)
    =
    D_q\bm{z}(t_n).
\end{equation}
The complete proxy history belongs to
\begin{equation}
    \mathbb{R}^{64\times64\times10}.
\end{equation}
This proxy history is standardized using training-set proxy statistics and
then passed to the common three-dimensional spectral encoder.
\subsection{Sensor representation}
The sensor model uses \(q=128\) approximately uniformly distributed spatial
locations,
\begin{equation}
    \left\{
        \bm{x}_{s_1},
        \ldots,
        \bm{x}_{s_q}
    \right\}.
\end{equation}
For every input time,
\begin{equation}
    z_j^{\mathrm{sens}}(t_n)
    =
    \omega(\bm{x}_{s_j},t_n).
\end{equation}
The measured values are embedded into a sparse field
\begin{equation}
    \omega_{\mathrm{sparse}}
    \in
    \mathbb{R}^{64\times64\times10},
\end{equation}
and a binary mask
\begin{equation}
    m_{\mathrm{sens}}
    \in
    \mathbb{R}^{64\times64\times10}
\end{equation}
indicates the sensor locations. The FNO branch therefore receives the
two-channel tensor
\begin{equation}
    \left[
        \omega_{\mathrm{sparse}},
        m_{\mathrm{sens}}
    \right]
    \in
    \mathbb{R}^{64\times64\times10\times2}.
\end{equation}
This representation is constructed explicitly in the sensor-proxy routine.
\subsection{Adaptive topological measurements}
The adaptive model starts from the fixed selected coordinates and adds a
trainable correction formed from all dictionary coordinates:
\begin{equation}
    \bm{z}_{\mathrm{ad}}(t_n)
    =
    \bm{h}_{\mathcal{I}_q}(t_n)
    +
    \bm{h}(t_n)\Delta A_\theta,
\end{equation}
where
\begin{equation}
    \Delta A_\theta
    \in
    \mathbb{R}^{K\times q}.
\end{equation}
The correction matrix is initialized as
\begin{equation}
    \Delta A_\theta=0,
\end{equation}
so that the adaptive and fixed coordinates coincide at initialization.
The adaptive proxy is reconstructed using the same fixed dual matrix:
\begin{equation}
    \widetilde{\bm{\omega}}_{\mathrm{ad}}(t_n)
    =
    D_q
    \bm{z}_{\mathrm{ad}}(t_n).
\end{equation}
Thus, the adaptive model changes the measurement coordinates while retaining
the same synthesis space and spectral-encoder branch architecture.
\subsection{Common three-dimensional Fourier spectral encoder}
All four Two-Step models use the same three-dimensional FNO encoder acting
over the two spatial dimensions and the input-time dimension. The encoder
receives a tensor of the form
\begin{equation}
    X
    \in
    \mathbb{R}^{N_x\times N_y\times T_{\mathrm{in}}\times C_{\mathrm{in}}},
\end{equation}
where \(C_{\mathrm{in}}=1\) for the full-field and topological models and
\(C_{\mathrm{in}}=2\) for the sensor model.
Normalized coordinate channels
\begin{equation}
    (x,y,t)
\end{equation}
are concatenated to the input. Each spectral layer applies a truncated
three-dimensional Fourier convolution and a local pointwise convolution:
\begin{equation}
    v_{\ell+1}
    =
    \sigma
    \left[
        \mathcal{F}^{-1}
        \left(
            R_\ell
            \odot
            \mathcal{F}(v_\ell)
        \right)
        +
        W_\ell v_\ell
    \right].
\end{equation}
The default spectral truncation is
\begin{equation}
    (k_x,k_y,k_t)
    =
    (12,12,6),
\end{equation}
with width \(32\) and four spectral layers. After the last layer, global
averaging over the spatial and input-time coordinates produces the branch
latent vector.
\subsection{Reduced output trajectory basis}
Let
\begin{equation}
    Y_{\mathrm{train}}
    \in
    \mathbb{R}^{4000\times(4096\cdot40)}
\end{equation}
contain the flattened training output trajectories. The mean trajectory is
\begin{equation}
    \overline{\bm{y}}
    =
    \frac{1}{N_{\mathrm{train}}}
    \sum_{i=1}^{N_{\mathrm{train}}}
    \bm{y}^{(i)},
\end{equation}
and the centered data matrix is
\begin{equation}
    Y_c
    =
    Y_{\mathrm{train}}
    -
    \bm{1}
    \overline{\bm{y}}^{\mathsf T}.
\end{equation}
The economy singular-value decomposition is
\begin{equation}
    Y_c
    =
    U\Sigma V^{\mathsf T}.
\end{equation}
The retained rank is selected so that
\begin{equation}
    \frac{
        \sum_{j=1}^{r}\sigma_j^2
    }{
        \sum_{j}\sigma_j^2
    }
    \geq
    \eta_{\mathrm{SVD}},
\end{equation}
subject to
\begin{equation}
    r\leq r_{\max}.
\end{equation}
In the default configuration,
\begin{equation}
    \eta_{\mathrm{SVD}}=0.999,
    \qquad
    r_{\max}=128.
\end{equation}
The output basis is
\begin{equation}
    Q
    =
    V_{:,1:r}
    \in
    \mathbb{R}^{(4096\cdot40)\times r}.
\end{equation}
The coefficient vector is
\begin{equation}
    \bm{c}^{(i)}
    =
    Q^{\mathsf T}
    \left(
        \bm{y}^{(i)}
        -
        \overline{\bm{y}}
    \right),
\end{equation}
and the reconstructed trajectory is
\begin{equation}
    \widehat{\bm{y}}^{(i)}
    =
    \overline{\bm{y}}
    +
    Q\widehat{\bm{c}}^{(i)}.
\end{equation}
The SVD basis is constructed from the training outputs only.
\subsection{Coefficient standardization and Two-Step loss}
Each retained coefficient is standardized independently using its training
mean and standard deviation:
\begin{equation}
    c_{j,\mathrm{std}}
    =
    \frac{
        c_j-\mu_{c_j}
    }{
        \sigma_{c_j}
    }.
\end{equation}
The branch network predicts standardized coefficients. Before evaluating the
training loss, the implementation maps both predictions and targets back to
the physical coefficient scale:
\begin{equation}
    \widehat{\bm{c}}
    =
    \widehat{\bm{c}}_{\mathrm{std}}
    \odot
    \bm{\sigma}_c
    +
    \bm{\mu}_c.
\end{equation}
The Two-Step loss is
\begin{equation}
\begin{aligned}
    \mathcal{L}_{\mathrm{TwoStep}}
    &=
    \frac{1}{Br}
    \sum_{i=1}^{B}
    \left\|
        \widehat{\bm{c}}^{(i)}
        -
        \bm{c}^{(i)}
    \right\|_2^2
    \\
    &\quad+
    \lambda_{\mathrm{rel}}
    \frac{1}{B}
    \sum_{i=1}^{B}
    \frac{
        \left\|
            \widehat{\bm{c}}^{(i)}
            -
            \bm{c}^{(i)}
        \right\|_2
    }{
        \left\|
            \bm{c}^{(i)}
        \right\|_2+\varepsilon
    }.
\end{aligned}
\end{equation}
Computing the loss after undoing coefficient standardization restores the
physical POD/SVD energy weighting and avoids overemphasizing weak
high-index modes.
For the adaptive model, the drift penalty
\begin{equation}
    \mathcal{L}_{\mathrm{drift}}
    =
    \lambda_{\mathrm{drift}}
    \left\|
        \Delta A_\theta
    \right\|_F^2
\end{equation}
is added to the coefficient-space loss.
\subsection{Adaptive training schedule}
The adaptive branch is initialized from the trained fixed topological branch.
Training then proceeds in two stages.
In the first stage, the FNO branch is frozen and only
\(\Delta A_\theta\) is optimized:
\begin{equation}
    \min_{\Delta A_\theta}
    \mathcal{L}_{\mathrm{TwoStep}}
    +
    \mathcal{L}_{\mathrm{drift}}.
\end{equation}
In the second stage, the measurement correction and the FNO branch are
optimized jointly:
\begin{equation}
    \min_{\Delta A_\theta,\theta_{\mathrm{FNO}}}
    \mathcal{L}_{\mathrm{TwoStep}}
    +
    \mathcal{L}_{\mathrm{drift}}.
\end{equation}
The default schedule uses \(1000\) measurement-only epochs followed by
\(2000\) joint epochs.
\subsection{Vanilla DeepONet representation}
Vanilla DeepONet uses the complete standardized input history. Its branch
network is the same three-dimensional FNO encoder, producing a latent vector
\begin{equation}
    \bm{b}
    \in
    \mathbb{R}^{p}.
\end{equation}
The trunk network receives a space-time coordinate
\begin{equation}
    \bm{\xi}
    =
    (x,y,t)
    \in
    [0,1]^3
\end{equation}
and returns
\begin{equation}
    \bm{\tau}(\bm{\xi})
    \in
    \mathbb{R}^{p}.
\end{equation}
The predicted vorticity is
\begin{equation}
    \widehat{\omega}(\bm{\xi})
    =
    \bm{b}^{\mathsf T}
    \bm{\tau}(\bm{\xi})
    +
    b_0.
\end{equation}
Unlike the four Two-Step models, Vanilla DeepONet predicts space-time values
directly and does not use the common output SVD basis.
\subsection{Evaluation metrics}
For test realization \(i\), the sample-wise relative \(L^2\) error over the
entire predicted trajectory is
\begin{equation}
    e_i
    =
    \frac{
        \left\|
            \widehat{\bm{y}}^{(i)}
            -
            \bm{y}^{(i)}
        \right\|_2
    }{
        \left\|
            \bm{y}^{(i)}
        \right\|_2
    }.
\end{equation}
The mean sample-wise relative error is
\begin{equation}
    E_{\mathrm{mean}}
    =
    \frac{1}{N_{\mathrm{test}}}
    \sum_{i=1}^{N_{\mathrm{test}}}
    e_i.
\end{equation}
The global relative \(L^2\) error is
\begin{equation}
    E_{\mathrm{global}}
    =
    \frac{
        \left\|
            \widehat Y-Y
        \right\|_F
    }{
        \left\|
            Y
        \right\|_F
    }.
\end{equation}
The root-mean-square error is
\begin{equation}
    \mathrm{RMSE}
    =
    \sqrt{
        \frac{
            1
        }{
            N_{\mathrm{test}}N_hT_{\mathrm{out}}
        }
        \sum_{i=1}^{N_{\mathrm{test}}}
        \sum_{m=1}^{N_h}
        \sum_{n=1}^{T_{\mathrm{out}}}
        \left(
            \widehat{\omega}_{m,n}^{(i)}
            -
            \omega_{m,n}^{(i)}
        \right)^2
    }.
\end{equation}
The mean absolute error is
\begin{equation}
    \mathrm{MAE}
    =
    \frac{
        1
    }{
        N_{\mathrm{test}}N_hT_{\mathrm{out}}
    }
    \sum_{i=1}^{N_{\mathrm{test}}}
    \sum_{m=1}^{N_h}
    \sum_{n=1}^{T_{\mathrm{out}}}
    \left|
        \widehat{\omega}_{m,n}^{(i)}
        -
        \omega_{m,n}^{(i)}
    \right|.
\end{equation}
These are the metrics reported by the implementation.
\subsection{Interpretation of the model comparison}
The Full Two-Step, Sensor Two-Step, Fixed Topological Two-Step, and Adaptive
Topological Two-Step models use a common output SVD basis and the same
three-dimensional FNO encoder family. Their principal difference is the
representation of the input history.
The full-field model uses
\begin{equation}
    4096\times10
\end{equation}
scalar input values. The sensor and topological models use
\begin{equation}
    128\times10
\end{equation}
measurements, corresponding to a compression ratio
\begin{equation}
    \frac{
        4096\times10
    }{
        128\times10
    }
    =
    32.
\end{equation}
The fixed topological and sensor comparisons should be interpreted as
equal representation-budget comparisons. They are not necessarily equal
physical sensing-cost comparisons, because a global functional measurement
may require access to the full spatial vorticity field.
The adaptive model retains the same final measurement dimension \(q=128\)
per input time but learns corrections within the span of the complete
dictionary. Its comparison with the fixed model therefore evaluates whether
adapting the measurement functionals improves prediction while preserving
the same compressed coordinate dimension.
\end{document}